\documentclass{article}

\PassOptionsToPackage{square, numbers}{natbib}
\usepackage[preprint]{neurips_2026}

\usepackage[utf8]{inputenc} % allow utf-8 input
\usepackage[T1]{fontenc}    % use 8-bit T1 fonts
\usepackage{hyperref}       % hyperlinks
\usepackage{url}            % simple URL typesetting
\usepackage{booktabs}       % professional-quality tables
\usepackage{amsfonts}       % blackboard math symbols
\usepackage{nicefrac}       % compact symbols for 1/2, etc.
\usepackage{microtype}      % microtypography
\usepackage{xcolor}         % colors

\usepackage{amsmath, amssymb, amsthm, mathtools}
\usepackage{bm}
\usepackage{mathrsfs}
\usepackage{physics}
\usepackage{cancel}
\usepackage{hyperref}
\usepackage{url}
\usepackage{graphicx}
\usepackage{algorithm}
\usepackage{algorithmic}
\usepackage{xcolor}
\usepackage{listings}
\hypersetup{
    colorlinks=true,
    linkcolor=blue,
    filecolor=blue,
    urlcolor=blue,
    pdftitle={Full Dynamical Stability Analysis of SGD},
    pdfpagemode=FullScreen,
}

\usepackage{booktabs}
\usepackage{enumitem}

\theoremstyle{plain}
\newtheorem{theorem}{Theorem}[section]
\newtheorem{proposition}[theorem]{Proposition}

\newtheorem{corollary}[theorem]{Corollary}
\theoremstyle{definition}

\newtheorem{assumption}[theorem]{Assumption}
\theoremstyle{remark}

\newcommand{\diff}[1]{\textcolor{blue!70!black}{\scriptsize(#1)}}
\newcommand{\diffbest}[1]{\textbf{\textcolor{green!70!black}{\scriptsize(#1)}}}
\newcommand{\valuebest}[1]{\textbf{\textcolor{green!70!black}{#1}}}
\newcommand{\valuenotbest}[1]{\textbf{\textcolor{blue!70!black}{#1}}}

\definecolor{softpurple}{RGB}{128, 70, 170}

\usepackage{amsmath}
\usepackage{amssymb}
\usepackage{mathtools}
\usepackage{amsthm}
\usepackage{enumitem}

\usepackage{microtype}
\usepackage{graphicx}
\usepackage{subcaption}
\usepackage{booktabs}
\usepackage{hyperref}
\usepackage{url}
\usepackage{amsthm}
\usepackage{amssymb}
\usepackage{graphicx}
\usepackage{algorithm}
\usepackage{algorithmic}
\usepackage{xcolor}
\usepackage{listings}
\usepackage{colortbl}
\definecolor{lightblue}{RGB}{219, 234, 254}
\definecolor{lightolive}{RGB}{220, 230, 180}
\definecolor{lightyellow}{RGB}{255, 249, 196}
\definecolor{lightpurple}{RGB}{238,226,241}
\definecolor{lightred}{RGB}{251,220,221}

\makeatletter
\renewcommand\paragraph{\@startsection{paragraph}{4}{\z@}%
  {0ex \@plus 0.05ex \@minus 0.05ex}%
  {-1em}%
  {\normalfont\normalsize\bfseries}}
\makeatother
\usepackage[font=small]{caption}
\title{Forgetting Has Neighbors:\\ Localized Collateral Forgetting in Machine Unlearning}

\author{
Polina Dolgova$^{1,2}$ \quad
Sebastian U. Stich$^{1}$\\
$^1$CISPA Helmholtz Center for Information Security \quad
$^2$Universität des Saarlandes\\
Saarbrücken, Germany\\
\texttt{\{polina.dolgova, stich\}@cispa.de}
}

\begin{document}

\maketitle

\begin{abstract}
Machine unlearning aims to remove the influence of selected training examples without full retraining. Standard evaluations often summarize unlearning quality with aggregate metrics, such as accuracy- and forgetting-based scores, which can hide localized failures. We study this failure mode at the example level by comparing the predictions of an unlearned model to those of the model retrained after deletion. We show that this pointwise discrepancy can be highly non-uniform: for gradient-ascent and random-labeling methods, with and without retain-set fine-tuning, it grows with geometric proximity to the forget set. We call this phenomenon \emph{localized collateral forgetting}. Our analysis identifies a mechanism behind the effect: surrogate targets used during unlearning can be inconsistent with the local prediction structure induced by retraining, and this inconsistency propagates through shared representations to nearby examples. Motivated by this mechanism, we propose \emph{Local Teacher Distillation}, a simple mitigation strategy that replaces random targets with soft labels from a small teacher trained only on retained neighbors of the forget set. On CIFAR-100 partial-class deletion, this local teacher brings the unlearned model substantially closer to retraining, especially near the forget set, while maintaining competitive aggregate unlearning metrics.
\end{abstract}

\section{Introduction}
\vspace*{-1ex}

Machine unlearning aims to remove the influence of designated training examples from a trained model without retraining the model from scratch. This is needed, for example, when data must be deleted post hoc under regulations such as the GDPR~\citep{european_commission_regulation_2016}. The gold-standard reference is retraining after deletion: an unlearning method should approximate this reference at lower cost, while preserving utility on non-deleted data and limiting the recoverability of information about the forgotten examples.

Standard evaluations compare unlearned models to retraining using aggregate metrics such as retain accuracy, forget accuracy, test accuracy, and privacy-audit scores. These summaries are useful but incomplete. An unlearned model can appear close to retraining on average while deviating sharply on particular retained or test examples. We therefore study unlearning at the example level, asking where the predictions of an unlearned model differ from those of the retrained reference.

This pointwise view reveals a structured failure mode in common image-classification unlearning methods. Gradient ascent (GA) methods~\citep{9797378} remove the forget set by reversing its training signal, while random-labeling (RL) approaches~\citep{WarneckePirchWressnegger2023_1000166047, fan2024salun} replace it with corrupted targets. These interventions suppress the forget-set signal, but because examples are coupled through shared representations, the induced perturbation need not remain confined to the deleted examples. It can also affect retained and test points whose representations are close to the forget set.

Figure~\ref{fig:initial-mnist-example} illustrates this effect on MNIST. For both random labeling and gradient ascent, the gap to retraining increases with proximity to the forgotten class. We call this phenomenon \emph{localized collateral forgetting}: unintended degradation on non-deleted examples, measured relative to retraining, that concentrates near the forget set in representation space. 
Retain-set fine-tuning can reduce this degradation, but our analysis shows that forget-set targets inconsistent with retraining can still propagate mismatch to nearby examples.

These observations motivate three questions: what governs the pointwise discrepancy between unlearning and retraining, why is this discrepancy localized, and how can it be reduced?
We show theoretically that, for GA and RL-based methods, the discrepancy is governed by alignment with the forget set. Our analysis identifies target mismatch as a key mechanism: corrupted forget-set targets conflict with the local prediction structure induced by retraining. This suggests a direct mitigation strategy: replacing random targets on the forget set with soft labels from a local teacher trained on nearby retained examples.

\begin{figure}[t]
    \centering \vspace*{-1ex}
    \includegraphics[width=1.\linewidth]{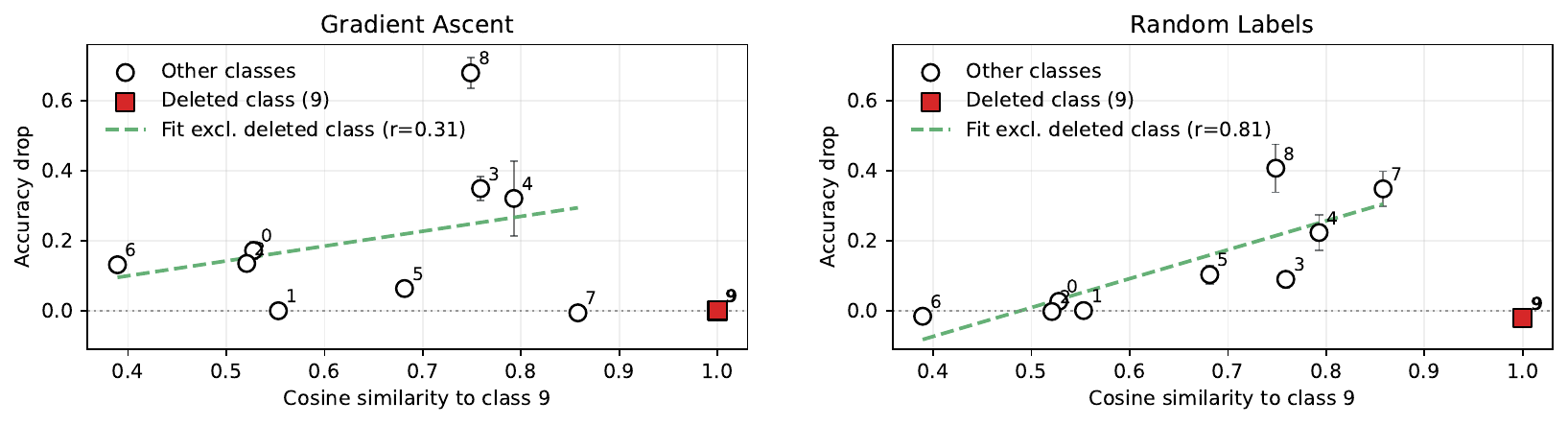}
    \caption{
    MNIST illustration of \emph{localized collateral forgetting} after deleting class \(9\), for Gradient Ascent (left) and forget-only Random Labels (right). Non-deleted classes closer to class \(9\) in representation space show larger accuracy drops relative to retraining. The dashed line is a linear fit over the non-deleted classes only.
    }
    \label{fig:initial-mnist-example}
\end{figure}

\textbf{Our contributions are as follows.}
\begin{enumerate}[topsep=0pt,itemsep=0pt,leftmargin=12pt]
    \item \textbf{Phenomenon.}
    We show that collateral forgetting has a localized structure: the discrepancy between an unlearned model and the retrained reference concentrates on non-forget examples that are close to the forget set in representation space.
    
    \item \textbf{Theory.}
    We analyze this localization for gradient ascent, forget-set-only random-label updates, and random-label unlearning with retain-set fine-tuning. The results show that the deviation from retraining is governed by alignment with the forget direction, or more generally by projection onto the forget span.
    The same analysis identifies target-mismatch as a mechanism underlying localized collateral forgetting.
    
    \item \textbf{Method and experiments.}
    Our analysis suggests a design principle: forget-set targets should be locally consistent with the behavior induced by
    retraining. We instantiate this principle with \emph{Local Teacher Distillation}, where a small teacher trained on retained neighbors provides soft forget-set targets in place of random labels. On CIFAR-100 partial-class deletion, this method substantially reduces the discrepancy from retraining, especially on affected-class and high-similarity examples, while maintaining competitive
    aggregate unlearning metrics.

\end{enumerate}

\section{Related work}
\vspace*{-1ex}

\paragraph{Machine unlearning for image classification.}
Practical unlearning methods for image classification include gradient ascent (GA)~\citep{9797378}, random labeling (RL)~\citep{WarneckePirchWressnegger2023_1000166047}, fine-tuning (FT)~\citep{WarneckePirchWressnegger2023_1000166047}, Influence Unlearning (IU)~\citep{pmlr-v70-koh17a, jia2023model}, SalUn~\citep{fan2024salun}, and AMUN~\citep{ebrahimpourboroojeny2025amunadversarialmachineunlearning}. While these methods are typically evaluated through aggregate metrics, we study a finer failure mode: several unlearning procedures can deviate systematically from retraining on retained and test examples that are close to the forget set.

\paragraph{Retain--forget entanglement.}
Recent work~\citep{zhao2024what, cheng2026machine, ebrahimpourboroojeny2025necessityoutputdistributionreweighting}
shows that retain--forget coupling in representation space makes unlearning harder. We study this coupling at the \emph{point level}, measuring how the unlearning--retraining discrepancy varies across retained and test examples with proximity to the forget set. Moreover, rather than only asking whether the forget set is removed, we identify which nearby non-forget examples are harmed despite being preserved by retraining.

\paragraph{Ripple and collateral effects.}
Ripple effects~\citep{Cohen2023EvaluatingTR, rinberg2025ripplebench} and related work on collateral or adjacency effects~\citep{wen-etal-2025-lock, thakral2025finegrainederasuretexttoimagediffusionbased} study original-model side effects on related facts or adjacent concepts. We instead take \emph{retraining} as the reference and study how the resulting discrepancy varies with similarity to the forget set. This separates benign retraining-induced changes from collateral forgetting: degradation on non-forget examples that retraining would preserve.

\section{Collateral Forgetting is Localized}
\label{sec:theory}

We study a side effect of machine unlearning that we call \emph{collateral forgetting}: the unlearned model unnecessarily degrades predictions on non-forget examples that retraining would preserve. With retraining as the reference, the relevant object is not an absolute prediction change, but the pointwise discrepancy between the unlearned and retrained models. We show that this discrepancy is localized, concentrating on retained and test examples close to the forget set in representation space.

\subsection{Theoretical Analysis}

We analyze localized collateral forgetting in two stylized settings in which the effect can be explicitly characterized. We begin with gradient-ascent unlearning in linear regression, which provides a simple geometric warm-up, and then turn to random-label unlearning, with and without retain-set fine-tuning, in logistic regression. In each case, our goal is to understand how the unlearning trajectory deviates from retraining after deleting the forget set.

Let the training data be split into a forget set and a retain set,
\[
D_f = \{(x_i,y_i)\}_{i\in I_f},
\qquad
D_r = \{(x_i,y_i)\}_{i\in I_r},
\]
where \(I_f\cap I_r=\emptyset\). Let \(\ell(\theta;x,y)\) denote the loss, let \(\theta_{\mathrm{full}}\) denote the model trained on the full dataset, and let \(\theta_{\mathrm{retrain}}\) denote the model retrained after deleting \(D_f\). Let \(X \in \mathbb{R}^{n \times d}\) denote the design matrix whose rows are the feature vectors \(x_i^\top\), and let \(y \in \mathbb{R}^n\) be the corresponding vector of responses.

Throughout this section, proximity is measured in representation space rather than input space. In the single-point case, the relevant quantity is the alignment between a query point and the deleted point; with normalized features, this corresponds to cosine similarity up to scale.

\paragraph{Gradient Ascent.}
We first consider deletion of a single point, $D_f=\{(x_f,y_f)\}$. Starting from $\theta_{\mathrm{full}}$, gradient-ascent unlearning on the deleted point is given by
\[
\theta_0=\theta_{\mathrm{full}},
\qquad
\theta_{t+1}
=
\theta_t+\eta \nabla_\theta \ell(\theta_t;x_f,y_f),
\qquad t\ge 0,\ \eta > 0.
\]

In linear regression, this gives a closed-form characterization of localized collateral forgetting: for a retained example $x_q$, we measure the effect by the excess squared loss relative to retraining.

\begin{proposition}[Localized collateral forgetting under gradient ascent]
\label{prop:linreg-ga}
Let \(D_f=\{(x_f,y_f)\}\), and let
\(X=(x_1,\ldots,x_n)^\top\in\mathbb{R}^{n\times d}\), with \(n>d\). Suppose the data follow the linear model
\[
y_i = x_i^\top \theta_{\mathrm{true}} + \xi_i,
\qquad
\xi_i \overset{\mathrm{i.i.d.}}{\sim} \mathcal{N}(0,\sigma^2), \; x_i \overset{\mathrm{i.i.d.}}{\sim} N(0, \Sigma),\;
\Sigma\succ 0,\; \theta_{\mathrm{true}} \in \mathbb{R}^d.
\]
Let \(\theta_{\mathrm{full}}\) and \(\theta_{\mathrm{retrain}}\) be the
least-squares solutions before and after deleting \(D_f\). Initialize \(\theta_0=\theta_{\mathrm{full}}\), and perform \(T\) steps of gradient ascent on the squared loss of the deleted point \((x_f,y_f)\). Then for any retained example \((x_q,y_q)\in D_r\),
\[
\mathbb{E}\!\left[(x_q^\top\theta_T-y_q)^2\mid  X \right]
-
\mathbb{E}\!\left[(x_q^\top\theta_{\mathrm{retrain}}-y_q)^2\mid  X\right]
=
\langle x_q, x_f\rangle^2 \cdot c(T,\eta,x_f)\, \sigma^2 + O_p(1/n),
\]
where the expectation is over the responses \(y\), equivalently over the observation noise \(\xi_i\), conditional on \(X\), and \(c(T,\eta,x_f)\) does not depend on \(x_q\).
\end{proposition}

Proposition~\ref{prop:linreg-ga} (for the proof see Appendix~\ref{sec:proof-GA}) isolates the basic geometry of localized collateral forgetting under gradient ascent. The update on the deleted point affects a retained point only through its overlap with the deleted direction, so the leading-order excess loss scales as $\langle x_q,x_f\rangle^2$. Thus, the discrepancy from retraining is not uniform over the retain set, but concentrates on points that are aligned with the forget point in representation space.

\paragraph{Random Labels.}
We now turn to random-label unlearning (RL) in logistic regression. In this part,
$\ell(\theta;x,y)$ denotes the logistic loss
\[
\ell(\theta;x,y)
=
-\Bigl(y\log \sigma(\langle \theta,x\rangle)
+(1-y)\log(1-\sigma(\langle \theta,x\rangle))\Bigr),
\qquad
\sigma(u)=(1+e^{-u})^{-1}.
\]
Here $y\in\{0,1\}$ for binary labels; for the reference models,
$\theta_{\mathrm{full}}$ and $\theta_{\mathrm{retrain}}$ are chosen as
logit-interpolating parameters, using soft targets $y=\sigma(z)\in[0,1]$ with the relevant \(X\theta=z\).

Write $m=|D_f|$, and let $X_F\in\mathbb{R}^{m\times d}$ denote the rows of $X$ associated with $D_f$. RL modifies only these examples, replacing their labels by fresh random labels at each step. This makes it a useful case for studying how a surrogate signal on the forget set propagates
to nearby retained points.

For each $t\ge 0$, let $\tilde y^{(t)}\in\{0,1\}^m$ be drawn coordinatewise from $\mathrm{Bernoulli}(1/2)$. Starting from $\theta_{\mathrm{full}}$, the RL iterates
with weight decay $\lambda\ge 0$ are
\[
\theta_0=\theta_{\mathrm{full}},
\qquad
\theta_{t+1}
=
(1-\eta_t\lambda)\theta_t
-
\eta_t \sum_{i\in I_f}\nabla_\theta \ell(\theta_t;x_i,\tilde y_i^{(t)}),
\qquad t\ge 0,\ \eta_t>0.
\]

The next result gives the exact retained-example logit discrepancy between RL and retraining.

\begin{proposition}[Localized collateral forgetting under Random Labeling]
\label{prop:logreg-rl}
Assume an interpolating overparameterized setting in which $d\ge n$ and
$\operatorname{rank}(X)=n$. For a retained example $x_q$, write
\[
z_t^{(q)}:=\langle\theta_t,x_q\rangle,
\qquad
z_t^{(F)}:=X_F\theta_t,
\qquad
\Gamma_{0:T}:=\prod_{t=0}^{T-1}(1-\eta_t\lambda).
\]
Let
\[
\alpha_q:=(X_FX_F^\top)^{-1}X_Fx_q,
\]
so that $X_F^\top\alpha_q$ is the orthogonal projection of $x_q$ onto
$\operatorname{span}(X_F^\top)$. Then
\[
\langle \theta_T,x_q\rangle-\langle \theta_{\mathrm{retrain}},x_q\rangle
=
(\Gamma_{0:T}-1)z_0^{(q)}
+
\alpha_q^\top\bigl(z_T^{(F)}-\Gamma_{0:T}z_0^{(F)}\bigr).
\]
\end{proposition}

Proposition~\ref{prop:logreg-rl} (for the proof see Appendix~\ref{sec:proof-RL}) shows that the discrepancy has two components:
a shrinkage term inherited from weight decay, and a perturbation term created by
the random-label updates on the forget set. The latter is transferred to $x_q$
through $\alpha_q$, whose induced vector $X_F^\top\alpha_q$ is the projection of $x_q$ onto the span
of the deleted examples. This projection is the multi-point analogue of alignment
with a single deleted point: for normalized features, a larger projected component
means that $x_q$ is closer to the forget set in representation space.

Thus, whenever the shrinkage term is dominated by the forget-set perturbation, for
example when
\[
\left|(\Gamma_{0:T}-1)z_0^{(q)}\right|
\le
\varepsilon_T
\left|
\alpha_q^\top\bigl(z_T^{(F)}-\Gamma_{0:T}z_0^{(F)}\bigr)
\right|,
\qquad \varepsilon_T<1,
\]
the finite-time discrepancy is controlled, up to relative error $\varepsilon_T$, by the component of $x_q$ lying in the forget span. In this sense, the formal projection term captures the intuitive locality effect: retained points closer to the forget set are affected more. This regime is especially relevant in practice, since RL unlearning is typically stopped after a small number of epochs rather than run to convergence.

The limiting behavior is also informative. Under the standard SGD convergence conditions in Assumption~\ref{ass:sgd-rl}, if the random-label dynamics are run to convergence, then
\[
\mathbb{E}\left[\langle \theta_T,x_q\rangle\right]\to 0,
\qquad
\mathbb{E}\left[
\langle \theta_T,x_q\rangle-\langle \theta_{\mathrm{retrain}},x_q\rangle
\right]
\to
-\langle \theta_{\mathrm{retrain}},x_q\rangle .
\]
Hence RL does not converge toward retraining in expectation; it drifts toward the zero-logit predictor induced by random labels. Together, the
finite-time and asymptotic views show two aspects of the same failure mode: at finite time, the discrepancy is locally structured by the projection of $x_q$ onto
the forget span, while in the limit the mean dynamics are misaligned with the retraining target.

Figure~\ref{fig:regression-plots-1} synthetically validates the GA and RL predictions: both the linear-regression MSE discrepancy and the logistic-regression squared-logit discrepancy increase with proximity to $D_f$.

\begin{figure}
    \centering
    \includegraphics[width=1.\linewidth]{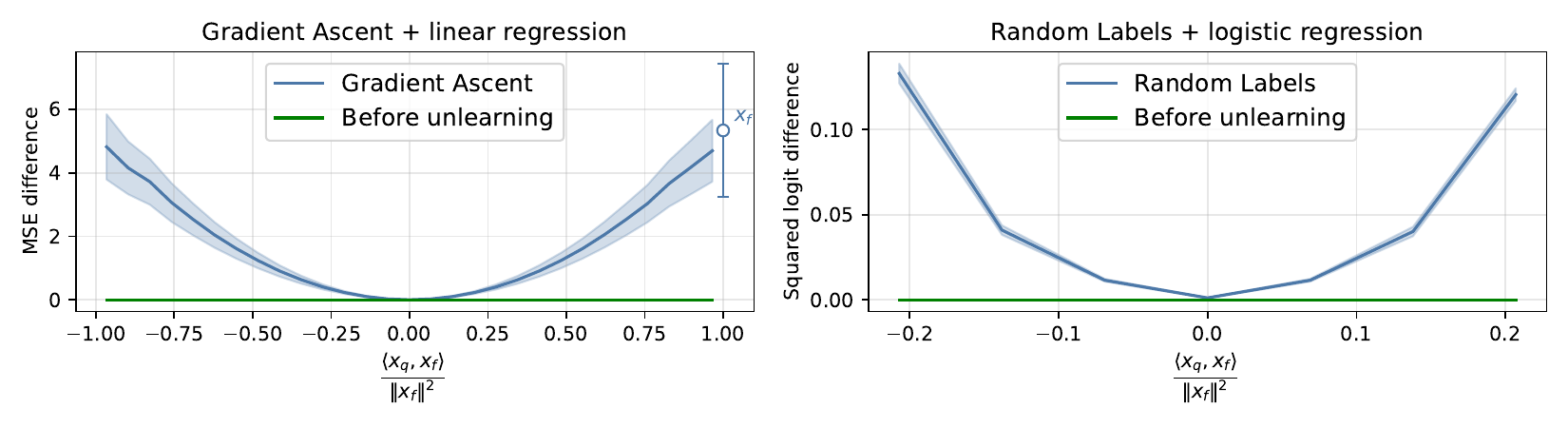}
    \caption{ Synthetic validation of the projection-based theory. For GA (left) and forget-only RL (right), retraining-relative discrepancies increase with alignment to the deleted point. RL queries are sampled from $\operatorname{span}(X_R^\top)$.
    }
    \label{fig:regression-plots-1}
\end{figure}

\paragraph{RL + retain fine-tuning.}
A natural question is whether fine-tuning on retained data can mitigate this localized degradation. We now add true-label updates on a retained subset $R\subseteq D_r$ to the random-label dynamics on $D_f$ defined above. This captures the basic structure of methods that combine corrupted forget-set updates with retain-set recovery.

Let $X_R$ be the input matrix of $R$, and let $X_{F\cup R}$ be obtained by
stacking $X_F$ and $X_R$. The random labels are still used on $D_f$; after taking
expectation over this randomness, the induced surrogate objective is
\[
L_{\mathrm{RL+FT}}(\theta)
:=
\sum_{(x_i,y_i)\in R}\ell(\theta;x_i,y_i)
+
\sum_{(x_i,y_i)\in D_f}\ell(\theta;x_i,1/2)
+
\frac{\lambda}{2}\|\theta\|^2 .
\]
Let $\theta^{(\lambda)}_{\mathrm{RL+FT}}$ denote the limiting solution of this surrogate, and
let $\{\theta_t\}_{t\ge0}$ be the corresponding RL+FT stochastic trajectory.
Under Assumption~\ref{ass:sgd-rl}, $\theta_t$ converges to $\theta^{(\lambda)}_{\mathrm{RL+FT}}$. As in Proposition~\ref{prop:logreg-rl}, the query dependence is
governed by projection onto $\operatorname{span}(X_{F\cup R}^\top)$. Thus,
retain fine-tuning changes the relevant subspace, but does not remove the forget
directions. We state this explicitly below; we derive the general identity in Appendix~\ref{app:rlft-identities} and prove
the full-retain statement in Appendix~\ref{prop:logreg-rl-ft}.

\begin{proposition}[RL with full retain fine-tuning]
\label{prop:logreg-rl-ft}
Assume the setting of Proposition~\ref{prop:logreg-rl} and consider the limit
$\lambda\to0^+$. Suppose $R=D_r$. Then, for every retained training point
$(x_q,y_q)\in D_r$,
\[
\lim_{\lambda\to0^+}\lim_{T\to\infty}\mathbb{E}\left[
\langle \theta_T,x_q\rangle
-
\langle \theta_{\mathrm{retrain}},x_q\rangle
\right]
= 0 .
\]
Moreover, if $\theta_{\mathrm{retrain}}$ is the minimum-norm interpolant on
$D_r$ and $x_q\in\operatorname{span}(X^\top)$ with
\[
x_q = X_R^\top v_R + X_F^\top v_F,
\]
then at the limiting surrogate solution,
\[
\lim_{\lambda\to0^+}
\langle \theta_{\mathrm{RL+FT}}^{(\lambda)},x_q\rangle
=
\langle \theta_{\mathrm{retrain}},x_q\rangle
-
\langle \theta_{\mathrm{retrain}},X_F^\top v_F\rangle .
\]
\end{proposition}

Thus, for $\lambda\to0^+$, full retain fine-tuning can restore agreement with retraining on $D_r$, but this does not imply agreement on general nearby queries. The remaining discrepancy is determined by the component of $x_q$ lying in the forget span. This suggests that mitigating localized collateral forgetting requires improving the signal injected on $D_f$, rather than relying only on retain-set recovery.

\section{Method: Local Teacher Distillation}
\label{sec:method}

\subsection{Motivation}

\paragraph{Target mismatch.}
The analysis in Section~\ref{sec:theory} suggests that the main failure mode behind localized collateral forgetting is a mismatch between the targets imposed on $D_f$ and the predictions induced by retraining. Random labels replace the behavior of the retrained model on $D_f$ with arbitrary targets. Since non-forget examples can have nonzero projection onto the feature subspace spanned by the forget examples, this mismatch is transferred to nearby points, causing the localized collateral effect.

Retain fine-tuning only partially addresses this issue. As shown by the RL+FT analysis, fine-tuning on \(D_r\) can restore agreement with retraining on retained training points, but it does not generally determine the model's behavior along components in
the span of $D_f$. Thus, query points with representation components in the forget subspace can retain a residual discrepancy. This suggests that the target used on $D_f$ itself must be improved, rather than relying only on retain-set fine-tuning to correct the effect afterward.

\paragraph{Ideal soft labels.}
The ideal target is the prediction of the retrained model. If \(f_{\theta_{\mathrm{retrain}}}\) were available, then for each forget example \(x\in D_f\) we would use
\[
\hat y_{\mathrm{ret}}(x):=f_{\theta_{\mathrm{retrain}}}(x)\in\Delta^{C-1}
\]
as a soft label, where \(\Delta^{C-1}\) is the probability simplex over \(C\) classes. Training against \(\hat y_{\mathrm{ret}}(x)\) would make the forget-set objective locally consistent with retraining, rather than pushing the model toward random targets.  However, these labels are unavailable without retraining and are difficult to infer a priori. This is especially clear in partial-class deletion, where the retrained prediction on a deleted example is not determined by its class identity alone, since other same-class examples remain in the training set.

\paragraph{Local approximation.} We therefore approximate these ideal soft labels with a small teacher trained on a selected subset of retained examples, rather than on all of $D_r$. The teacher
is used only to produce soft predictions on $D_f$, which serve as a proxy for $\hat y_{\mathrm{ret}}(x)$. The key question is which retained examples should be used to train this teacher.

\paragraph{Support selection.}~\label{par:support-selection} Natural candidates for the teacher support are retained examples near $D_f$ in
representation space, as they are most informative about the local behavior of retraining near the forget set. The projection analysis motivates this choice. In the single-point setting $D_f=\{(x_f,y_f)\}$, the induced discrepancy is
controlled by the residual left after projecting the forget point $x_f$ onto the support direction $x_s$:
\[
r_s=x_f-P_{x_s}x_f,
\qquad
\|r_s\|^2=\|x_f\|^2\bigl(1-\cos^2(x_f,x_s)\bigr),
\]
where $P_{x_s}=x_sx_s^\top/\|x_s\|^2$; see
Proposition~\ref{prop:distill-labels}. Thus, higher cosine similarity yields a smaller residual in this idealized setting. For multiple forget examples, this suggests choosing retained examples close to the forget subspace \(\mathrm{span}(X_F^\top)\); in practice, we use nearest-neighbors in representation space.

\subsection{Local Support and Teacher-Guided Unlearning}

We instantiate the design principle above as \emph{Local Teacher Distillation}. Let \(f_{\theta_{\mathrm{full}}}\) be the model trained on the full dataset, with forget set \(D_f\) and retain set \(D_r\). The unlearned model is initialized at \(\theta_{\mathrm{full}}\). Instead of assigning random labels to \(D_f\), we train a small auxiliary teacher on a local subset of retained examples and use its predictions as soft labels on \(D_f\).

\paragraph{Support set selection.}~\label{par:similarity}
Let \(h_{\theta_{\mathrm{full}}}(x)\in\mathbb{R}^p\) be the penultimate-layer representation of \(x\), and let \(\bar h_{\theta_{\mathrm{full}}}(x)=h_{\theta_{\mathrm{full}}}(x)/\|h_{\theta_{\mathrm{full}}}(x)\|_2\).
In our partial-class deletion setting, \(D_f\) is relatively homogeneous, so we
represent it by the average normalized forget embedding
\[
\bar h_F:=\frac{u_F}{\|u_F\|_2}, \quad \text{where}\;\;\;
u_F:=\sum_{(x_f,y_f)\in D_f} h_{\theta_{\mathrm{full}}}(x_f).
\]
We score each retained example by cosine similarity to this average,
\[
s(x,D_f)
=
\left\langle
\bar h_{\theta_{\mathrm{full}}}(x),
\bar h_F
\right\rangle,
\]
and define \(S_k(D_f)\subset D_r\) as the \(k\) retained examples with largest
scores. For heterogeneous forget sets, one can replace the average embedding by
per-example or cluster-based nearest-neighbor selection; we discuss this variant
in Appendix~\ref{app:heterogeneous-support}.

\paragraph{Teacher training.}
We train a small auxiliary teacher \(g_\psi\) on \(S_k(D_f)\) using the original hard labels of these retained examples. The teacher is not trained on \(D_f\) and does not use forget labels. Its role is to approximate the local prediction structure near the forget set at substantially lower cost than full retraining. In our experiments, the base model is ResNet-56 and \(g_\psi\) is ResNet-8. After training, the teacher produces soft labels
\[
\hat y_\psi(x):=g_\psi(x)\in\Delta^{C-1},
\qquad x\in D_f,
\]
where \(\Delta^{C-1}\) is the probability simplex over \(C\) classes. These soft labels are used as proxies for the unavailable retraining soft labels \(\hat y_{\mathrm{ret}}(x)\). Equivalently, the teacher defines a soft-labeled forget set
\[
\widehat D_f:=\{(x,\hat y_\psi(x)):\,x\in D_f\}.
\]
In practice, we keep and renormalize the top-3 teacher probabilities, which speeds convergence while retaining more than \(90\%\) of the prediction mass.

\paragraph{Unlearning objective.}
Starting from $\theta_{\mathrm{full}}$, we optimize
\[
\mathcal{L}_{\mathrm{LTD}}(\theta)
=
\sum_{(x,y)\in D_r}
\mathrm{CE}\!\left(e_y,f_\theta(x)\right)
+
\beta
\sum_{(x,\hat y)\in \widehat D_f}
\mathrm{CE}\!\left(\hat y,f_\theta(x)\right),
\]
where $e_y\in\Delta^{C-1}$ is the one-hot label vector, $\beta>0$ controls the
strength of the forget-set distillation term, and $\mathrm{CE}(p,q):=-\sum_{c=1}^C p_c\log q_c$.
The first term preserves performance on \(D_r\), while the second replaces random labels with teacher soft labels that approximate retraining locally near \(D_f\). In practice, we optimize this objective using minibatches from \(D_r\) and \(\widehat D_f\).

\begin{algorithm}[t]
\caption{Local Teacher Distillation for Homogeneous Forget Sets}
\label{alg:teacher_unlearning}
\begin{algorithmic}[1]
\STATE \textbf{Input:} trained model $f_{\theta_{\mathrm{full}}}$, forget set $D_f$, retain set $D_r$, support size $k$, loss weight $\beta$, unlearning epochs $E$, teacher budget $B$
\STATE Compute normalized embeddings $\bar h_{\theta_{\mathrm{full}}}(x)$ for all $x\in D_f\cup D_r$ and $\bar h_F$.
\STATE Select \(S_k(D_f)\subset D_r\) as the \(k\) retained examples most similar to $\bar h_F$
\STATE Train a small teacher $g_\psi$ on $S_k(D_f)$ using retained hard labels for budget $B$
\STATE Construct $\widehat D_f:=\{(x,\hat y_\psi(x)):\,x\in D_f\}$ from processed teacher predictions
\STATE Initialize $\theta\leftarrow\theta_{\mathrm{full}}$
\STATE Optimize \(\mathcal{L}_{\mathrm{LTD}}(\theta)\) on \(D_r\cup\widehat D_f\) for \(E\) epochs
\STATE \textbf{Return:} unlearned model $f_\theta$
\end{algorithmic}
\end{algorithm}

\paragraph{Hyperparameters.}
Local Teacher Distillation has four main hyperparameters: the support size \(k\), the distillation weight \(\beta\), the number of unlearning epochs \(E\), and the teacher-training budget \(B\). In the homogeneous-forgetting variant used in our experiments, \(k\) is the total number of retained examples selected around the forget set. The weight \(\beta\) controls the strength of the soft-label loss on
\(D_f\), \(E\) controls the adaptation of the original model after teacher labels are constructed, and \(B\) controls how well the auxiliary teacher is fit to the
selected support set; in our experiments, we train until a target support-set accuracy is reached, with a maximum epoch budget.

We choose these hyperparameters without access to the retrained model. For the support size $k$, we use a retrain-free locality diagnostic based on the same similarity structure as support selection. We bin retained examples by similarity, apply one gradient-ascent step on $D_f$ to a copy of the full model, and measure the accuracy drop in each bin relative to the original full model. The largest
drops in high-similarity bins identify the neighborhood most affected by forget-set perturbations, which we use to choose $k$. The remaining parameters are chosen to fit the teacher targets on $D_f$ while preserving performance on $D_r$. Exact values and ablations are reported in Appendices~\ref{app:hyperparameters} and~\ref{app:ablations}.

\paragraph{Computational cost.}
The extra cost of LTD is embedding, support selection, and teacher training. For
the homogeneous-forgetting variant used in our partial-class deletion experiments, we represent \(D_f\) by the average normalized forget embedding and score retained
examples by cosine similarity. After one embedding pass over \(D_f\cup D_r\),
this costs
\[
O(|D_f|p) + O(|D_r|p) + O(|D_r|),
\]
for averaging, scoring, and unordered top-$k$ selection using a linear-time selection algorithm, where
\(p\) is the embedding dimension. The GA-based locality diagnostic adds
\(O(\mathrm{grad}(D_f)) + O(\mathrm{forward}(D_r)) + O(|D_r|(p + 1))\). Teacher training scales with the
selected support size \(k\), rather than \(|D_r|\), and uses a smaller
architecture. The final unlearning stage costs \(E\) epochs of standard
fine-tuning on \(\mathcal{L}_{\mathrm{LTD}}\). Heterogeneous forget set variants are
discussed in Appendix~\ref{app:heterogeneous-support}.

\section{Experiments}

\paragraph{Setup.}
We evaluate our method on CIFAR-100~\citep{krizhevsky2009learning} with a ResNet-56~\citep{He2015DeepRL} backbone. We consider partial-deletion regimes, removing $50\%$ and $90\%$ of a single class (we use the '\texttt{couch}' class). Additional robustness studies, including
$90\%$ deletion, additional affected classes, and SVHN with a ViT-Tiny backbone,
are reported in Appendix~\ref{app:additional-experiments}.
We refer to the partially deleted class as the \emph{affected class}.
Partial-class deletion is a stricter and more realistic setting than whole-class removal: since same-class examples remain in the retained data, retraining behavior near $D_f$ is nontrivial.  Hyperparameter ranges and training details for all methods are reported in Appendix~\ref{app:hyperparameters}. Code is available at \url{https://github.com/polina-dolgova/local_teacher_distillation}. Results are averaged over 5 runs.

\paragraph{Baselines.}
We compare against standard unlearning baselines: random-labeling (RL), fine-tuning (FT), gradient ascent (GA), SalUn, Influence Unlearning (IU), and
AMUN.

\paragraph{Metrics.}
We report standard unlearning metrics: retain accuracy (RA), test accuracy (TA), accuracy on $D_f$ (UA), SVC-based membership inference (MIA)~\citep{jia2023model}, Avg. Gap (mean absolute deviation from retrain values), and runtime in seconds (RTE). MIA is included as a standard audit, while our main analysis focuses on \emph{discrepancies from retraining} rather than on optimizing a particular attack metric. To capture localized collateral forgetting, we also report affected-class accuracy on $D_r$, $D_f$, and the test subset, isolating localized degradation hidden by aggregate
accuracy. Finally, for each bin $B$ induced by the forget-set similarity score $s(x,D_f)$,
we measure drops relative to retraining, where zero means agreement with retraining, while larger magnitudes indicate stronger localized deviation:
\resizebox{\linewidth}{!}{
\begin{minipage}{1.1\linewidth}
\[ 
\Delta_{\mathrm{acc}}(B)
=
\mathrm{Acc}_{\mathrm{retrain}}(B)
-
\mathrm{Acc}_{\mathrm{unlearn}}(B),
\qquad
\Delta_{\mathrm{conf}}(B)
=
\frac{1}{|B|}\sum_{(x,y)\in B}
\left[
f_{\theta_{\mathrm{retrain}}}(x)_y
-
f_{\theta_{\mathrm{unlearn}}}(x)_y
\right].
\]
\end{minipage}
}

\subsection{Main results}

\paragraph{Aggregate and affected-class performance.} Following standard empirical unlearning, we compare methods by agreement with retraining, rather than by minimizing accuracy on $D_f$ in isolation.  Across aggregate metrics (Table~\ref{tab:main_results}), our method achieves the lowest Avg. Gap, indicating the closest approximation to retraining.  However, aggregate metrics alone can be misleading. As shown in Table~\ref{tab:forget_class_subset} and Figure~\ref{fig:forget-class-comparison}, methods that perform well globally may still induce substantial degradation on affected-class points. In particular, IU -- the second-best method by Avg. Gap -- exhibits significant drops in RA$_{\text{affected-class}}$ and TA$_{\text{affected-class}}$, indicating instability on the retrain-relevant neighborhood. In contrast, our method remains consistently close to retraining both globally and on affected-class subsets, suggesting improved control over localized
collateral forgetting. Appendix~\ref{app:additional-experiments} shows that the same pattern persists under $90\%$ deletion, across five additional affected classes, and on
SVHN with a ViT-Tiny backbone.

\paragraph{Similarity-based analysis.}
We further analyze these spillover effects by measuring accuracy drop and correct-class confidence drop relative to retraining as a function of similarity to the forget set (Figure~\ref{fig:acc-drop}). We observe a clear locality effect: points with higher similarity to the forget set exhibit substantially larger deviations from retraining, both on the retain and test sets. This confirms that unlearning errors concentrate in the neighborhood of the removed data. Across all similarity bins, our method maintains consistently low deviation from retraining, remaining stable on both datasets. In contrast, competing methods exhibit sharp degradation in high-similarity regions, indicating poor control over localized collateral forgetting.

\begin{table}[t]
\caption{Main results for class-fraction-to-forget $=0.5$. Values are mean $\pm$ std over 5 runs; all accuracy metrics are in $\%$. Blue values show deviation from Retrain. Avg. Gap is the mean absolute deviation from Retrain over UA, RA, TA, and MIA.}
\label{tab:main_results}
\resizebox{\linewidth}{!}{%
\centering
\small
\setlength{\tabcolsep}{4pt}
\begin{tabular}{lcccccc}
\multicolumn{1}{c}{\bf Method} &
\multicolumn{1}{c}{\bf UA} &
\multicolumn{1}{c}{\bf RA} &
\multicolumn{1}{c}{\bf TA} &
\multicolumn{1}{c}{\bf MIA} &
\multicolumn{1}{c}{\bf Avg. Gap} &
\multicolumn{1}{c}{\bf RTE}
\\ \toprule
Retrain & $47.4 \pm 3.2$ & $99.9 \pm 0.0$ & $72.1 \pm 0.2$ & $71.3 \pm 3.1$ & 0 & $1558.9 \pm 22.0$ \\
\midrule
RL    & $73.2 \pm 7.1$ \diff{+25.8} & $94.6 \pm 0.8$ \diff{-5.3} & $67.9 \pm 0.6$ \diff{-4.2} & $84.9 \pm 13.6$ \diff{+13.6} & \valuenotbest{$12.2$} & $154.4 \pm 0.8$ \\
FT    & $79.0 \pm 6.8$ \diff{+31.6} & $94.1 \pm 1.1$ \diff{-5.8} & $67.6 \pm 1.0$ \diff{-4.5} & $28.6 \pm 8.3$ \diff{-42.7} & \valuenotbest{$21.2$} & $158.6 \pm 1.7$ \\
GA    & $68.6 \pm 0.6$ \diff{+21.2} & $98.6 \pm 0.0$ \diff{-1.3} & $69.6 \pm 0.0$ \diff{-2.5} & $65.7 \pm 0.3$ \diff{-5.6} & \valuenotbest{$7.7$} & $1.8 \pm 0.0$ \\
IU    & $70.6 \pm 8.1$ \diff{+23.2} & $99.6 \pm 0.0$ \diffbest{-0.3} & $71.5 \pm 0.2$ \diff{-0.6} & $69.1 \pm 7.7$ \diff{-2.2} & \valuenotbest{$6.6$} & $29.8 \pm 0.5$ \\
SalUn & $67.2 \pm 5.4$ \diff{+19.8} & $94.3 \pm 0.8$ \diff{-5.6} & $67.9 \pm 0.7$ \diff{-4.2} & $84.0 \pm 12.0$ \diff{+12.7} & \valuenotbest{$10.6$} & $361.1 \pm 6.6$ \\
AMUN & $67.4 \pm 2.2$ \diff{+20.0} & $94.3 \pm 0.1$ \diff{-5.6} & $72.6 \pm 0.1$ \diffbest{+0.5} & $45.3 \pm 3.8$ \diff{-26.0} & \valuenotbest{$13.0$} & $171.0 \pm 0.8$ \\
\midrule
 LTD (ours) & $53.4 \pm 8.4$ \diffbest{+6.0} & $99.4 \pm 0.1$ \diff{-0.5} & $70.8 \pm 0.3$ \diff{-1.3} & $69.4 \pm 9.4$ \diffbest{-1.9} & \valuebest{$2.4$} & $226.5 \pm 5.2$ \\
\bottomrule
\end{tabular}
}
\end{table}

\begin{table}[t]
\caption{Results on the subset of points whose true label equals the affected class. Values are mean $\pm$ std over 5 runs, in $\%$. Blue values show deviation from Retrain.}
\label{tab:forget_class_subset}
\centering
\small
\setlength{\tabcolsep}{6pt}
\scalebox{0.9}{
\begin{tabular}{lcccc}
\multicolumn{1}{c}{\bf Method} &
\multicolumn{1}{c}{\bf RA$_{\text{affected-class}}$} &
\multicolumn{1}{c}{\bf UA$_{\text{affected-class}}$} &
\multicolumn{1}{c}{\bf TA$_{\text{affected-class}}$} &
\multicolumn{1}{c}{\bf Avg. Gap}
\\ \toprule
Retrain & $99.7 \pm 0.3$ & $47.4 \pm 3.1$ & $49.6 \pm 5.2$ & 0 \\
\midrule
RL      & $88.2 \pm 4.0$ \diff{-11.4} & $73.2 \pm 7.1$ \diff{+25.8} & $50.2 \pm 6.0$ \diffbest{+0.6} & \valuenotbest{$12.6$} \\
FT      & $95.2 \pm 3.6$ \diff{-4.5}  & $79.0 \pm 6.8$ \diff{+31.6} & $58.6 \pm 7.9$ \diff{+9.0} & \valuenotbest{$15.0$} \\
GA      & $66.6 \pm 0.2$ \diff{-33.1} & $68.6 \pm 0.6$ \diff{+21.2} & $31.0 \pm 0.0$ \diff{-18.6} & \valuenotbest{$24.3$} \\
IU      & $74.6 \pm 5.3$ \diff{-25.0} & $70.6 \pm 8.1$ \diff{+23.1} & $32.0 \pm 3.2$ \diff{-17.6} & \valuenotbest{$21.9$} \\
SalUn   & $81.2 \pm 5.8$ \diff{-18.5} & $67.2 \pm 5.4$ \diff{+19.8} & $43.0 \pm 6.3$ \diff{-6.6} & \valuenotbest{$15.0$} \\
AMUN    & $87.0 \pm 1.1$ \diff{-12.6} & $67.4 \pm 2.2$ \diff{+19.9} & $54.4 \pm 2.9$ \diff{+4.8} & \valuenotbest{$12.4$} \\
\midrule
LTD (Ours) & $97.4 \pm 1.4$ \diffbest{-2.2} & $53.4 \pm 8.4$ \diffbest{+5.9} & $43.0 \pm 5.9$ \diff{-6.6} & \valuebest{$4.9$} \\
\bottomrule
\end{tabular}}
\end{table}

\begin{figure}[!t]
    \centering \vspace*{-3pt}    \includegraphics[width=1.\linewidth]{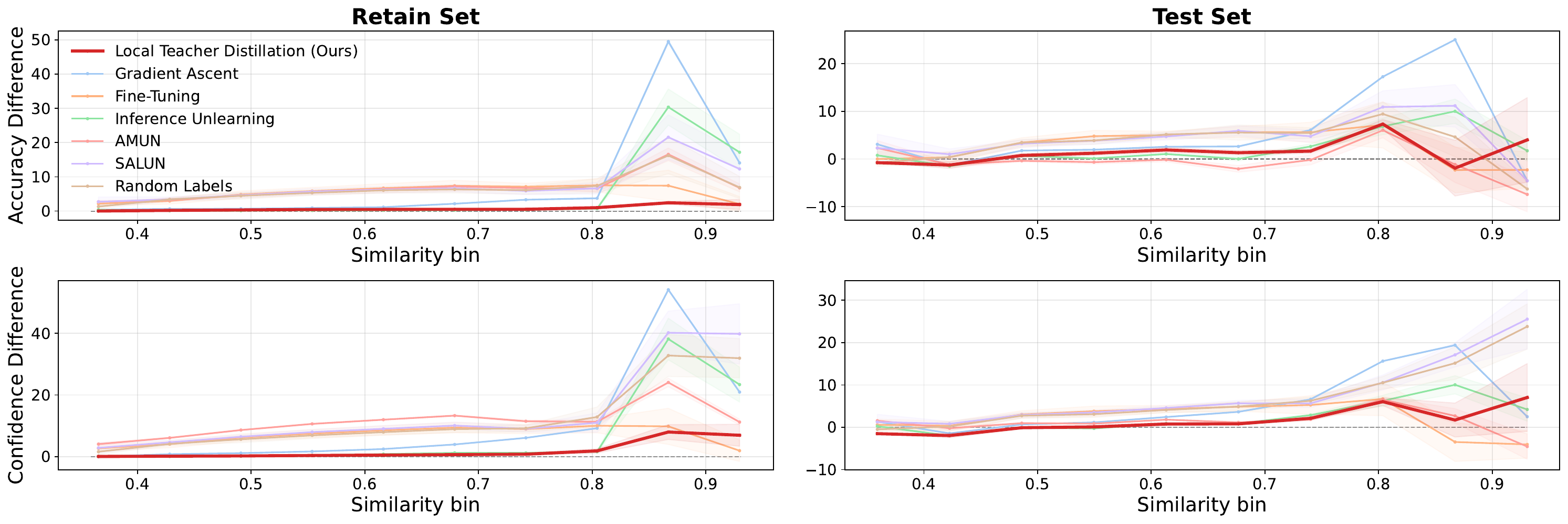}
    \caption{\textbf{Localized prediction deviations under $50\%$ class deletion.}
    We plot bin-level differences from retraining:
    $\mathrm{Acc}_{\mathrm{retrain}}-\mathrm{Acc}_{\mathrm{unlearn}}$ for accuracy
    (top row) and
    $f_{\theta_{\mathrm{retrain}}}(x)_y-f_{\theta_{\mathrm{unlearn}}}(x)_y$ for
    correct-class confidence (bottom row), as functions of similarity to $D_f$ on
    both retain and test sets. Larger deviations at high similarity indicate that
    unlearning effects concentrate near the removed data.}
    \label{fig:acc-drop}
\end{figure}

\begin{figure}[!t]
    \centering  \vspace*{-6pt}   \includegraphics[width=1.\linewidth]{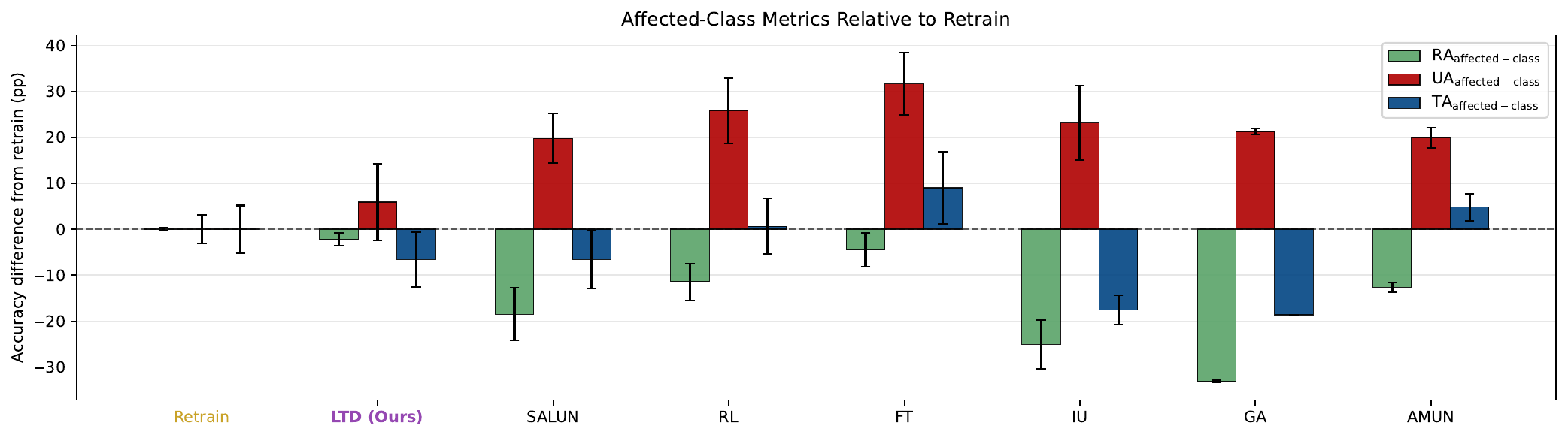}
    \caption{\textbf{Affected-class performance difference from Retrain.}
    Bars show differences from Retrain for RA$_{\text{affected-class}}$,
    UA$_{\text{affected-class}}$, and TA$_{\text{affected-class}}$; values closer to
    zero indicate better agreement. The figure reveals localized collateral
    forgetting not captured by aggregate metrics.}
    \label{fig:forget-class-comparison}
\end{figure}

\paragraph{Ablations.} We ablate support selection, teacher architecture, teacher quality, and support
size, with full results in Appendix~\ref{app:ablations}. Locality is the most
important factor: using nearest-neighbor support substantially improves alignment with retraining compared to random or distant support (Table~\ref{tab:support_selection_ablation}). Additional ablations show that the method is robust to the teacher architecture, while teacher quality and support size affect how well the teacher approximates retraining near the forget set.

\begin{table}[!t]
\caption{\textbf{Ablation on support selection.} Nearest-neighbor support yields substantially better alignment with retraining, while random or distant support fails to capture the relevant prediction structure.}
\label{tab:support_selection_ablation}
\resizebox{\linewidth}{!}{%
\centering
\small
\setlength{\tabcolsep}{5pt}
\begin{tabular}{lcccccc}
\toprule
  &
\multicolumn{2}{c}{\bf Standard metrics} &
\multicolumn{3}{c}{\bf Forget-class subset metrics} &
\multicolumn{1}{c}{\bf Teacher}  \\ \cmidrule(r){2-3} \cmidrule(lr){4-6} \cmidrule(l){7-7}
\multicolumn{1}{l}{\bf Method} &
\multicolumn{1}{c}{\bf RA} &
\multicolumn{1}{c}{\bf TA} &
\multicolumn{1}{c}{\bf RA$_{\text{affected-class}}$} &
\multicolumn{1}{c}{\bf UA$_{\text{affected-class}}$} &
\multicolumn{1}{c}{\bf TA$_{\text{affected-class}}$} &
\multicolumn{1}{c}{\bf UA on $D_f$}
\\ \midrule
Retrain
& $99.9 \pm 0.0$
& $72.1 \pm 0.2$
& $99.7 \pm 0.3$
& $47.4 \pm 3.1$
& $49.6 \pm 5.2$
& $47.4 \pm 3.1$ \\
\midrule
Random
& $99.4 \pm 0.0$ \diff{-0.5}
& $70.5 \pm 0.2$ \diff{-1.5}
& $97.6 \pm 1.3$ \diffbest{-2.1}
& $20.0 \pm 3.0$ \diff{-27.4}
& $37.6 \pm 3.9$ \diff{-12.0}
& $0.3 \pm 0.3$ \diff{-47.1}\\
Farthest
& $99.4 \pm 0.0$ \diff{-0.4}
& $70.6 \pm 0.2$ \diff{-1.5}
& $96.2 \pm 2.0$ \diff{-3.5}
& $15.9 \pm 7.1$ \diff{-31.5}
& $37.8 \pm 4.9$ \diff{-11.8}
& $0.0 \pm 0.0$ \diff{-47.4} \\
Closest
& $99.4 \pm 0.1$ \diffbest{-0.4}
& $70.8 \pm 0.3$ \diffbest{-1.3}
& $97.4 \pm 1.4$ \diff{-2.2}
& $53.4 \pm 8.4$ \diffbest{+5.9}
& $43.0 \pm 5.9$ \diffbest{-6.6}
& $47.92 \pm 8.0$ \diffbest{+0.52}\\
\bottomrule
\end{tabular}}
\end{table}

\section{Conclusion}

We studied \emph{localized collateral forgetting}: a discrepancy between unlearning and retraining that concentrates on non-forget examples close to the forget set. Our theoretical analysis shows that this localization arises naturally in gradient-ascent and random-label-based unlearning, and can persist even when retain-set fine-tuning is used as a recovery step.

The analysis identifies target mismatch as a key mechanism: random or corrupted targets on the forget set need not match the local prediction structure induced by retraining. Motivated by this, we proposed \emph{Local Teacher Distillation}, which replaces random targets with soft labels from a small teacher trained on retained neighbors of the forget set. On CIFAR-100 partial-class deletion, LTD reduces discrepancy to retraining while maintaining strong unlearning performance.

\paragraph{Limitations.}
Our method relies on the quality of the local support set and the representation
used to select it. In our experiments, the forget set is relatively homogeneous,
so a nearest-neighbor support around an average forget representation is
effective. More heterogeneous forget sets may require
clustering, multiple local teachers, or adaptive neighborhood construction.
We evaluate controlled image-classification
partial-deletion settings designed to isolate localized collateral forgetting; extension
to more general architectures, datasets, and deletion regimes remains future work.

\section*{Acknowledgements}

This work was supported by the Helmholtz Association’s Initiative and Networking Fund on the
HAICORE@FZJ partition. We gratefully acknowledge funding from the European Research Council (ERC) under the Horizon Europe Framework Programme (HORIZON) for proposal number 101170430 CollectiveMinds. Views and opinions expressed are however those of the authors only and do not necessarily reflect those of the European Union or the European Research Council. Neither the European Union nor the granting authority can be held responsible for them.

\bibliography{references}

\appendix

\section{Proofs}

\subsection{Proof of Proposition~\ref{prop:linreg-ga} (Localized collateral forgetting under gradient ascent)}
\label{sec:proof-GA}

\begin{proof}

We initialize $\theta_0 := \theta_{\mathrm{full}}$ and perform gradient ascent updates on the deleted point:
\[
\theta_{t+1} = \theta_t + 2\eta\, x_f (x_f^\top \theta_t - y_f).
\]
Denote by $e^{\mathrm{full}}$ the prediction error of $\theta_{\mathrm{full}}$:
\[
e^{\mathrm{full}}_q := x_q^\top \theta_{\mathrm{full}} - y_q,
\quad
e^{\mathrm{full}}_f := x_f^\top \theta_{\mathrm{full}} - y_f.
\]

Define the residual at step $t$:
\[
r_t^{(i)} := x_i^\top \theta_t - y_i.
\]
Our goal is to compare the squared prediction loss after gradient ascent with
the retrained reference:
\[
\Delta_{\mathrm{loss}}(x_q)
:=
\mathbb{E}\!\left[\bigl(r_T^{(q)}\bigr)^2 \mid X\right]
-
\mathbb{E}\!\left[(x_q^\top\theta_{\mathrm{retrain}}-y_q)^2 \mid X\right].
\]
We decompose this quantity by adding and subtracting the full-model loss:
\[
\Delta_{\mathrm{loss}}(x_q)
=
\underbrace{
\mathbb{E}\!\left[\bigl(r_T^{(q)}\bigr)^2 \mid X\right]
-
\mathbb{E}\!\left[\bigl(r_0^{(q)}\bigr)^2 \mid X\right]
}_{A_n}
+
\underbrace{
\mathbb{E}\!\left[(x_q^\top\theta_{\mathrm{full}}-y_q)^2 \mid X\right]
-
\mathbb{E}\!\left[(x_q^\top\theta_{\mathrm{retrain}}-y_q)^2 \mid X\right]
}_{B_n}.
\]
The first term \(A_n\) captures the effect of the gradient-ascent updates,
while \(B_n\) is the correction from the full model to retraining.

\paragraph{Dynamics on the deleted point.}
We first track how the gradient-ascent update changes the residual on the deleted example. By substituting the update rule, we get
\[
r^{(f)}_{t+1}
= x_f^\top \theta_{t+1} - y_f
= r^{(f)}_t + 2\eta \|x_f\|^2 r^{(f)}_t
= (1 + 2\eta \|x_f\|^2) r^{(f)}_t.
\]
Thus, gradient ascent amplifies the initial full-model residual on the deleted
point:
\[
r^{(f)}_t = (1 + 2\eta \|x_f\|^2)^t \, e^{\mathrm{full}}_f.
\]

\paragraph{Dynamics on an arbitrary point.} 
We next track how the same update affects a query point \(x_q\). Its residual changes only through the overlap with the deleted direction:
\[
r^{(q)}_{t+1}
= r^{(q)}_t + 2\eta \langle x_q, x_f\rangle r^{(f)}_t.
\]
Substituting the deleted-point dynamics and unrolling the recursion gives
\[
r^{(q)}_T
= e^{\mathrm{full}}_q
+ 2\eta \langle x_q, x_f\rangle
\sum_{s=0}^{T-1} (1 + 2\eta \|x_f\|^2)^s \, e^{\mathrm{full}}_f.
\]

\paragraph{Contribution relative to the full model.}
Define
\[
\alpha_T := 2\eta \langle x_q, x_f\rangle
\sum_{r=0}^{T-1} (1 + 2\eta \|x_f\|^2)^r = \dfrac{\langle x_q, x_f\rangle}{\langle x_f, x_f\rangle}\Big((1 + 2 \eta \; \|x_f\|^2)^T - 1\Big).
\]
Then the query residual after \(T\) gradient-ascent steps can be written as
\[
r^{(q)}_T = e^{\mathrm{full}}_q + \alpha_T e^{\mathrm{full}}_f.
\]
Since \(\theta_0=\theta_{\mathrm{full}}\), we have
\(r^{(q)}_0=e^{\mathrm{full}}_q = x_q^\top\theta_{\mathrm{full}} - y_q\). Hence the change in the conditional squared
prediction error relative to the full model is

\begin{align*}
\mathbb{E}\,[\bigl(r_T^{(q)}\bigr )^2 \mid X] 
&-
\mathbb{E}[(x_q^\top\theta_{\mathrm{full}} - y_q)^2 \mid X] 
=
\mathbb{E}\,[\bigl(r_T^{(q)}\bigr )^2 \mid X] - \mathbb{E}\, [\bigl(r_0^{(q)}\bigr)^2 \mid X]
\\&= 
\mathbb{E}[(e^{\mathrm{full}}_q + \alpha_T e^{\mathrm{full}}_f)^2 \mid X] 
- 
\mathbb{E}[(e^{\mathrm{full}}_q)^2 \mid X]
=
\alpha^2_T\mathbb{E} [(e^{\mathrm{full}}_f)^2 \mid X]
+ 
 2\alpha_T\mathbb{E}[e^{\mathrm{full}}_f e^{\mathrm{full}}_q \mid X ].
\end{align*}

Therefore, the first term in the decomposition is
\[
A_n
=
\alpha_T^2 \mathbb{E}\left[(e_f^{\mathrm{full}})^2 \mid X\right]
+
2\alpha_T \mathbb{E}\left[e_f^{\mathrm{full}} e_q^{\mathrm{full}} \mid X\right].
\]

\paragraph{Asymptotic behavior.}
We now evaluate \(A_n\) and \(B_n\) under the linear model. The goal is to show
that the leading term comes from the gradient-ascent contribution \(A_n\), while
the full-to-retrain correction \(B_n\) is lower order.

Let \(X_{-f}\) and \(y_{-f}\) denote the design matrix and response vector with
the deleted example removed. Since \(n>d\) and \(\Sigma\succ 0\), both \(X^\top X\) and
\(X_{-f}^\top X_{-f}\) are invertible with probability one. In this linear-regression setting,
\[
\theta_{\mathrm{full}}
=
(X^\top X)^{-1}X^\top y,
\qquad
\theta_{\mathrm{retrain}}
=
(X_{-f}^\top X_{-f})^{-1}X_{-f}^\top y_{-f}.
\]
Conditioning on \(X\), the only randomness is the observation noise. Let $H=X(X^\top X)^{-1}X^\top$ and $h_{ij}=H_{ij}$. It is well-known from standard OLS analysis \citep[Sec.~4.2.2]{montgomery2012introduction}

\[
\mathbb{E}[e_i^{\mathrm{full}} e_j^{\mathrm{full}} \mid X]
=
\sigma^2(\mathbf{1}\{i=j\} - h_{ij}).
\]

Thus,
\[
\mathbb{E}[(e_f^{\mathrm{full}})^2 \mid X]
=
\sigma^2(1 - h_{ff}),
\quad
\mathbb{E}[e_f^{\mathrm{full}} e_q^{\mathrm{full}} \mid X]
=
-\sigma^2 h_{fq}.
\]

By the law of large numbers, as the sample size \(n\) grows,
\[
\frac{1}{n}X^\top X \xrightarrow{p} \Sigma,
\]
which implies \(h_{ij}=O_p(1/n)\) for fixed \(i,j\). Therefore,
\begin{align*}
A_n &= 
\alpha^2_T\mathbb{E} [(e^{\mathrm{full}}_f)^2 \mid X] + 
2\alpha_T\mathbb{E}[e^{\mathrm{full}}_f e^{\mathrm{full}}_q | X ] 
=
\alpha_T^2 \sigma^2 + O_p(1/n).
\end{align*}

For $B_n$, by the standard case-deletion identity for OLS \citep{cook1986residuals}, we have
\[
x_q^\top\theta_{\mathrm{retrain}} - y_q
=
e_q^{\mathrm{full}}
+
\frac{h_{qf}}{1-h_{ff}} e_f^{\mathrm{full}}.
\]
Therefore,
\begin{align*}
&B_n = \mathbb{E}\!\left[(x_q^\top\theta_{\mathrm{full}} - y_q)^2 \mid X \right]
-
\mathbb{E}\!\left[(x_q^\top\theta_{\mathrm{retrain}} - y_q)^2 \mid X \right]
=
\frac{\sigma^2 h_{qf}^2}{1-h_{ff}} = O_p(1/n^2),
\end{align*}

Combining the asymptotic expansions of $A_n$ and $B_n$, we obtain
\[
\Delta_{\mathrm{loss}}(x_q)
=
\alpha_T^2 \sigma^2 + O_p(1/n) +  O_p(1/n^2) = \langle x_q, x_f\rangle^2 \;\Big(\dfrac{(1 + 2 \eta \; \|x_f\|^2)^T - 1}{\|x_f\|^2}\Big)^2 \sigma^2 + O_p(1/n),
\]

where 
\[
c(T, \eta, x_f) := \Big(\dfrac{(1 + 2 \eta \; \|x_f\|^2)^T - 1}{\|x_f\|^2}\Big)^2\]
does not depend on $x_q$, and the \(B_n\) term is lower order and is absorbed into the \(O_p(1/n)\) remainder.
\end{proof}

\subsection{Proof of Proposition~\ref{prop:logreg-rl} (Localized collateral forgetting under random-label unlearning)} 
\label{sec:proof-RL}

\begin{assumption}[Conditions for convergence of RL iterates]
\label{ass:sgd-rl}
Let $D_f=\{(x_i,y_i)\}_{i\in I_f}$ be fixed and finite, and let
$\tilde y=(\tilde y_i)_{i\in I_f}$, where
$\tilde y_i\sim \mathrm{Bernoulli}(1/2)$ independently. Define
\[
F(\theta)
=
\mathbb{E}_{\tilde y}\!\left[
\sum_{i\in I_f}\ell(\theta;x_i,\tilde y_i)
\right]
+
\frac{\lambda}{2}\|\theta\|^2 .
\]
Assume:
\begin{enumerate}
    \item \(\lambda>0\), so that \(F\) is strongly convex;
    \item the step sizes satisfy
    \[
    \sum_{t=0}^\infty \eta_t = \infty,
    \qquad
    \sum_{t=0}^\infty \eta_t^2 < \infty;
    \]
    \item the stochastic gradients
    \[
    g_t(\theta)
    :=
    \sum_{i\in I_f}
    \nabla_\theta \ell(\theta;x_i,\tilde y_i^{(t)})
    +
    \lambda \theta
    \]
    are unbiased,
    \[
    \mathbb{E}\!\left[g_t(\theta)\mid \theta\right]=\nabla F(\theta),
    \]
    and have bounded second moment,
    \[
    \mathbb{E}\!\left[\|g_t(\theta)\|^2 \mid \theta\right]
    \le C\bigl(1+\|\theta\|^2\bigr)
    \]
    for some constant \(C>0\).
\end{enumerate}
\end{assumption}

\paragraph{Proof of the proposition.}
\begin{proof}
The gradient descent step on $D_f$ with randomized labels can be written as:

\[
\theta_{t+1}
=
(1-\eta_t\lambda)\theta_t
-
\eta_t\,X_F^\top\bigl(\sigma(z_t^{(F)})-\tilde y^{(t)}\bigr).
\]

Let
\[
G := X_F X_F^\top,
\qquad
k_q := X_F x_q.
\] 

Taking inner products gives
\begin{align*}
z_{t+1}^{(F)}
&=
(1-\eta_t\lambda)z_t^{(F)}
-
\eta_t\,G\bigl(\sigma(z_t^{(F)})-\tilde y^{(t)}\bigr),\\
z_{t+1}^{(q)}
&=
(1-\eta_t\lambda)z_t^{(q)}
-
\eta_t\,k_q^\top\bigl(\sigma(z_t^{(F)})-\tilde y^{(t)}\bigr).
\end{align*}

Matrix $G$ is invertible by $\operatorname{rank}(X)=n$, and we obtain
\[
z_{t+1}^{(q)}-(1-\eta_t\lambda)z_t^{(q)}
=
\alpha_q^\top\bigl(z_{t+1}^{(F)}-(1-\eta_t\lambda)z_t^{(F)}\bigr),
\]
where
\[
\alpha_q := G^{-1}k_q
\]
are the coordinates of the orthogonal projection of $x_q$ onto $\mathrm{span}(X^\top_F)$ in the basis given by the rows of $X_F$.

Unrolling the recursion from $t=0$ to $T-1$ and using the fact that $\theta_0 = \theta_{\mathrm{full}}$ yields
\[
\langle \theta_T,x_q\rangle-\langle \theta_{\mathrm{full}},x_q\rangle
=
(\Gamma_{0:T}-1)z_{0}^{(q)}
+
\alpha_q^\top\bigl(z_{T}^{(F)}-\Gamma_{0:T}z_{0}^{(F)}\bigr).
\]

For the full and retrained references, the retained target logits are the same.
Since both references interpolate these logits,
$X_R\theta_{\mathrm{full}}=X_R\theta_{\mathrm{retrain}}=z_R$. Hence, for
$x_q\in D_r$,
$\langle\theta_{\mathrm{full}},x_q\rangle
=\langle\theta_{\mathrm{retrain}},x_q\rangle$, and therefore
\[
\langle \theta_T,x_q\rangle-\langle \theta_{\mathrm{retrain}},x_q\rangle
=
(\Gamma_{0:T}-1)z_{0}^{(q)}
+
\alpha_q^\top\bigl(z_{T}^{(F)}-\Gamma_{0:T}z_{0}^{(F)}\bigr).
\]
 
\end{proof}

\paragraph{Variance.} Under the setting of Proposition~\ref{prop:logreg-rl}, the only randomness in $\langle \theta_T,x_q\rangle-\langle \theta_{\mathrm{retrain}},x_q\rangle$ comes from the random labels used in the
unlearning dynamics. Therefore,
\[
\mathrm{Var}(\langle \theta_T,x_q\rangle-\langle \theta_{\mathrm{retrain}},x_q\rangle)
=
\alpha_q^\top \,\mathrm{Cov}(z_{T}^{(F)})\,\alpha_q,
\]

This shows that the stochastic variability of RL is not uniform across retained points. It is filtered through the same projection coefficients $\alpha_q$ as the mean discrepancy: points with a larger component in the forget span can exhibit larger variance in their logit discrepancy.

\paragraph{Finite-time regime.}
If
\[
\left|(\Gamma_{0:T}-1)z_0^{(q)}\right|
\le
\varepsilon_T
\left|
\alpha_q^\top\bigl(z_T^{(F)}-\Gamma_{0:T}z_0^{(F)}\bigr)
\right|,
\qquad \varepsilon_T<1,
\]
then the discrepancy is controlled up to relative error $\varepsilon_T$ by the
forget-set perturbation term. In particular, its dependence on $x_q$ is through
the projection coefficients $\alpha_q$.

\paragraph{Asymptotic behavior.} Under Assumption~\ref{ass:sgd-rl}, the RL iterates are stochastic-gradient
iterates for the expected random-label objective
\[
F(\theta)
=
\mathbb{E}_{\tilde y}\!\left[
\sum_{i\in I_f}\ell(\theta;x_i,\tilde y_i)
\right]
+
\frac{\lambda}{2}\|\theta\|^2 .
\]
Since $\mathbb{E}[\tilde y_i]=1/2$, this objective is equivalently
\[
F(\theta)
=
\sum_{i\in I_f}\ell(\theta;x_i,1/2)
+
\frac{\lambda}{2}\|\theta\|^2 .
\]
Moreover,
\[
\nabla F(\theta)
=
X_F^\top\!\left(\sigma(X_F\theta)-\frac12\mathbf{1}\right)
+
\lambda\theta .
\]
Thus $\nabla F(0)=0$, and since $\lambda>0$, $F$ is strongly convex, so its
unique minimizer is $\theta^\star=0$. By standard SGD convergence,
$\theta_t\to 0$ under Assumption~\ref{ass:sgd-rl}. Hence, for any fixed query
$x_q$,
\[
\langle \theta_t,x_q\rangle \to 0,
\]
and, under the corresponding $L^1$ convergence implied by the same bounded-moment
conditions,
\[
\mathbb{E}\langle \theta_t,x_q\rangle \to 0.
\]

Therefore,
\[
\mathbb{E}\left[
\langle\theta_T,x_q\rangle
-
\langle\theta_{\mathrm{retrain}},x_q\rangle
\right]
\to
-\langle\theta_{\mathrm{retrain}},x_q\rangle .
\]

\paragraph{Relation to the single-point case.}
In the single-point setting, the logit dynamics is governed by the scalar coefficient
\[
\alpha_{q} = \frac{\langle x_q, x_f\rangle}{\langle x_f, x_f\rangle}.
\]

In the multi-point case, this coefficient is replaced by the vector
\[
\alpha_q = (X_F X_F^\top)^{-1} X_F x_q,
\]
which can be interpreted as the coordinates of the projection of $x_q$ onto
$\mathrm{span}(X_F^\top)$.

Thus, the scalar similarity $\langle x_q, x_f\rangle$ is replaced by a projection
onto the subspace spanned by the deleted points.

\subsection{General RL+FT identities}~\label{app:rlft-identities}

Let $A=F\cup R$, and let $X_A:=X_{F\cup R}$. Define
\[
z_t^{(A)}:=X_A\theta_t,
\qquad
G_A:=X_AX_A^\top,
\qquad
k_{q,A}:=X_Ax_q,
\qquad
\alpha_{q,A}:=G_A^{-1}k_{q,A}.
\]
Then $X_A^\top\alpha_{q,A}$ is the projection of $x_q$ onto
$\operatorname{span}(X_A^\top)$.

The RL+FT update can be written as
\[
\theta_{t+1}
=
(1-\eta_t\lambda)\theta_t
-
\eta_t X_A^\top g_t,
\]
where $g_t$ contains the logistic-gradient factors on $R$ with true labels and on
$D_f$ with random labels. Hence
\[
z_{t+1}^{(q)}-(1-\eta_t\lambda)z_t^{(q)}
=
\alpha_{q,A}^\top
\left(
z_{t+1}^{(A)}-(1-\eta_t\lambda)z_t^{(A)}
\right).
\]
Unrolling gives
\[
z_T^{(q)}
=
\Gamma_{0:T}z_0^{(q)}
+
\alpha_{q,A}^\top
\left(
z_T^{(A)}-\Gamma_{0:T}z_0^{(A)}
\right).
\]
Therefore, for any query point $x_q$,
\[
\langle\theta_T,x_q\rangle-\langle\theta_{\mathrm{retrain}},x_q\rangle
=
(\Gamma_{0:T}-1)z_0^{(q)}
+
\alpha_{q,A}^\top
\left(
z_T^{(A)}-\Gamma_{0:T}z_0^{(A)}
\right)
+
\langle\theta_{\mathrm{full}}-\theta_{\mathrm{retrain}},x_q\rangle .
\]
For retained training points, the last term vanishes under the interpolation
assumption.

\paragraph{Mean dynamics.}
Taking expectation over the random labels on $D_f$ replaces them by the soft target
$1/2$. Thus the mean dynamics converge, under the corresponding SGD convergence
conditions, to a minimizer of
\[
L_{\mathrm{RL+FT}}(\theta)
=
\sum_{(x_i,y_i)\in R}\ell(\theta;x_i,y_i)
+
\sum_{(x_i,y_i)\in D_f}\ell(\theta;x_i,1/2)
+
\frac{\lambda}{2}\|\theta\|^2 .
\]
We denote this minimizer by $\theta_{\mathrm{RL+FT}}$. In general,
\[
\mathbb{E}\left[\langle\theta_T,x_q\rangle\right]
\to
\langle\theta^{(\lambda)}_{\mathrm{RL+FT}},x_q\rangle ,
\]
which need not equal $\langle\theta_{\mathrm{retrain}},x_q\rangle$.

\subsection{Proof of Proposition~\ref{prop:logreg-rl-ft}}

\begin{proof}
Assume \(R=D_r\). For fixed \(\lambda>0\), Assumption~\ref{ass:sgd-rl}
implies that the stochastic trajectory converges to the limiting surrogate
solution \(\theta_{\mathrm{RL+FT}}^{(\lambda)}\). In the limit
\(\lambda\to0^+\), this solution agrees with retraining on the retained training
logits. Hence, for every retained training point \((x_q,y_q)\in D_r\),
\[
\lim_{\lambda\to0^+}\lim_{T\to\infty}
\mathbb{E}\left[
\langle\theta_T,x_q\rangle
-
\langle\theta_{\mathrm{retrain}},x_q\rangle
\right]
=0 .
\]

Now assume that \(\theta_{\mathrm{retrain}}\) is the minimum-norm interpolating
solution on \(D_r\) and that \(x_q\in\operatorname{span}(X^\top)\). Write
\[
x_q=X_R^\top v_R+X_F^\top v_F .
\]
Since \(R=D_r\), in the limit \(\lambda\to0^+\) the surrogate solution agrees
with retraining on \(X_R\), while the random-label part imposes zero logits on
\(X_F\) in expectation. Therefore,
\[
\lim_{\lambda\to0^+}
\langle\theta^{(\lambda)}_{\mathrm{RL+FT}},x_q\rangle
=
\langle\theta_{\mathrm{retrain}},X_R^\top v_R\rangle .
\]
Using
\[
\langle\theta_{\mathrm{retrain}},x_q\rangle
=
\langle\theta_{\mathrm{retrain}},X_R^\top v_R\rangle
+
\langle\theta_{\mathrm{retrain}},X_F^\top v_F\rangle,
\]
we obtain
\[
\lim_{\lambda\to0^+}
\langle\theta^{(\lambda)}_{\mathrm{RL+FT}},x_q\rangle
=
\langle\theta_{\mathrm{retrain}},x_q\rangle
-
\langle\theta_{\mathrm{retrain}},X_F^\top v_F\rangle .
\]
\end{proof}

Thus, full retain fine-tuning can restore agreement with retraining on retained training points, but this does not imply agreement on general query or test points. For such points, the component aligned with the forget span can still create a discrepancy, as shown by the term $\langle \theta_{\mathrm{retrain}},X_F^\top v_F\rangle$. Hence retain-set recovery alone is insufficient: mitigating localized collateral forgetting requires improving the signal injected on $D_f$ itself.

\subsection{Proof for the Local Support Selection ~\ref{par:support-selection}.}

\begin{proposition}\label{prop:distill-labels}
Under the same setting as Proposition~\ref{prop:logreg-rl-ft}, let \(x_f\) be
the deleted point and \(x_s\in R\) a support point used to construct a
pseudo-label for \(x_f\). Let \(\theta_{\mathrm{retrain}}\) be the retrained
solution. For fixed \(\lambda>0\), let
\(\theta_{\mathrm{unlearn}}^{(\lambda)}\) be the limiting solution of the
corresponding pseudo-label unlearning objective, and set
\[
\theta_{\mathrm{unlearn}}
:=
\lim_{\lambda\to0^+}\theta_{\mathrm{unlearn}}^{(\lambda)} .
\]
Then for any query point \(x_q\),
\[
\Delta :=
\left|
\langle \theta_{\mathrm{unlearn}}, x_q \rangle
-
\langle \theta_{\mathrm{retrain}}, x_q \rangle
\right|
=
\left|
\alpha_{q,F}\,(\theta_{\mathrm{retrain}})^\top (P_{x_s}x_f - x_f)
\right|,
\]
where \(P_{x_s} = x_s x_s^\top/\|x_s\|^2\).
\end{proposition}

\begin{corollary}\label{cor:distill-labels}
Let
\[
r_s := x_f - P_{x_s}x_f.
\]
Then
\[
\Delta
=
\bigl|\alpha_{q,F}\,(\theta_{\mathrm{retrain}})^\top r_s\bigr|
\le
|\alpha_{q,F}|\,\|\theta_{\mathrm{retrain}}\|\,\|r_s\|.
\]
Moreover,
\[
\|r_s\|^2
=
\|x_f\|^2\bigl(1-\cos^2(x_f,x_s)\bigr).
\]
Thus, support points with larger cosine similarity to \(x_f\) yield a smaller upper bound on \(\Delta\). In particular, nearest neighbors of \(x_f\) in cosine similarity provide a principled choice of support examples for constructing the teacher model.
\end{corollary}

\begin{proof}
Let
\[
X_R := \begin{bmatrix} x_{1} & \cdots & x_{f-1} & x_{f+1}& \cdots & x_{n-1} \end{bmatrix}^{\top}, \quad\quad X_{F\cup R}:=
\begin{bmatrix}
x_f^\top\\
X_R
\end{bmatrix}
\in\mathbb{R}^{n\times d}.
\]

Then
\[
k_{q,F\cup R}
=
X_{F\cup R}x_q, 
\quad 
G_{F\cup R}
=
X_{F\cup R}X_{F\cup R}^\top.
\]

The coefficients of the orthogonal projection of $x_q$ onto $\operatorname{span}(X_{F\cup R}^\top)$ are given by
\[
\alpha_q
=
G_{F\cup R}^{-1} k_{q,F\cup R}.
\]

\paragraph{Helper model trained on $x_s$.}

Since the helper model is trained only on $x_s$ and we consider the minimum-norm solution, it is collinear with $x_s$, hence there exists $\beta \in \mathbb{R}$ such that
\[
\theta_s^* = \beta\, x_{s}.
\]

Let $\theta_{\mathrm{full}}$ be the minimum-norm logistic regression solution trained on the full dataset $D$. Matching logits on $x_s$ gives
\[
\langle \theta_s^*, x_s \rangle
=
\beta\, x_{s}^\top x_{s}
=
\langle \theta_{\mathrm{full}}, x_s \rangle
=
x_{s}^{\top} \theta_{\mathrm{full}},
\]
so
\[
\beta = \frac{x_{s}^{\top} \theta_{\mathrm{full}}}{\|x_{s}\|^2},
\qquad
\theta_s^* =
\frac{x_{s}^{\top} \theta_{\mathrm{full}}}{\|x_s\|^2}x_s.
\]

\paragraph{Transferred logits.}

Applying the helper model to $F$:
\[
z_{\infty}^{(F)}
=
\langle \theta_s^*, x_f \rangle
=
\frac{x_s^{\top} \theta_{\mathrm{full}}}{\|x_s\|^2}\, x_s^\top x_f.
\]

For $R$ (true labels are preserved):
\[
z_{\infty}^{(R)}
=
X_R \theta_{\mathrm{full}}
\]

Thus,
\[
z_{\infty}^{(F \cup R)} =
\begin{bmatrix}
z_{\infty}^{(F)} \\
z_{\infty}^{(R)}
\end{bmatrix}.
\]

\paragraph{Prediction on $x_q$.}

\[
\alpha_{q}^\top z_{\infty}^{(F \cup R)}
=
\alpha_{q,F}^\top z_{\infty}^{(F)} + \alpha_{q,R}^\top z_{\infty}^{(R)},
\]
where
\[
\alpha_q=
\begin{bmatrix}
\alpha_{q,F}\\
\alpha_{q,R}
\end{bmatrix},
\qquad
\alpha_{q,F}\in\mathbb{R}^{|F|},
\quad
\alpha_{q,R}\in\mathbb{R}^{|R|},
\]
with $\alpha_{q,F}$ containing the first $|F|$ coordinates of $\alpha_q$, and $\alpha_{q,R}$ containing the remaining $|R|$ coordinates.

Substituting:
\[
\alpha_{q}^\top z_{\infty}^{(F \cup R)}
=
\alpha_{q, F}^\top\frac{x_s^{\top} \theta_{\mathrm{full}}}{\|x_s\|^2}\, x_s^\top x_f
+
\alpha_{q, R}^\top X_R\theta_{\mathrm{full}}
\]

\paragraph{Deviation from the retrained model.}

By definition, we have 
\[
\Delta :=
\left|
\langle \theta_{\mathrm{unlearn}}, x_q \rangle
-
\langle \theta_{\mathrm{retrain}}, x_q \rangle
\right|=|
\alpha_{q}^\top z_{\infty}^{(F \cup R)}
-
\langle \theta_{\mathrm{retrain}}, x_q \rangle|.
\]

Since the coordinates of $\alpha_q$ are the coefficients of the orthogonal projection of $x_q$ onto $\mathrm{span}\{D\}$, and $\theta_{\mathrm{retrain}} \in \mathrm{span}\{R\}$, we have 

\begin{align*}
\langle \theta_{\mathrm{retrain}}, x_q \rangle
&= \langle \theta_{\mathrm{retrain}}, X_{F \cup R}^\top \alpha_q \rangle
= \langle \theta_{\mathrm{retrain}}, x_f \alpha_{q, F} + X_R^\top\alpha_{q, R} \rangle \\
=&  \alpha_{q, F}^{\top} x_f^{\top}\theta_{\mathrm{retrain}} + \alpha^{\top}_{q, R} X_R \theta_{\mathrm{retrain}}
=\alpha_{q, F}^{\top} \left(P_{X^{\top}_R}x_f\right)^{\top}\theta_{\mathrm{full}} + \alpha^{\top}_{q, R} X_R \theta_{\mathrm{full}},
\end{align*}

where $P_{X^\top_R}x_f := \Pi_{\operatorname{span}(X^\top_R)}(x_f)$.

Then

\begin{align*}
\Delta
&=|\alpha_{q}^\top z_{\infty}^{(F \cup R)}
-
\langle \theta_{\mathrm{retrain}}, x_q \rangle|
\\&=
\Big|\Big(
\alpha_{q, F}^\top\frac{x_s^{\top} \theta_{\mathrm{full}}}{\|x_s\|^2}\, x_s^\top x_f
+
\alpha_{q, R}^\top X_R\theta_{\mathrm{full}}
\Big)
-
\Big(
\alpha_{q, F}^{\top} \left(P_{X^\top_R}x_f\right)^{\top}\theta_{\mathrm{full}} + \alpha^{\top}_{q, R} X_R \theta_{\mathrm{full}}
\Big)
\Big|
\\&=
\Big|\alpha_{q, F}^\top \Big(\frac{x_s^{\top} \theta_{\mathrm{full}}}{\|x_s\|^2}\, x_s^\top x_f - \left(P_{X^\top_R}x_f\right)^{\top}\theta_{\mathrm{full}} \Big) \Big|
\\&= 
\Big|\alpha_{q, F}^\top \theta_{\mathrm{full}}^{\top} \Big(\frac{x_sx_s^\top}{\|x_s\|^2}\,  x_f - P_{X^\top_R}x_f \Big)\Big| 
= \Big|\alpha_{q, F}^\top \theta_{\mathrm{full}}^{\top} \Big(P_{x_s}x_f - P_{X^\top_R}x_f \Big)\Big|
\\&= \Big|\alpha_{q, F} \theta_{\mathrm{retrain}}^{\top} \Big(P_{x_s}x_f - x_f \Big)\Big|,
\end{align*}
where we used \(P_{x_s}x_f\in\operatorname{span}(X_R^\top)\) because
\(x_s\in R\), and \(\theta_{\mathrm{retrain}}=P_R\theta_{\mathrm{full}}\) for
the minimum-norm retrained interpolant.

Let
\[
r_s := x_f - P_{x_s}x_f.
\]
Then
\[
\Delta = \,|\alpha_{q,F}\,\theta_{\mathrm{retrain}}^{\top} r_s|.
\]
Hence, the deviation depends on three factors: the coefficient $\alpha_{q,F}$, the retrained parameter vector $\theta_{\mathrm{retrain}}$, and the residual $r_s$ of approximating $x_f$ by the one-dimensional direction $x_s$.

Moreover,
\[
\|r_s\|^2
=
\|x_f - P_{x_s}x_f\|^2
=
\|x_f\|^2 - \|P_{x_s}x_f\|^2
=
\|x_f\|^2\bigl(1-\cos^2(x_f,x_s)\bigr),
\]
so the norm of the residual is small when $x_s$ is well aligned with $x_f$. At the same time, the deviation is not determined by $\|r_s\|$ alone, but by its interaction with $\theta_{\mathrm{retrain}}$ through the inner product $\theta_{\mathrm{retrain}}^\top r_s$.
Thus, even for a small residual, the deviation may vary depending on how strongly this residual correlates with the retrained solution.

Importantly, $\alpha_{q,F}$ depends only on the geometry of $x_f$, $X_R$, and $x_q$, and does not depend on the choice of $x_s$.
In particular, selecting $x_s$ close to $x_f$ affects only the residual term and cannot make $\alpha_{q,F}$ larger.
Note that for $x_q \in R$ one has $\alpha_{q,F}=0$, whereas for unseen test points $x_q \notin \operatorname{span}(R)$ this coefficient is generally nonzero.
\end{proof}

\section{Experimental details}

\subsection{Heterogeneous forget sets}
\label{app:heterogeneous-support}

Our main experiments use partial-class deletion, where \(D_f\) is relatively
homogeneous and support is selected by similarity to the average forget
embedding. When \(D_f\) contains several distinct modes, the average embedding
may mix different local neighborhoods. In this case, support selection can be
made more local. One option is per-example nearest-neighbor selection,
\[
S_k(D_f)=\bigcup_{(x_f,y_f)\in D_f} N_k(x_f),
\]
where \(N_k(x_f)\subset D_r\) contains the \(k\) retained examples most similar
to \(x_f\) in representation space. Alternatively, one can \emph{cluster} \(D_f\) and
select retained neighbors for each cluster center. These variants preserve the
teacher-training and LTD unlearning objectives, but change the support
construction and cost.

For clustered support, the teacher can either be trained on the union of the
selected neighborhoods or trained separately per cluster. The latter can be
parallelized and may be preferable when the clusters correspond to distinct local
prediction structures, since a single small teacher may be less effective at
fitting all clusters simultaneously.

Support selection can be implemented exactly by scoring all retained examples or
approximately using approximate nearest-neighbor search in representation space.
Exact average-embedding selection costs \(O(|D_r|p)\) after embeddings are
computed, while per-example selection costs \(O(|D_f||D_r|p)\) without indexing;
ANN methods can reduce this cost in large-scale settings.

\subsection{Hyperparameters}
\label{app:hyperparameters}

\paragraph{Hyperparameter selection.}
Hyperparameters are selected without access to retraining, using the same rule
for all methods. High retain accuracy is a useful retrain-free sanity check in this setting,
since overparameterized image-classification models typically fit the retained
training data well both before and after retraining. Among configurations with affected-class retain accuracy above
\(70\%\) and overall retain accuracy above \(90\%\), we choose the one with the
lowest accuracy on \(D_f\); if none satisfy these constraints, we choose the one
with the highest overall retain accuracy. For LTD, using the teacher accuracy on
\(D_f\) as a proxy target leads to the same selected regime, close to the best
RA/teacher-UA trade-off.

\paragraph{Main experiments (CIFAR-100).} For the full and retrained models, we use the ResNet-56 architecture from the
PyTorch CIFAR Models repository and follow its CIFAR training protocol
\cite{chenyaofo_pytorch_cifar_models}. Specifically, we train for 200 epochs with
SGD, batch size 256, learning rate 0.1, momentum 0.9, weight decay
$5\cdot 10^{-4}$, Nesterov momentum, and cosine learning-rate decay to zero. We
select the checkpoint with the best test accuracy.

For GA, FT, RL, SalUn, and AMUN, we perform the unlearning updates with SGD using
momentum $0.9$ and weight decay $10^{-6}$, following the corresponding public implementations.
For Influence Unlearning, we use the implementation released with the SalUn paper.

We tune the main method-specific hyperparameters as follows. For GA, we use 5
unlearning epochs and tune the learning rate in $[10^{-6},10^{-2}]$. For FT and
RL, we use 20 unlearning epochs and tune the learning rate in $[10^{-3},10^{-1}]$.
For Influence Unlearning, we tune the WoodFisher inverse-Hessian scaling parameter
$\alpha$ in $[1,20]$. For SalUn, we use 20 unlearning epochs, tune the learning
rate in $[10^{-3},10^{-1}]$, and tune the mask ratio in $[0.1,0.9]$. For AMUN, we
use 20 unlearning epochs, tune the learning rate in $[10^{-3},10^{-1}]$, and
evaluate both the standard and \texttt{advonly} variants.

For our method, we use 20 unlearning epochs and tune the learning rate in
$[10^{-6},10^{-3}]$. We set $\beta=2$, selected from $[1,5]$ using early
training dynamics, to ensure that the forget-set distillation signal is not
overwhelmed by retain-set updates. Following the retrain-free support-size
diagnostic in Appendix~\ref{app:support-size-retrain-free}, we use a support set
of size 1000, which approximately corresponds to selecting examples above a similarity
threshold of $0.75$; the same threshold regime works for the additional affected classes. We train a ResNet-8 local teacher with SGD until it reaches 0.99
training accuracy. The support size and teacher-accuracy threshold are justified by the ablations. For the final student update, we use AdamW rather than SGD, since the goal is to move quickly toward the
teacher-provided soft targets on the forget set.
To control for optimizer effects, we also tested AdamW for the baseline methods (GA, FT, RL, SalUn, and AMUN), but this did not improve their results; in most cases, performance was worse than with their standard SGD-based updates.

\paragraph{MNIST.} For the MNIST illustration, we delete class $9$ and average results over 10 random repeats. We use 300 unlearning steps for RL and 120 steps for GA, with learning rate $5\cdot 10^{-5}$ and weight decay $10^{-6}$ for both methods. Error
bars show 95\% confidence intervals for the class-wise accuracy drop, computed across random repeats using the Student-$t$ interval.

\paragraph{Synthetic illustrations.}
For the linear-regression GA illustration, we delete one point and evaluate the
MSE shift on the retain set. We use $n_{\mathrm{train}}=5000$, $d=10$, noise variance $\sigma^2=0.1$, 450 unlearning
steps, and unlearning learning rate $0.05$. Results are averaged over 100 independent runs.

For the logistic-regression RL illustration, we delete one point and sample
$n_{\mathrm{test}}=10000$ query inputs from $\operatorname{span}(X_R^\top)$, where $X_R$ is the retained-input matrix. The corresponding query logits are defined using the same linear combinations of the retained logits. We use $n_{\mathrm{train}}=15$, $n_{\mathrm{delete}}=1$, $d=20$, 500 unlearning steps, and unlearning learning rate $0.01$. Results are averaged over 100 independent runs.

\paragraph{Additional experiments (SVHN, ViT-Tiny).}
For the full and retrained models we use ViT-Tiny~\citep{DBLP:journals/corr/abs-2010-11929} (\texttt{vit\_tiny\_patch16\_224} from \texttt{timm}~\cite{rw2019timm}), adapted for $32{\times}32$ inputs with patch size~$4$ (yielding 64 patch tokens). We train from scratch for 100 epochs with AdamW (lr$=10^{-3}$, weight decay $0.05$) and cosine learning-rate decay to zero, batch size 256, selecting the checkpoint with the best validation accuracy. For the baseline unlearning methods we search the learning rate within the same ranges as in the CIFAR-100 experiments, using 10 unlearning epochs instead. For our method, we use 10 unlearning epochs with AdamW (lr$=10^{-4}$) and forget-loss weight $\beta=3$.
The ResNet-8 local teacher is trained with SGD (lr$=0.005$, cosine decay) until it reaches 0.99 training accuracy. For $50\%$ deletion we use a support set of size $n_s=4500$; for $90\%$ we
use $n_s=1200$.

\paragraph{Compute resources.}
All experiments were run on a single NVIDIA A100 GPU with \texttt{num\_workers=4} for data loading. Per-run wall-clock runtimes are
reported as RTE in the experimental tables.

\paragraph{Existing assets.}
We use CIFAR-100~\citep{krizhevsky2009learning}, MNIST~\citep{lecun1998mnist}, and the ResNet-56 implementation from PyTorch CIFAR Models~\cite{chenyaofo_pytorch_cifar_models}, which is released under the BSD-3-Clause license.  We access these models through the public repository and do not redistribute modified versions. We also use public baseline implementations where available and cite the corresponding works.

\section{Ablations}~\label{app:ablations}

We ablate four components of the proposed method: support selection (reported in the main text), teacher architecture, teacher quality, and support size. Each ablation is reported as a separate subsection to isolate the contribution of the corresponding design choice.

\subsection{Teacher Type}

We study the effect of the teacher model class. We compare lightweight image-space models (ResNet-8 and a shallow CNN), the full model (ResNet-56), and an MLP trained on frozen embeddings of the full model. This ablation evaluates how the choice of model family and representation affects approximation to retraining.

\begin{table}[H]
\caption{Teacher model ablation: standard metrics and runtime. Values are mean $\pm$ std over 5 runs.}
\label{tab:teacher_model_ablation_standard}
\centering
\small
\setlength{\tabcolsep}{5pt}
\begin{tabular}{lccc}
\multicolumn{1}{c}{\bf Method} &
\multicolumn{1}{c}{\bf RA} &
\multicolumn{1}{c}{\bf TA} &
\multicolumn{1}{c}{\bf RTE}
\\ \toprule
Retrain
& $99.9 \pm 0.0$
& $72.1 \pm 0.2$
& $1548.8 \pm 11.6$ \\
\midrule
Small CNN
& $99.4 \pm 0.2$ \diffbest{-0.4}
& $70.7 \pm 0.2$ \diff{-1.3}
& $224.6 \pm 6.1$ \\
ResNet-8
& $99.4 \pm 0.1$ \diffbest{-0.4}
& $70.8 \pm 0.3$ \diff{-1.3}
& $226.5 \pm 5.2$ \\
ResNet-56
& $99.5 \pm 0.1$ \diffbest{-0.4}
& $70.8 \pm 0.3$ \diff{-1.3}
& $361.6 \pm 31.6$ \\
MLP Frozen
& $99.5 \pm 0.1$ \diffbest{-0.4}
& $71.0 \pm 0.5$ \diffbest{-1.1}
& $191.2 \pm 2.1$ \\
\bottomrule
\end{tabular}
\end{table}

\begin{table}[H]
\caption{Teacher model ablation: affected-class metrics and teacher accuracy on $D_f$. Values are mean $\pm$ std over 5 runs.}
\label{tab:teacher_model_ablation_affected}
\centering
\small
\setlength{\tabcolsep}{4pt}
\begin{tabular}{lcccc}
\multicolumn{1}{c}{\bf Method} &
\multicolumn{1}{c}{\bf RA$_{\text{affected-class}}$} &
\multicolumn{1}{c}{\bf UA$_{\text{affected-class}}$} &
\multicolumn{1}{c}{\bf TA$_{\text{affected-class}}$} &
\multicolumn{1}{c}{\bf Teacher UA on $D_f$}
\\ \toprule
Retrain
& $99.7 \pm 0.3$
& $47.4 \pm 3.1$
& $49.6 \pm 5.2$
& $47.4 \pm 3.1$ \\
\midrule
Small CNN
& $98.8 \pm 0.9$ \diff{-0.9}
& $54.4 \pm 6.4$ \diff{+7.0}
& $46.6 \pm 2.9$ \diffbest{-3.0}
& $43.2 \pm 9.6$ \diff{-4.2} \\
ResNet-8
& $97.4 \pm 1.4$ \diff{-2.2}
& $53.4 \pm 8.4$ \diffbest{+5.9}
& $43.0 \pm 5.9$ \diff{-6.6}
& $47.9 \pm 8.0$ \diffbest{+0.5} \\
ResNet-56
& $97.5 \pm 1.2$ \diff{-2.2}
& $53.3 \pm 6.5$ \diffbest{+5.9}
& $41.8 \pm 6.5$ \diff{-7.8}
& $42.6 \pm 9.0$ \diff{-4.8} \\
MLP Frozen
& $99.6 \pm 0.4$ \diffbest{-0.1}
& $99.4 \pm 0.2$ \diff{+52.0}
& $65.2 \pm 4.2$ \diff{+15.6}
& $98.5 \pm 0.4$ \diff{+51.1} \\
\bottomrule
\end{tabular}
\end{table}

We observe in the Tables~\ref{tab:teacher_model_ablation_standard}--\ref{tab:teacher_model_ablation_affected} that the choice of teacher architecture within image-space models has only a minor effect on performance: small CNNs, ResNet-8, and ResNet-56 achieve comparable results across both aggregate and affected-class metrics. At the same time, smaller models are significantly more efficient to train, offering a favorable trade-off between accuracy and runtime.

In contrast, the MLP trained on frozen embeddings fails to forget, exhibiting near-perfect accuracy on the forget set and large deviations from retraining. This is expected, as the embedding representation is produced by the full model and already encodes the forgotten data, preventing effective unlearning.

\subsection{Teacher versus direct neighbor targets}
\label{app:teacher-vs-neighbor-targets}

We additionally test whether the local teacher is necessary, or whether it is
sufficient to construct forget-set targets directly from nearby retained examples.
For this ablation, we keep the same locality principle but remove the teacher
training step. Instead, the neighbor-label variants construct soft targets for
$D_f$ directly from the labels of nearby retained examples, using neighborhoods
of size $k=500$ or $k=1000$. Thus, this comparison isolates the role of the
teacher: both approaches use local retained information, but LTD first trains a
small model on the selected support set and then uses its predictions as soft
targets on $D_f$.

Tables~\ref{tab:neighbor_main_results}--\ref{tab:neighbor_affected_class_results}
show that direct neighbor labels preserve overall RA and TA reasonably well, but
do not match the retraining behavior on $D_f$. In particular, the neighbor
variants substantially overestimate UA compared to retraining, which leads to
larger aggregate and affected-class Avg. Gap. Increasing the neighborhood size
from $k=500$ to $k=1000$ improves alignment, but the gap remains considerably
larger than for LTD.

These results suggest that locality alone is not sufficient. Directly using nearby retained labels gives a useful local signal, but it does
not capture the retraining-induced prediction structure as well as a trained
local teacher. The teacher appears to act as a smoother local approximation: it uses the same
neighboring retained data, but converts them into soft targets that better match
the behavior of the retrained model near the forget set.

\begin{table}[H]
\caption{Main results for the neighbor-label variants. Values are mean $\pm$ std; All accuracy metrics are in $\%$. Blue values show deviation from Retrain. Avg. Gap is the mean absolute deviation from Retrain over UA, RA, TA, and MIA.}
\label{tab:neighbor_main_results}
\centering
\small
\resizebox{\linewidth}{!}{%
\setlength{\tabcolsep}{4pt}
\begin{tabular}{lcccccc}
\multicolumn{1}{c}{\bf Method} &
\multicolumn{1}{c}{\bf UA} &
\multicolumn{1}{c}{\bf RA} &
\multicolumn{1}{c}{\bf TA} &
\multicolumn{1}{c}{\bf MIA} &
\multicolumn{1}{c}{\bf Avg. Gap} &
\multicolumn{1}{c}{\bf RTE}
\\ \toprule
Retrain & $47.4 \pm 3.2$ & $99.9 \pm 0.0$ & $72.1 \pm 0.2$ & $71.3 \pm 3.1$ & 0 & $1723.9 \pm 38.6$ \\
\midrule
Neighbor ($k=500$) & $83.3 \pm 4.2$ \diff{+35.9} & $99.6 \pm 0.1$ \diffbest{-0.3} & $71.1 \pm 0.2$ \diffbest{-1.0} & $94.7 \pm 2.2$ \diff{+23.4} & \valuenotbest{$15.1$} & $186.8 \pm 0.7$ \\
Neighbor ($k=1000$) & $67.1 \pm 7.4$ \diff{+19.6} & $99.5 \pm 0.1$ \diff{-0.3} & $70.7 \pm 0.3$ \diff{-1.3} & $99.5 \pm 0.4$ \diff{+28.2} & \valuenotbest{$12.4$} & $185.8 \pm 1.6$ \\
\midrule
LTD (ours) & $53.4 \pm 8.4$ \diffbest{+6.0} & $99.4 \pm 0.1$ \diff{-0.5} & $70.8 \pm 0.3$ \diff{-1.3} & $69.4 \pm 9.4$ \diffbest{-1.9} & \valuebest{$2.4$} & $226.5 \pm 5.2$ \\
\bottomrule
\end{tabular}}
\end{table}

\begin{table}[H]
\caption{Results on the affected-class subset for the neighbor-label variants. Values are mean $\pm$ std; Retrain is averaged over 5 runs. All accuracy metrics are in $\%$. Blue values show deviation from Retrain. Avg. Gap is computed over RA$_{\text{affected-class}}$, UA$_{\text{affected-class}}$, and TA$_{\text{affected-class}}$.}
\label{tab:neighbor_affected_class_results}
\centering
\small
\scalebox{0.9}{%
\setlength{\tabcolsep}{6pt}
\begin{tabular}{lcccc}
\multicolumn{1}{c}{\bf Method} &
\multicolumn{1}{c}{\bf RA$_{\text{affected-class}}$} &
\multicolumn{1}{c}{\bf UA$_{\text{affected-class}}$} &
\multicolumn{1}{c}{\bf TA$_{\text{affected-class}}$} &
\multicolumn{1}{c}{\bf Avg. Gap}
\\ \toprule
Retrain & $99.7 \pm 0.3$ & $47.4 \pm 3.1$ & $49.6 \pm 5.2$ & 0 \\
\midrule
Neighbor ($k=500$) & $98.8 \pm 0.0$ \diffbest{-0.9} & $83.3 \pm 4.2$ \diff{+35.9} & $46.7 \pm 1.7$ \diffbest{-2.9} & \valuenotbest{$13.2$} \\
Neighbor ($k=1000$) & $98.1 \pm 1.1$ \diff{-1.5} & $67.1 \pm 7.4$ \diff{+19.6} & $42.0 \pm 4.2$ \diff{-7.6} & \valuenotbest{$9.6$} \\
\midrule
LTD (Ours) & $97.4 \pm 1.4$ \diff{-2.2} & $53.4 \pm 8.4$ \diffbest{+5.9} & $43.0 \pm 5.9$ \diff{-6.6} & \valuebest{$4.9$} \\
\bottomrule
\end{tabular}}
\end{table}

\subsection{Teacher Confidence}

In our algorithm, the teacher is trained on the support set until it reaches a
target training accuracy threshold on the support set. We vary this threshold and study its effect on unlearning quality, localized collateral forgetting mitigation, and runtime (Figures~\ref{fig:threshold-ablations-standard}--~\ref{fig:threshold-ablations-num-epochs}).

\begin{figure}[H]
    \centering
    \includegraphics[width=1.\linewidth]{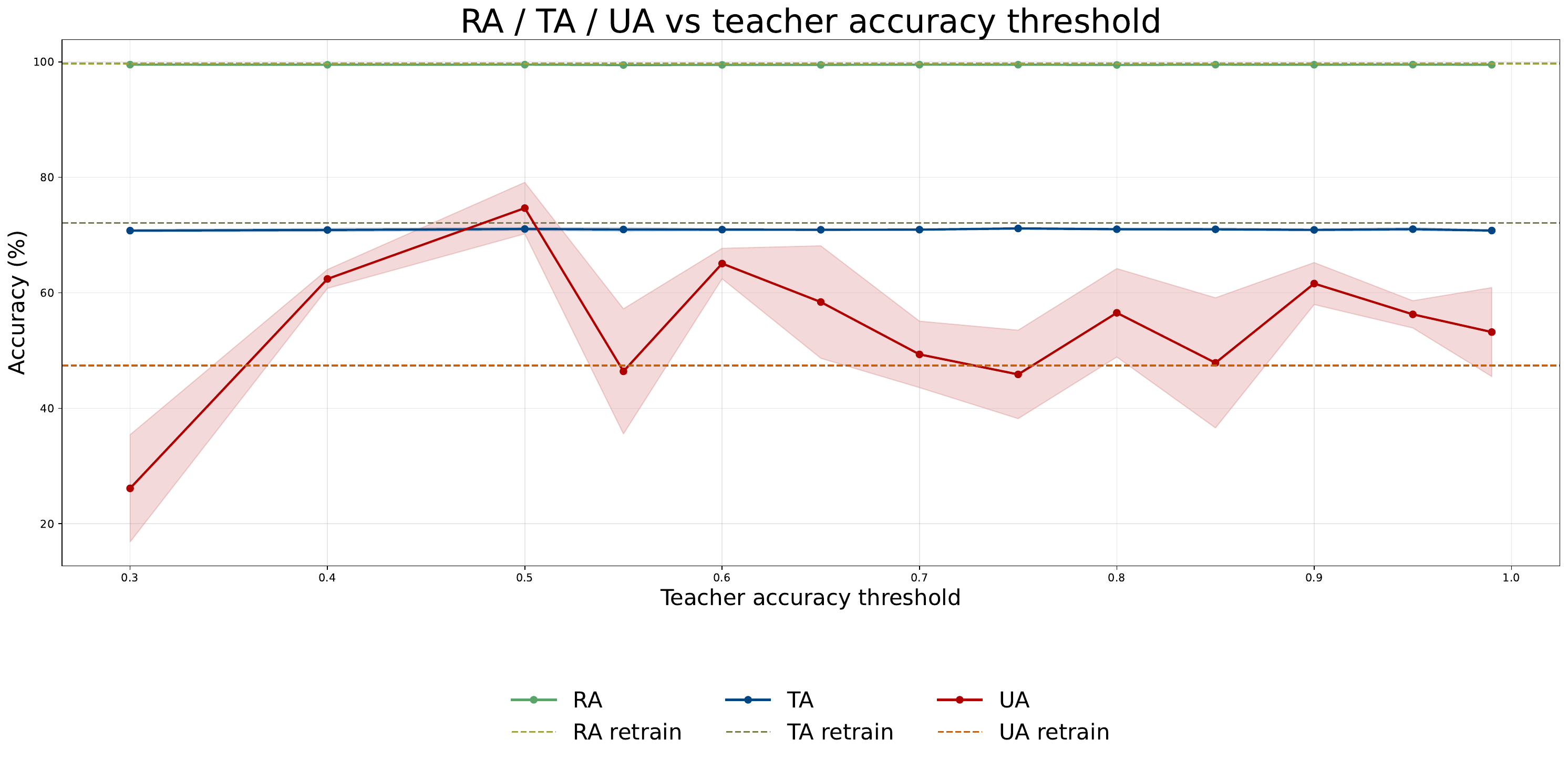}
    \caption{Standard metrics (RA, TA, UA) as a function of the teacher accuracy threshold. Performance remains stable across the considered range.}
    \label{fig:threshold-ablations-standard}
\end{figure}

\begin{figure}[H]
    \centering
    \includegraphics[width=1.\linewidth]{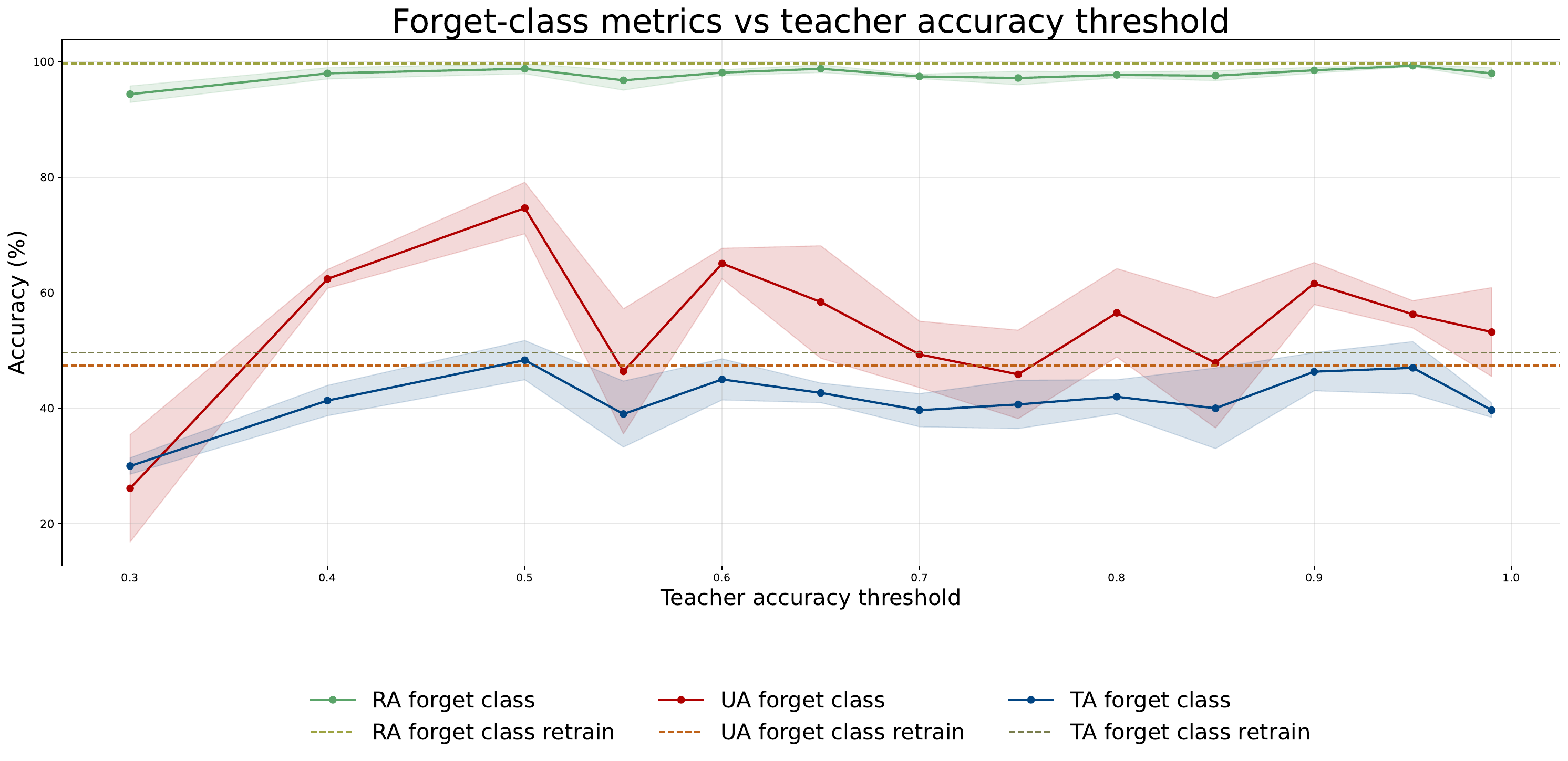}
    \caption{Affected-class metrics as a function of the teacher accuracy threshold. Metrics quickly saturate, indicating that moderate teacher accuracy is sufficient to approximate retraining.}
    \label{fig:threshold-ablations-class}
\end{figure}

\begin{figure}[H]
    \centering
    \includegraphics[width=1.\linewidth]{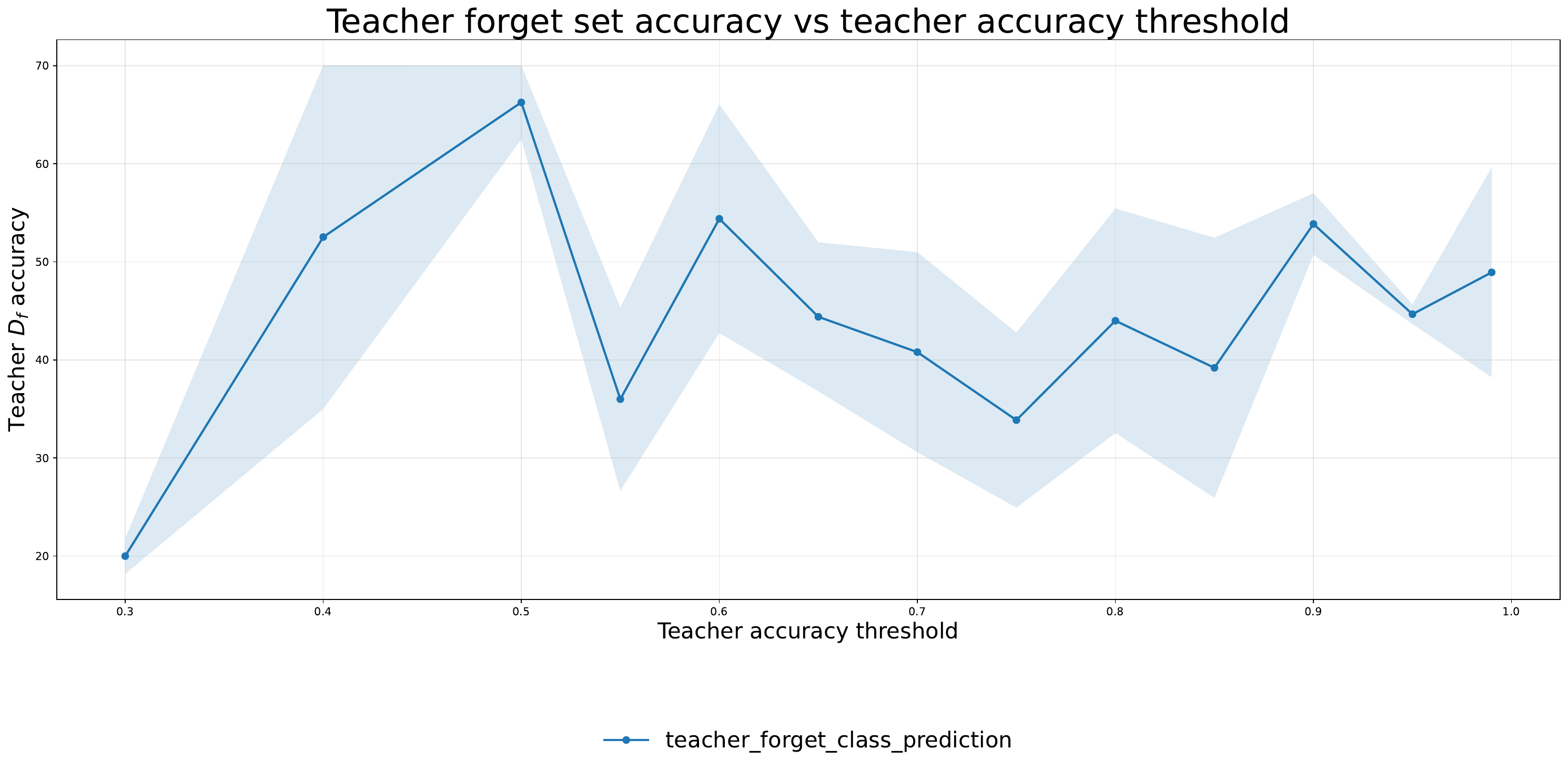}
    \caption{Teacher accuracy on the forget set ($D_f$) as a function of the training threshold. The teacher increasingly matches the retrained model as the threshold grows.}
    \label{fig:threshold-ablations-teacher-prediction}
\end{figure}

\begin{figure}[H]
    \centering
    \includegraphics[width=1.\linewidth]{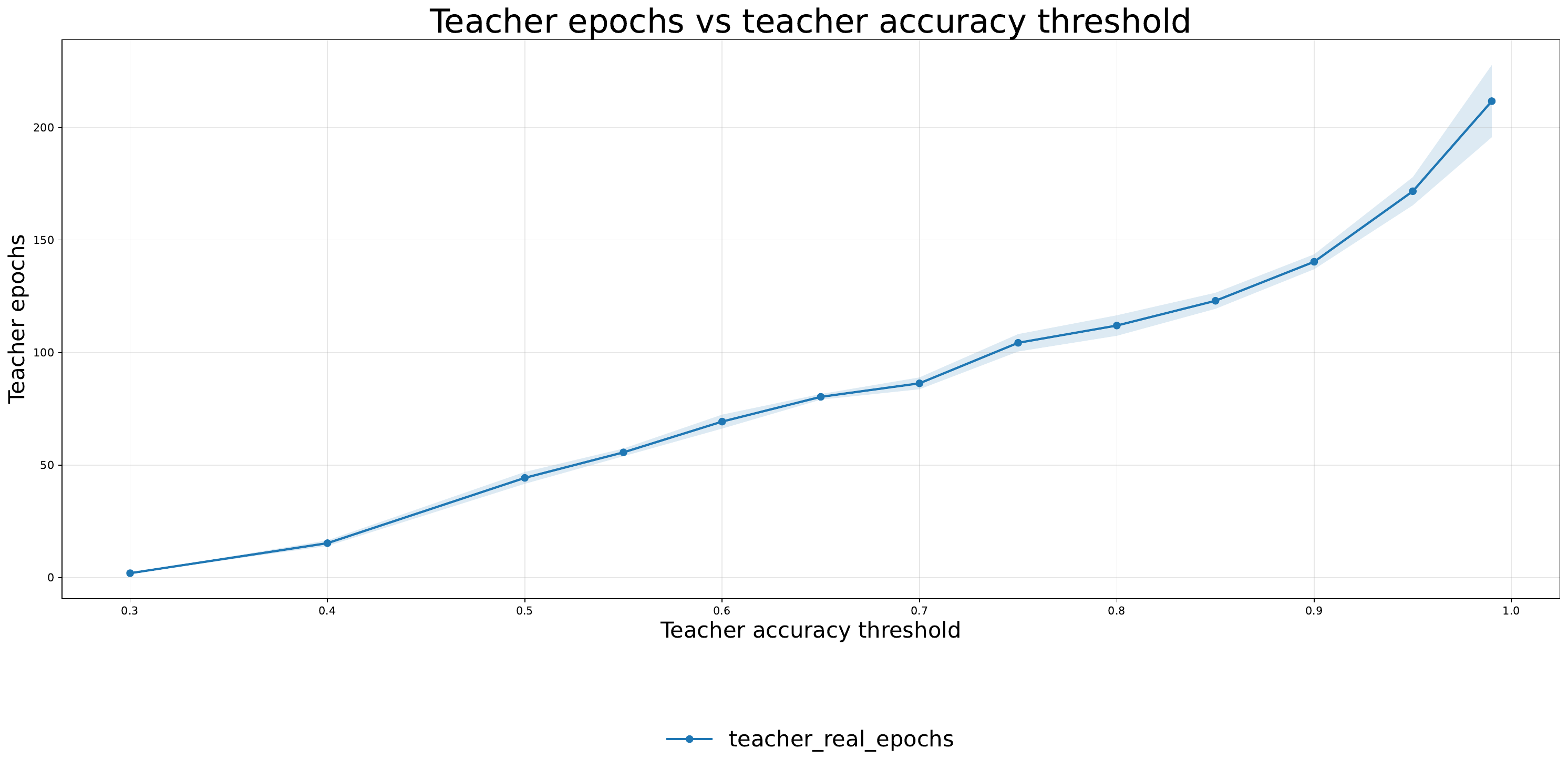}
    \caption{Number of training epochs required to reach the target teacher accuracy threshold on the support set. Higher thresholds lead to substantially increased training cost.}
    \label{fig:threshold-ablations-num-epochs}
\end{figure}

We vary the target teacher accuracy threshold on the support set in the range $[0.3, 0.99]$. Across this range, the final unlearning metrics remain largely stable, suggesting that the method is not highly sensitive to the exact stopping threshold once the teacher reaches moderate accuracy.

Overall, these results suggest that a moderately trained teacher is sufficient: increasing the support-set accuracy threshold within the considered range does not lead to substantial improvements in either aggregate or affected-class metrics.

\subsection{Number of epochs}

We study the effect of the number of unlearning epochs on performance. Figure~\ref{fig:epoch_dynamics} shows the evolution of affected-class metrics over training.

We observe that the model quickly approaches retraining behavior within the first few epochs. In particular, both retain and test accuracy stabilize early, while the forget accuracy rapidly decreases and then remains stable.

Overall, these results indicate that only a small number of unlearning steps is sufficient to capture the relevant local structure, and further training yields diminishing returns.

\begin{figure}[H]
    \centering
    \includegraphics[width=1.0\linewidth]{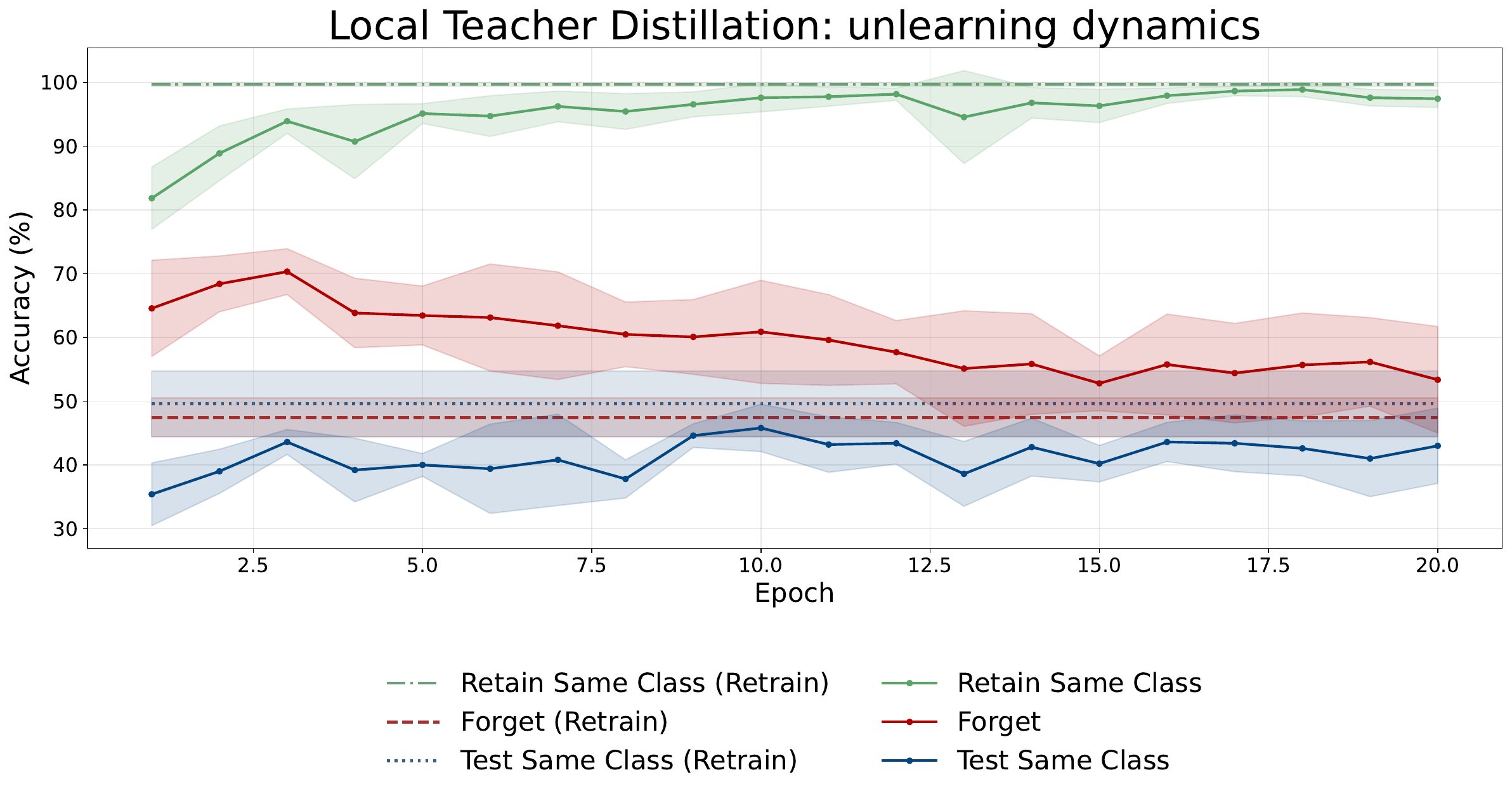}
    \caption{Unlearning dynamics as a function of the number of epochs. Solid lines show the proposed method, dashed lines indicate retraining performance. Metrics quickly stabilize, suggesting fast convergence to retraining behavior.}
    \label{fig:epoch_dynamics}
\end{figure}

\subsection{Support Size}

We study the effect of the support size \(k\) on unlearning quality. Our goal is to identify a locally sufficient support set that approximates retraining without using the retrained model for selection. Because the final unlearning method uses top-3-renormalized teacher targets, all model predictions (teachers' and retrains') in this ablation are processed in the same way.

\subsubsection{Retrain-free diagnostics}
\label{app:support-size-retrain-free}

Using too small support sets leads to degenerate behavior: for instance, if the support set contains only points from the same class, the teacher trivially predicts that class and fails to capture the retrained decision boundary.

To analyze this effect, we vary $k$ from 100 to 2700 (with the same step) and track several signals.

First, we measure the change in the predicted label distribution on $D_f$ as the support size increases. Figure~\ref{fig:kl_previous} shows the KL divergence between consecutive histogram predictions. We observe a transition point: for small $k$, the distribution changes smoothly, while beyond a certain point it starts to shift significantly at each step, indicating that the support set begins to incorporate noisy or less relevant points.

\begin{figure}[H]
    \centering
    \includegraphics[width=1.\linewidth]{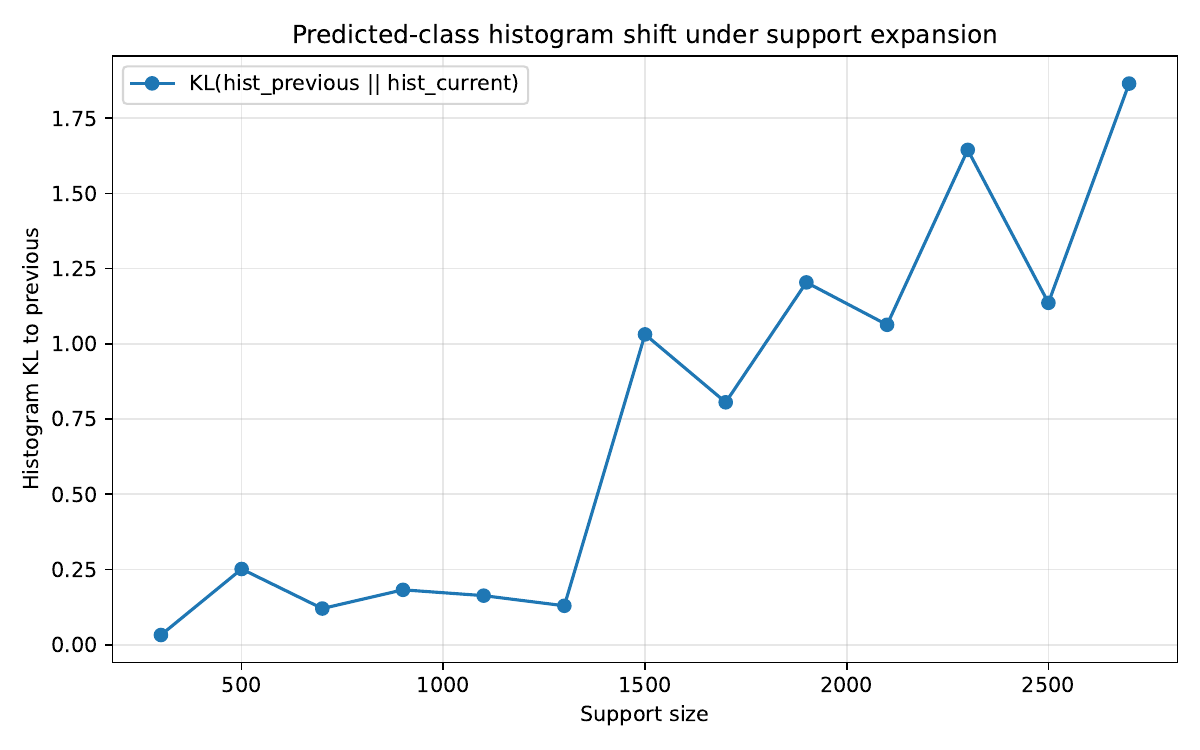}
    \caption{KL divergence between consecutive predicted-class histograms on $D_f$ as the support size increases. A sharp increase indicates a transition from stable local structure to a noisier regime.}
    \label{fig:kl_previous}
\end{figure}

Second, we consider a retrain-free proxy: the accuracy on $D_f$. As shown in Figure~\ref{fig:df_accuracy}, the accuracy initially decreases, then stabilizes, and finally begins to decrease again after the transition point, aligning with the onset of large distributional shifts.

\begin{figure}[H]
    \centering
    \includegraphics[width=1.\linewidth]{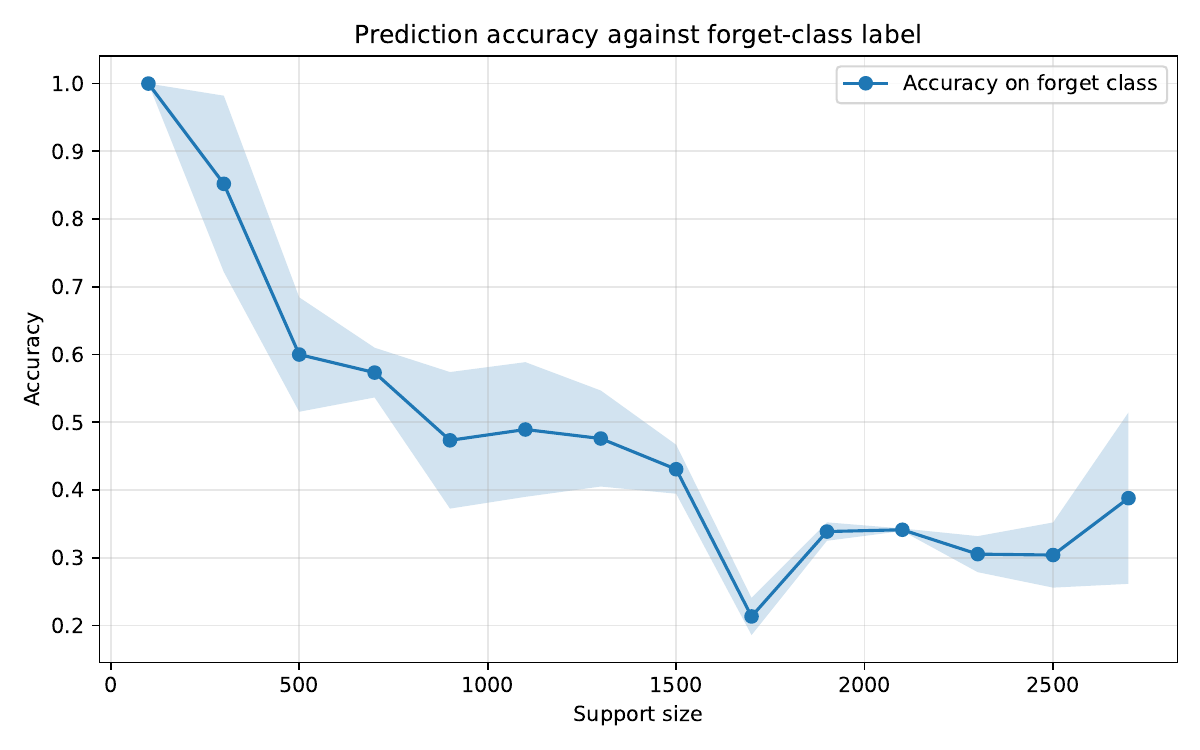}
    \caption{Accuracy on the forget class as a function of support size. The curve exhibits an initial decrease, followed by a plateau and a subsequent degradation, indicating the transition from sufficient to overly large support sets.}
    \label{fig:df_accuracy}
\end{figure}

Finally, we relate this transition to similarity structure. The plateau region corresponds to similarity thresholds around $0.72$--$0.76$, which matches the regime where local neighborhoods remain coherent. As a practical heuristic, we perform a single step of gradient ascent and measure accuracy drop across similarity bins (Figure~\ref{fig:similarity_drop}). Points in high-similarity bins exhibit disproportionate degradation, suggesting that this signal can guide the choice of $k$.

\begin{figure}[H]
    \centering
    \includegraphics[width=1.\linewidth]{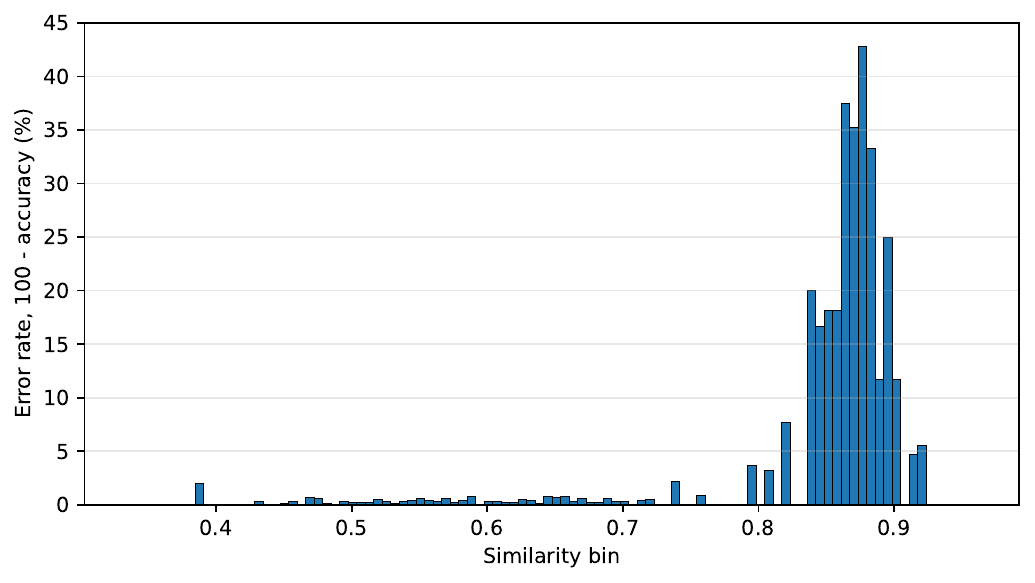}
    \caption{Accuracy drop across similarity bins after a single gradient ascent step. High-similarity points are disproportionately affected, providing a signal for selecting an appropriate support size.}
    \label{fig:similarity_drop}
\end{figure}

Overall, these results indicate that there exists a regime of \emph{locally sufficient} support sizes: beyond this regime, increasing $k$ introduces noise and degrades the ability to capture the local structure relevant for unlearning.

\subsubsection{Post-hoc comparison to retraining}
\label{app:support-size-retrain-comparison}

We additionally compare the behavior of our method to retraining across several signals.

As a post-hoc validation, we compare the predicted-class distribution on \(D_f\)
to that of retraining. Figure~\ref{fig:hist_kl_retrain} shows that the KL
divergence drops sharply once the support set becomes sufficiently large, and
then changes only mildly. This supports the retrain-free diagnostics above:
beyond a moderate local support size, adding more support examples yields limited
additional alignment with retraining.

\begin{figure}[H]
    \centering
    \includegraphics[width=1.\linewidth]{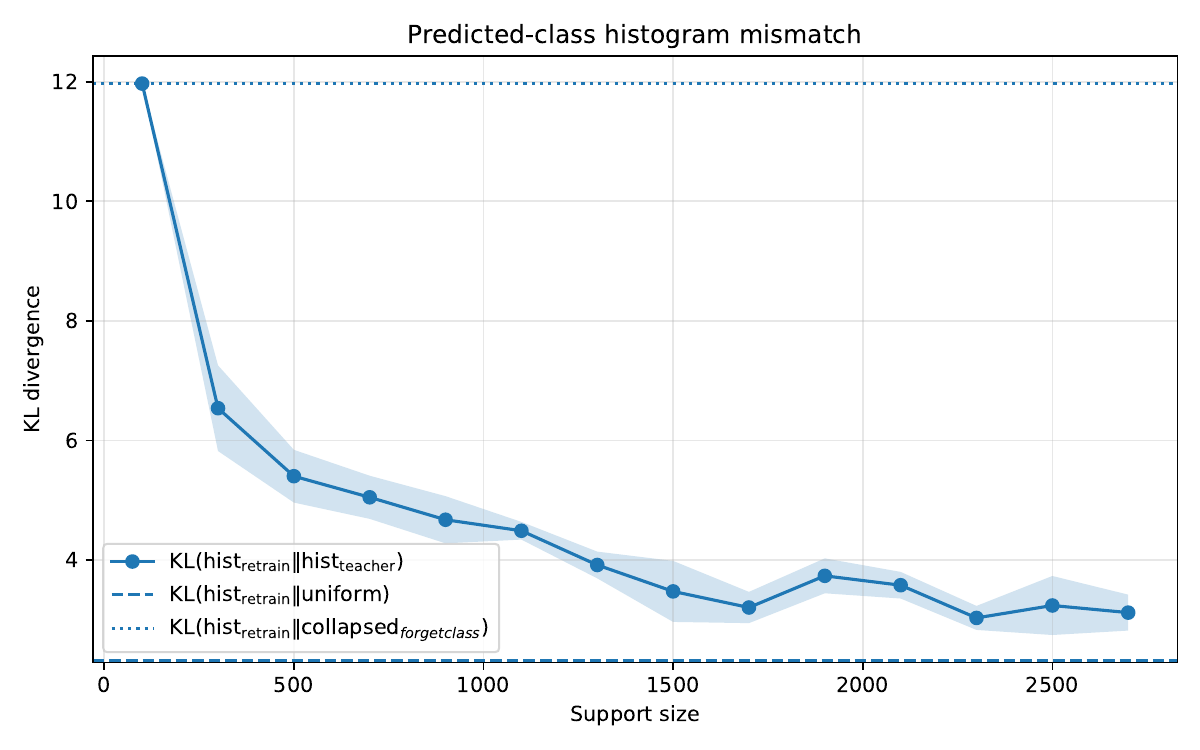}
    \caption{Post-hoc KL divergence between teacher and retrained predicted-class
    distributions on \(D_f\) as a function of support size. The divergence drops
    sharply for small support sizes and then changes only mildly, indicating that a
    moderate local support is sufficient for distribution-level alignment with
    retraining.}
    \label{fig:hist_kl_retrain}
\end{figure}

Together, these results confirm that increasing the support size beyond the locally sufficient regime does not improve approximation to retraining, and may even introduce additional noise.

\section{Additional Experiments}~\label{app:additional-experiments}

\subsection{90\% Class Deletion Experiments}~\label{app:experiments-90pct}

We additionally evaluate our method in a more challenging setting, where \(90\%\)
of a class is removed. This regime leaves only \(10\%\) of the affected-class
training points in \(D_r\), so retraining must preserve high accuracy on these
remaining points while substantially reducing accuracy on \(D_f\). This creates
a large retain--forget prediction gap and makes approximation to retraining more
difficult.

\paragraph{Aggregate and affected-class performance.}
Across aggregate metrics (Table~\ref{tab:90pct_main_results}), our method
achieves the lowest Avg. Gap, indicating the closest overall approximation to
retraining in this more challenging regime. Although LTD does not attain the
smallest UA deviation, it provides the best trade-off across UA, RA, TA, and MIA.

\begin{table}[H]
\caption{Main results for the $90\%$ class deletion setting. Values are mean $\pm$ std over 5 runs; all accuracy metrics are in $\%$. Blue values show deviation from Retrain. Avg. Gap is the mean absolute deviation from Retrain over UA, RA, TA, and MIA.}
\label{tab:90pct_main_results}
\centering
\small
\resizebox{\linewidth}{!}{%
\setlength{\tabcolsep}{4pt}
\begin{tabular}{lcccccc}
\multicolumn{1}{c}{\bf Method} &
\multicolumn{1}{c}{\bf UA} &
\multicolumn{1}{c}{\bf RA} &
\multicolumn{1}{c}{\bf TA} &
\multicolumn{1}{c}{\bf MIA} &
\multicolumn{1}{c}{\bf Avg. Gap} &
\multicolumn{1}{c}{\bf RTE}
\\ \toprule
Retrain & $17.2 \pm 3.0$ & $99.9 \pm 0.1$ & $71.7 \pm 0.3$ & $96.1 \pm 1.4$ & 0 & $1564.7 \pm 10.4$ \\
\midrule
RL      & $46.4 \pm 5.7$ \diff{+29.3} & $94.3 \pm 0.7$ \diff{-5.6} & $67.5 \pm 0.3$ \diff{-4.3} & $100.0 \pm 0.1$ \diff{+3.8} & \valuenotbest{$10.7$} & $162.0 \pm 1.1$ \\
FT      & $60.2 \pm 7.1$ \diff{+43.0} & $94.0 \pm 1.4$ \diff{-5.8} & $67.6 \pm 1.0$ \diff{-4.1} & $50.4 \pm 7.7$ \diff{-45.8} & \valuenotbest{$24.7$} & $155.2 \pm 0.4$ \\
GA      & $24.8 \pm 0.3$ \diff{+7.6}  & $94.8 \pm 0.0$ \diff{-5.1} & $66.2 \pm 0.0$ \diff{-5.6} & $83.9 \pm 0.1$ \diff{-12.3} & \valuenotbest{$7.6$} & $1.9 \pm 0.0$ \\
IU      & $19.2 \pm 11.6$ \diffbest{+2.0} & $97.8 \pm 1.1$ \diff{-2.1} & $69.1 \pm 1.0$ \diff{-2.7} & $91.3 \pm 5.4$ \diff{-4.8} & \valuenotbest{$2.9$} & $27.3 \pm 0.5$ \\
SalUn   & $40.0 \pm 10.5$ \diff{+22.9} & $93.6 \pm 1.0$ \diff{-6.2} & $67.0 \pm 0.9$ \diff{-4.8} & $100.0 \pm 0.0$ \diff{+3.9} & \valuenotbest{$9.4$} & $361.1 \pm 3.3$ \\
AMUN & $32.5 \pm 3.3$ \diff{+15.4} & $94.3 \pm 0.1$ \diff{-5.5} & $72.5 \pm 0.2$ \diffbest{+0.8} & $81.2 \pm 2.4$ \diff{-14.9} & \valuenotbest{$9.2$} & $173.1 \pm 0.8$ \\
\midrule
LTD (Ours) & $11.4 \pm 2.1$ \diff{-5.8} & $99.4 \pm 0.0$ \diffbest{-0.4} & $70.4 \pm 0.1$ \diff{-1.4} & $99.2 \pm 0.3$ \diffbest{+3.0} & \valuebest{$2.7$} & $215.6 \pm 1.5$ \\
\bottomrule
\end{tabular}}
\end{table}

As in the $50\%$ setting, aggregate metrics alone can be misleading. Table~\ref{tab:forget_class_subset_90pct} shows that competing methods exhibit substantial degradation on affected-class points. In particular, methods such as IU and GA significantly deviate from retraining, despite relatively strong performance on some global metrics.

\begin{table}[H]
\caption{Results on the affected-class subset for the $90\%$ deletion setting. Values are mean $\pm$ std over 5 runs, in $\%$. Blue values show deviation from Retrain. Avg. Gap is computed over RA$_{\text{affected-class}}$, UA$_{\text{affected-class}}$, and TA$_{\text{affected-class}}$.}
\label{tab:forget_class_subset_90pct}
\centering
\small
\scalebox{0.9}{%
\setlength{\tabcolsep}{6pt}
\begin{tabular}{lcccc}
\multicolumn{1}{c}{\bf Method} &
\multicolumn{1}{c}{\bf RA$_{\text{affected-class}}$} &
\multicolumn{1}{c}{\bf UA$_{\text{affected-class}}$} &
\multicolumn{1}{c}{\bf TA$_{\text{affected-class}}$} &
\multicolumn{1}{c}{\bf Avg. Gap}
\\ \toprule
Retrain & $98.8 \pm 1.6$ & $17.2 \pm 3.0$ & $17.8 \pm 2.0$ & 0 \\
\midrule
RL      & $68.0 \pm 8.5$ \diff{-30.8} & $46.4 \pm 5.7$ \diff{+29.3} & $26.0 \pm 4.5$ \diff{+8.2} & \valuenotbest{$22.8$} \\
FT      & $92.0 \pm 2.8$ \diffbest{-6.8} & $60.2 \pm 7.1$ \diff{+43.0} & $39.6 \pm 4.2$ \diff{+21.8} & \valuenotbest{$23.9$} \\
GA      & $32.4 \pm 1.5$ \diff{-66.4} & $24.8 \pm 0.3$ \diff{+7.6} & $16.0 \pm 0.0$ \diffbest{-1.8} & \valuenotbest{$25.3$} \\
IU      & $17.6 \pm 14.0$ \diff{-81.2} & $19.2 \pm 11.6$ \diffbest{+2.0} & $10.8 \pm 5.3$ \diff{-7.0} & \valuenotbest{$30.1$} \\
SalUn   & $58.8 \pm 14.0$ \diff{-40.0} & $40.0 \pm 10.5$ \diff{+22.9} & $24.8 \pm 6.7$ \diff{+7.0} & \valuenotbest{$23.3$} \\
AMUN & $60.0 \pm 7.0$ \diff{-38.8} & $32.5 \pm 3.3$ \diff{+15.4} & $25.0 \pm 1.1$ \diff{+7.2} & \valuenotbest{$20.5$} \\
\midrule
LTD (Ours) & $75.6 \pm 10.1$ \diff{-23.2} & $11.4 \pm 2.1$ \diff{-5.8} & $9.4 \pm 3.9$ \diff{-8.4} & \valuebest{$12.5$} \\
\bottomrule
\end{tabular}}
\end{table}

\begin{figure}[H]
    \centering
    \includegraphics[width=1.\linewidth]{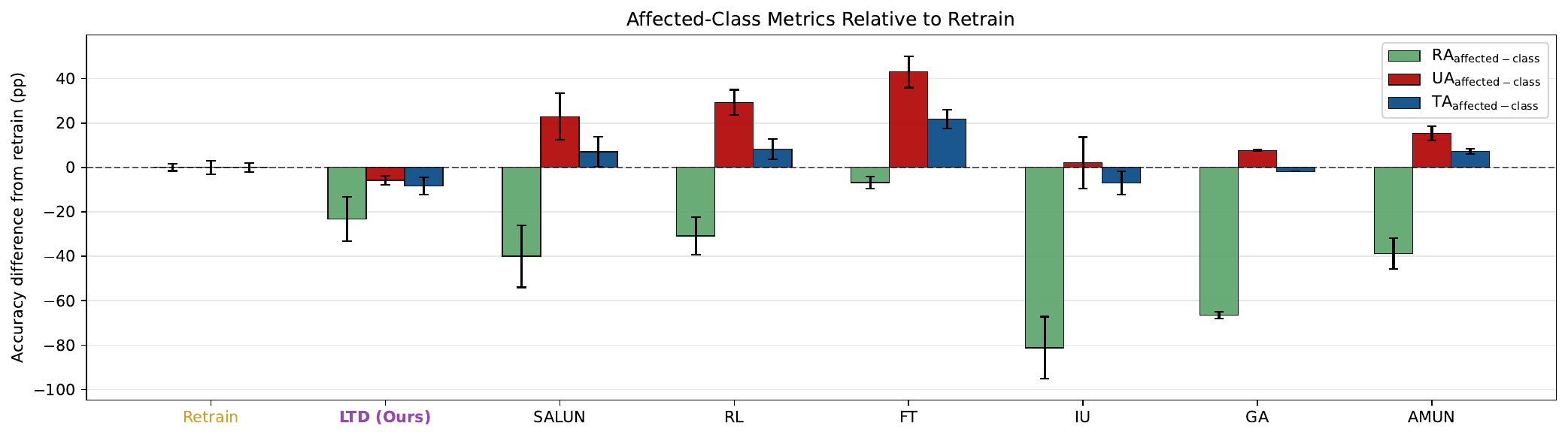}
    \caption{\textbf{Affected-class performance difference from Retrain, 90\% of class deletion.}
    Bars show differences from Retrain for RA$_{\text{affected-class}}$,
    UA$_{\text{affected-class}}$, and TA$_{\text{affected-class}}$; values closer to
    zero indicate better agreement. The figure reveals localized collateral
    forgetting not captured by aggregate metrics.}
    \label{fig:placeholder}
\end{figure}

In contrast, our method remains consistently closer to retraining across both aggregate and affected-class metrics. Notably, it achieves the lowest UA on the affected class, indicating more effective removal of the target data, while maintaining reasonable performance on the remaining points.

Overall, these results confirm that the advantages of our approach persist in more extreme deletion regimes, where accurate control of local behavior becomes even more critical.

\begin{figure}[H]
    \centering
    \includegraphics[width=1.\linewidth]{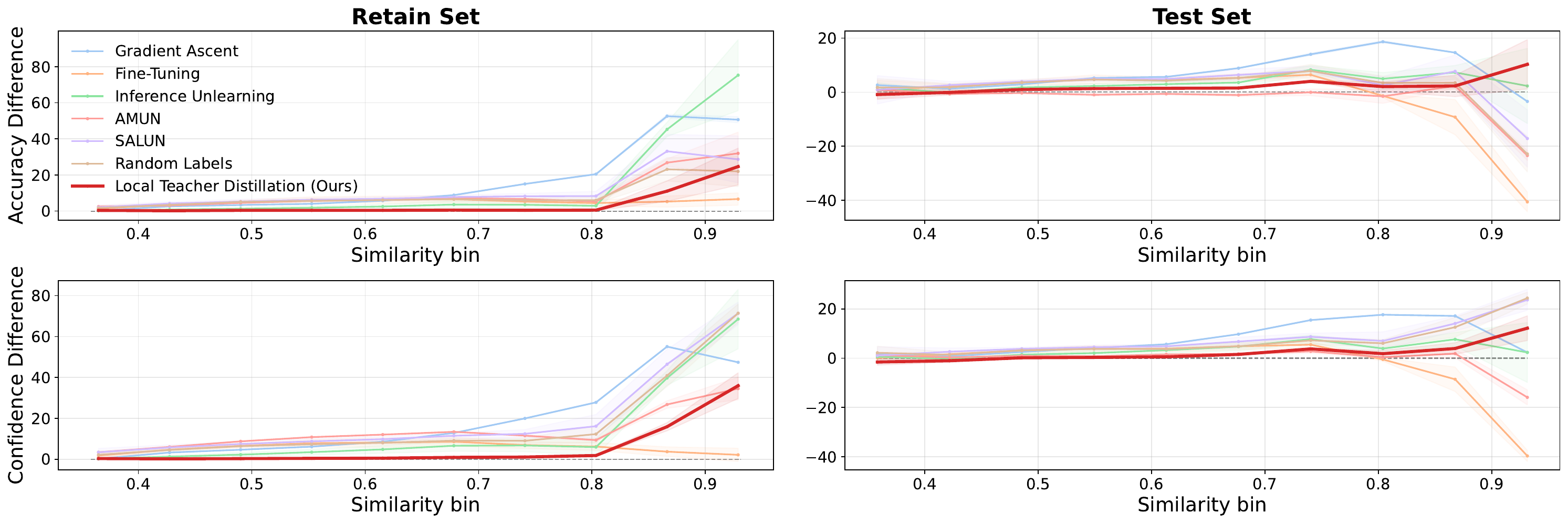}
    \caption{\textbf{Localized prediction deviations under $90\%$ class deletion.}
    We plot bin-level differences from retraining:
    $\mathrm{Acc}_{\mathrm{retrain}}-\mathrm{Acc}_{\mathrm{unlearn}}$ for accuracy
    (top row) and
    $f_{\theta_{\mathrm{retrain}}}(x)_y-f_{\theta_{\mathrm{unlearn}}}(x)_y$ for
    correct-class confidence (bottom row), as functions of similarity to $D_f$ on
    both retain and test sets. Larger deviations at high similarity indicate that
    unlearning effects concentrate near the removed data.}
    \label{fig:acc-drop-90}
\end{figure}

\paragraph{Similarity-based analysis.} We further analyze locality effects in the $90\%$ deletion setting by measuring accuracy and confidence differences with respect to retraining as a function of similarity to the forget set (Figure~\ref{fig:acc-drop-90}). We observe that deviations from retraining become significantly more pronounced for high-similarity points compared to the $50\%$ setting. In particular, competing methods exhibit sharp degradation in both accuracy and confidence on points close to the forget set, indicating amplified localized collateral forgetting. In contrast, our method remains comparatively stable across similarity bins, maintaining low deviation from retraining even in this more challenging regime.

\subsection{Additional affected classes under 50\% deletion}
\label{app:additional-classes-50pct}

\paragraph{Setup.}
We report additional results for the $50\%$ partial-deletion setting on five randomly selected affected classes:
\texttt{oak\_tree} (52), \texttt{apple} (0), \texttt{hamster} (36), \texttt{road} (68), and \texttt{shark} (73).
These experiments complement the main \texttt{couch} results and test whether localized collateral forgetting
and the retraining alignment of LTD persist across different semantic classes. For each class, we report the
same two views as in the main experiments: aggregate metrics and affected-class metrics. The aggregate table
measures global agreement with retraining across UA, RA, TA, and MIA, while the affected-class table isolates
the local neighborhood where collateral forgetting is most likely to appear.

\subsubsection{Class 52: \texttt{oak\_tree}}

For \texttt{oak\_tree}, LTD achieves the lowest aggregate Avg. Gap and the lowest affected-class Avg. Gap. This indicates that the method improves alignment with retraining both globally and on the localized affected-class subset. The improvement is not due to matching a single metric: LTD stays close to retraining on RA and TA, while also keeping the affected-class discrepancy moderate. Several baselines match retraining well on individual metrics, but their deviations are less balanced across the aggregate and affected-class views.

\begin{figure}[H]
    \centering
    \includegraphics[width=1.\linewidth]{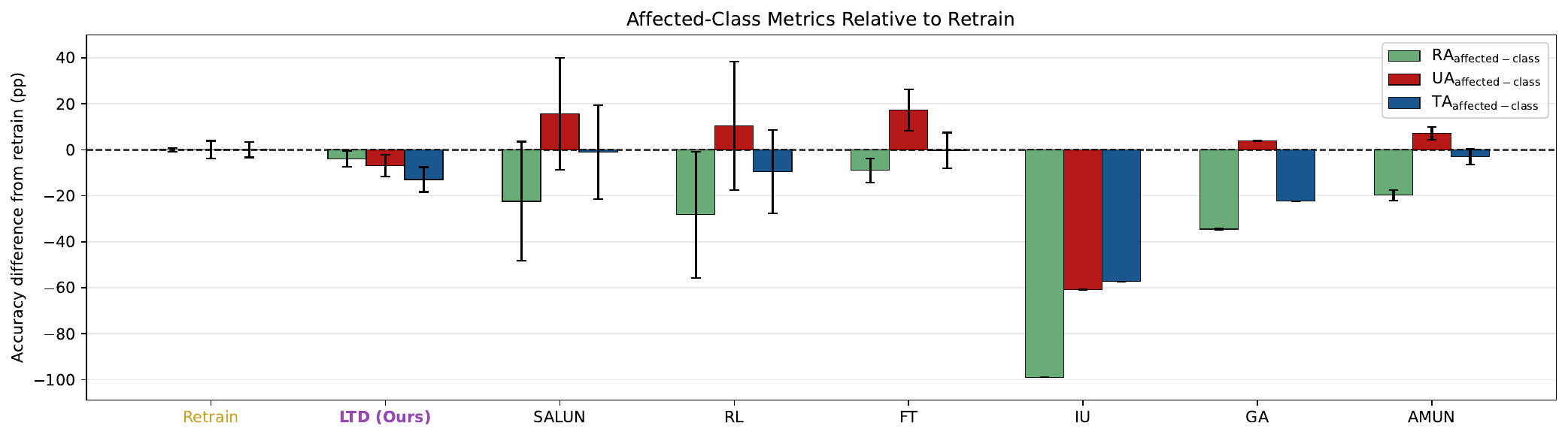}
    \caption{\textbf{Affected-class performance difference from Retrain, \texttt{oak\_tree}.}
    Bars show differences from Retrain for RA$_{\text{affected-class}}$,
    UA$_{\text{affected-class}}$, and TA$_{\text{affected-class}}$; values closer to
    zero indicate better agreement. The figure reveals localized collateral
    forgetting not captured by aggregate metrics.}
    \label{fig:forget-class-52}
\end{figure}

\begin{figure}[H]
    \centering
    \includegraphics[width=1.\linewidth]{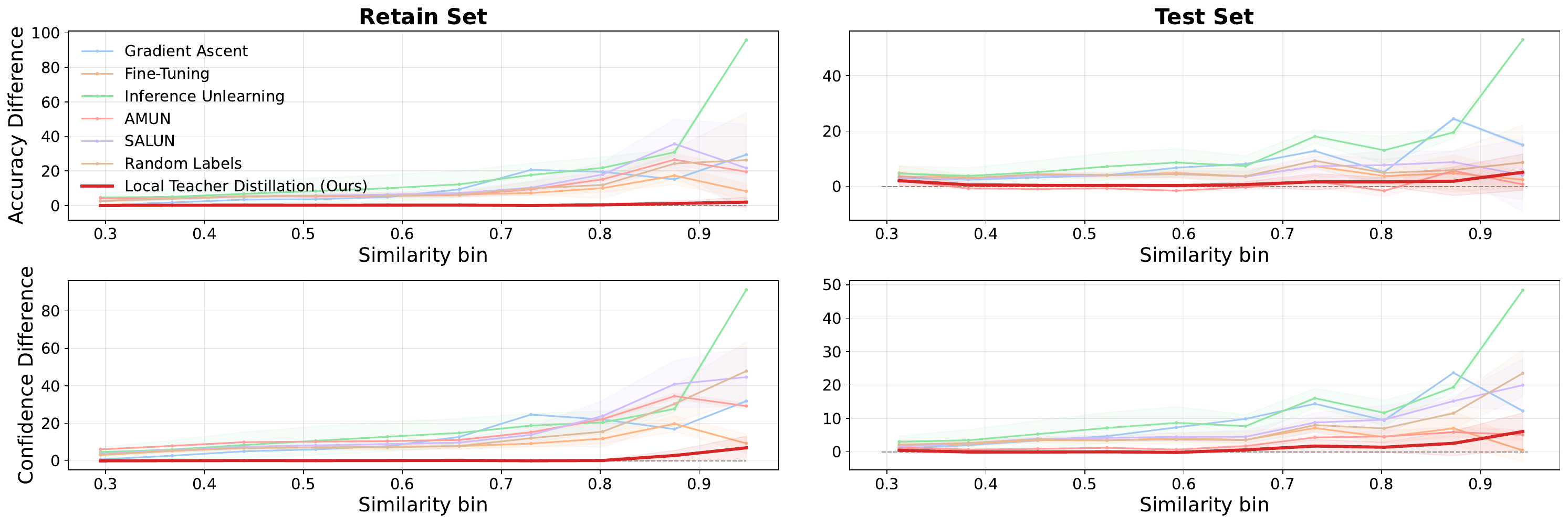}
    \caption{\textbf{Localized prediction deviations.}
    We plot bin-level differences from retraining:
    $\mathrm{Acc}_{\mathrm{retrain}}-\mathrm{Acc}_{\mathrm{unlearn}}$ for accuracy
    (top row) and
    $f_{\theta_{\mathrm{retrain}}}(x)_y-f_{\theta_{\mathrm{unlearn}}}(x)_y$ for
    correct-class confidence (bottom row), as functions of similarity to $D_f$ on
    both retain and test sets. Larger deviations at high similarity indicate that
    unlearning effects concentrate near the removed data.}
    \label{fig:acc-drop-52}
\end{figure}

\begin{table}[H]
\caption{Main results for the $50\%$ deletion setting for class 52 (\texttt{oak\_tree}). Values are mean $\pm$ std over 5 runs; all accuracy metrics are in $\%$. Blue values show deviation from Retrain. Avg. Gap is the mean absolute deviation from Retrain over UA, RA, TA, and MIA.}
\label{tab:apple_50pct_main_results}
\centering
\small
\resizebox{\linewidth}{!}{%
\setlength{\tabcolsep}{4pt}
\begin{tabular}{lcccccc}
\multicolumn{1}{c}{\bf Method} &
\multicolumn{1}{c}{\bf UA} &
\multicolumn{1}{c}{\bf RA} &
\multicolumn{1}{c}{\bf TA} &
\multicolumn{1}{c}{\bf MIA} &
\multicolumn{1}{c}{\bf Avg. Gap} &
\multicolumn{1}{c}{\bf RTE}
\\ \toprule
Retrain & $61.1 \pm 3.9$ & $99.9 \pm 0.0$ & $71.9 \pm 0.3$ & $64.1 \pm 1.5$ & 0 & $1587.9 \pm 12.5$ \\
\midrule
RL & $71.5 \pm 27.9$ \diff{+10.5} & $94.2 \pm 0.8$ \diff{-5.7} & $67.7 \pm 0.4$ \diff{-4.2} & $89.4 \pm 10.8$ \diff{+25.2} & \valuenotbest{$11.4$} & $157.3 \pm 0.7$ \\
FT & $78.5 \pm 9.0$ \diff{+17.4} & $94.4 \pm 1.2$ \diff{-5.5} & $67.8 \pm 1.0$ \diff{-4.1} & $32.6 \pm 11.5$ \diff{-31.6} & \valuenotbest{$14.6$} & $156.5 \pm 1.8$ \\
GA & $65.0 \pm 0.2$ \diffbest{+4.0} & $95.5 \pm 0.0$ \diff{-4.4} & $67.3 \pm 0.0$ \diff{-4.6} & $49.4 \pm 0.2$ \diffbest{-14.8} & \valuenotbest{$6.9$} & $1.7 \pm 0.0$ \\
IU & $0.2 \pm 0.2$ \diff{-60.9} & $91.2 \pm 7.0$ \diff{-8.7} & $64.9 \pm 4.3$ \diff{-7.0} & $100.0 \pm 0.0$ \diff{+35.9} & \valuenotbest{$28.1$} & $27.7 \pm 0.4$ \\
SalUn & $76.8 \pm 24.3$ \diff{+15.7} & $94.0 \pm 1.1$ \diff{-5.8} & $67.9 \pm 0.6$ \diff{-4.0} & $82.1 \pm 16.1$ \diff{+17.9} & \valuenotbest{$10.9$} & $350.2 \pm 5.4$ \\
AMUN & $68.2 \pm 2.7$ \diff{+7.2} & $94.4 \pm 0.1$ \diff{-5.5} & $72.7 \pm 0.3$ \diffbest{+0.8} & $49.3 \pm 2.9$ \diff{-14.9} & \valuenotbest{$7.1$} & $151.2 \pm 0.7$ \\
\midrule
LTD (Ours) & $54.2 \pm 4.7$ \diff{-6.9} & $99.6 \pm 0.1$ \diffbest{-0.3} & $71.1 \pm 0.3$ \diffbest{-0.8} & $80.9 \pm 3.1$ \diff{+16.7} & \valuebest{$6.2$} & $215.4 \pm 2.9$ \\
\bottomrule
\end{tabular}}
\end{table}

\begin{table}[H]
\caption{Results on the affected-class subset for the $50\%$ deletion setting for class 52 (\texttt{oak\_tree}). Values are mean $\pm$ std over 5 runs, in $\%$. Blue values show deviation from Retrain. Avg. Gap is computed over RA$_{\text{affected-class}}$, UA$_{\text{affected-class}}$, and TA$_{\text{affected-class}}$.}
\label{tab:apple_forget_class_subset_50pct}
\centering
\small
\scalebox{0.9}{%
\setlength{\tabcolsep}{6pt}
\begin{tabular}{lcccc}
\multicolumn{1}{c}{\bf Method} &
\multicolumn{1}{c}{\bf RA$_{\text{affected-class}}$} &
\multicolumn{1}{c}{\bf UA$_{\text{affected-class}}$} &
\multicolumn{1}{c}{\bf TA$_{\text{affected-class}}$} &
\multicolumn{1}{c}{\bf Avg. Gap}
\\ \toprule
Retrain & $98.9 \pm 0.8$ & $61.1 \pm 3.9$ & $57.3 \pm 3.3$ & 0 \\
\midrule
RL & $70.7 \pm 27.4$ \diff{-28.2} & $71.5 \pm 27.9$ \diff{+10.5} & $47.8 \pm 18.2$ \diff{-9.5} & \valuenotbest{$16.1$} \\
FT & $90.0 \pm 5.3$ \diff{-8.9} & $78.5 \pm 9.0$ \diff{+17.4} & $57.0 \pm 7.8$ \diffbest{-0.3} & \valuenotbest{$8.9$} \\
GA & $64.4 \pm 0.4$ \diff{-34.5} & $65.0 \pm 0.2$ \diffbest{+4.0} & $35.0 \pm 0.0$ \diff{-22.3} & \valuenotbest{$20.3$} \\
IU & $0.0 \pm 0.0$ \diff{-98.9} & $0.2 \pm 0.2$ \diff{-60.9} & $0.0 \pm 0.0$ \diff{-57.3} & \valuenotbest{$72.4$} \\
SalUn & $76.6 \pm 25.9$ \diff{-22.4} & $76.8 \pm 24.3$ \diff{+15.7} & $56.4 \pm 20.4$ \diff{-0.9} & \valuenotbest{$13.0$} \\
AMUN & $79.2 \pm 2.2$ \diff{-19.7} & $68.2 \pm 2.7$ \diff{+7.2} & $54.4 \pm 3.4$ \diff{-2.9} & \valuenotbest{$9.9$} \\
\midrule
LTD (Ours) & $95.0 \pm 3.4$ \diffbest{-3.9} & $54.2 \pm 4.7$ \diff{-6.9} & $44.4 \pm 5.4$ \diff{-12.9} & \valuebest{$7.9$} \\
\bottomrule
\end{tabular}}
\end{table}

\subsubsection{Class 0: \texttt{apple}}

For \texttt{apple}, LTD obtains the best aggregate Avg. Gap, showing the closest overall agreement with
retraining across UA, RA, TA, and MIA. On the affected-class subset, however, AMUN achieves the best
Avg. Gap, with LTD ranking second. This is a useful case where another method matches the local
affected-class behavior slightly better, while LTD remains more balanced globally. The comparison highlights
that the aggregate and affected-class views are complementary: a method can be strong locally but less aligned
overall, and vice versa.

\begin{figure}[H]
    \centering
    \includegraphics[width=1.\linewidth]{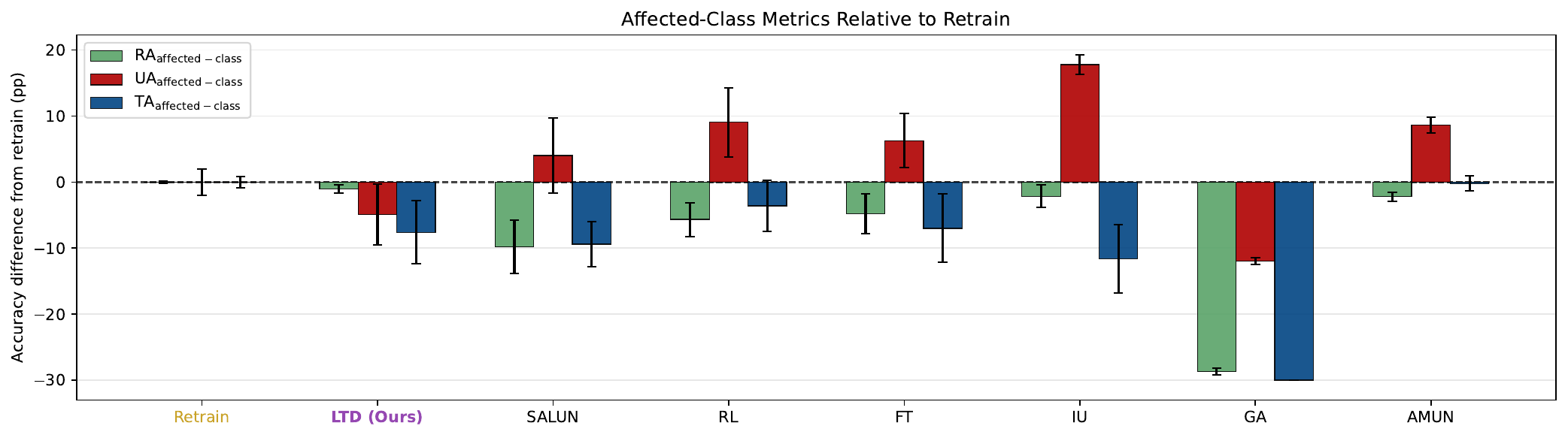}
    \caption{\textbf{Affected-class performance difference from Retrain.}
    Bars show differences from Retrain for RA$_{\text{affected-class}}$,
    UA$_{\text{affected-class}}$, and TA$_{\text{affected-class}}$; values closer to
    zero indicate better agreement. The figure reveals localized collateral
    forgetting not captured by aggregate metrics.}
    \label{fig:forget-class-0}
\end{figure}

\begin{figure}[H]
    \centering
    \includegraphics[width=1.\linewidth]{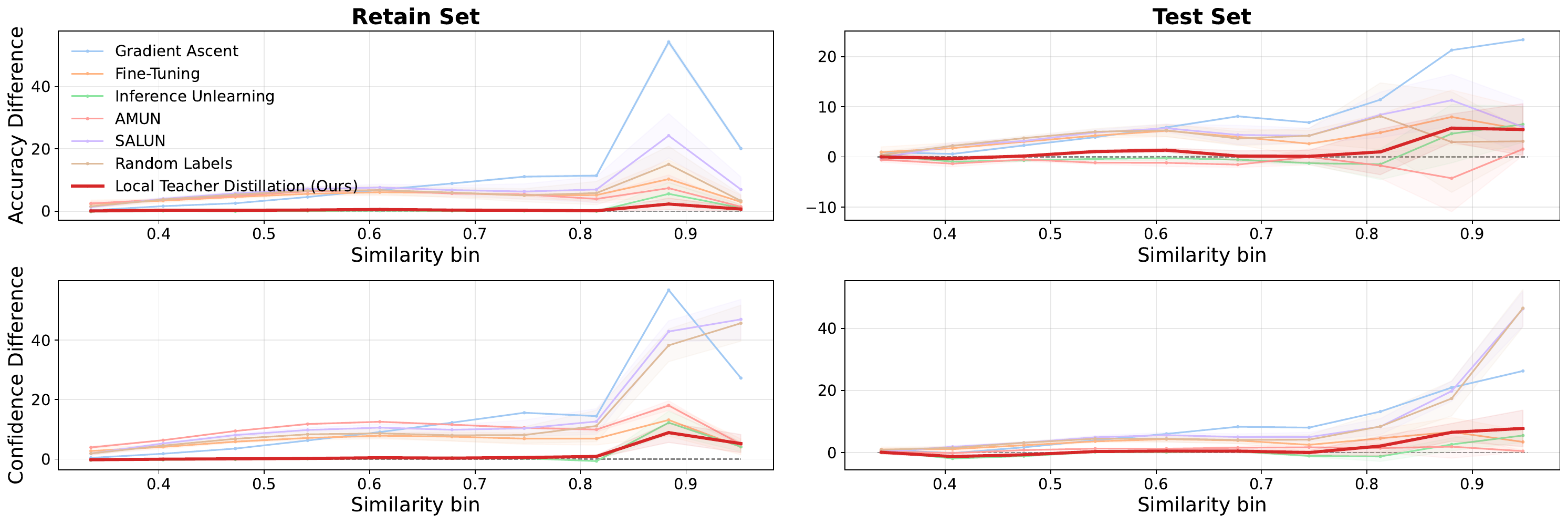}
    \caption{\textbf{Localized prediction deviations.}
    We plot bin-level differences from retraining:
    $\mathrm{Acc}_{\mathrm{retrain}}-\mathrm{Acc}_{\mathrm{unlearn}}$ for accuracy
    (top row) and
    $f_{\theta_{\mathrm{retrain}}}(x)_y-f_{\theta_{\mathrm{unlearn}}}(x)_y$ for
    correct-class confidence (bottom row), as functions of similarity to $D_f$ on
    both retain and test sets. Larger deviations at high similarity indicate that
    unlearning effects concentrate near the removed data.}
    \label{fig:acc-drop-0}
\end{figure}

\begin{table}[H]
\caption{Main results for the $50\%$ deletion setting for class 0 (\texttt{apple}). Values are mean $\pm$ std over 5 runs; all accuracy metrics are in $\%$. Blue values show deviation from Retrain. Avg. Gap is the mean absolute deviation from Retrain over UA, RA, TA, and MIA.}
\label{tab:apple_50pct_main_results}
\centering
\small
\resizebox{\linewidth}{!}{%
\setlength{\tabcolsep}{4pt}
\begin{tabular}{lcccccc}
\multicolumn{1}{c}{\bf Method} &
\multicolumn{1}{c}{\bf UA} &
\multicolumn{1}{c}{\bf RA} &
\multicolumn{1}{c}{\bf TA} &
\multicolumn{1}{c}{\bf MIA} &
\multicolumn{1}{c}{\bf Avg. Gap} &
\multicolumn{1}{c}{\bf RTE}
\\ \toprule
Retrain & $79.3 \pm 2.0$ & $99.8 \pm 0.0$ & $71.7 \pm 0.1$ & $35.7 \pm 1.8$ & 0 & $1540.6 \pm 10.1$ \\
\midrule
RL & $88.4 \pm 5.2$ \diff{+9.1} & $94.1 \pm 0.9$ \diff{-5.7} & $67.5 \pm 0.6$ \diff{-4.3} & $94.4 \pm 3.2$ \diff{+58.7} & \valuenotbest{$19.4$} & $158.0 \pm 0.9$ \\
FT & $85.6 \pm 4.1$ \diff{+6.3} & $94.8 \pm 0.7$ \diff{-5.0} & $68.0 \pm 0.6$ \diff{-3.7} & $21.2 \pm 5.3$ \diff{-14.5} & \valuenotbest{$7.4$} & $155.7 \pm 1.0$ \\
GA & $67.4 \pm 0.5$ \diff{-12.0} & $95.5 \pm 0.0$ \diff{-4.3} & $67.8 \pm 0.0$ \diff{-3.9} & $52.2 \pm 0.2$ \diff{+16.5} & \valuenotbest{$9.2$} & $1.7 \pm 0.0$ \\
IU & $97.1 \pm 1.5$ \diff{+17.8} & $99.8 \pm 0.1$ \diffbest{-0.1} & $72.2 \pm 0.3$ \diffbest{+0.5} & $25.2 \pm 5.5$ \diffbest{-10.5} & \valuenotbest{$7.2$} & $27.5 \pm 0.3$ \\
SalUn & $83.4 \pm 5.7$ \diffbest{+4.0} & $93.6 \pm 0.6$ \diff{-6.2} & $67.5 \pm 0.3$ \diff{-4.2} & $90.5 \pm 5.0$ \diff{+54.7} & \valuenotbest{$17.3$} & $351.1 \pm 2.4$ \\
AMUN & $88.0 \pm 1.2$ \diff{+8.7} & $94.3 \pm 0.1$ \diff{-5.5} & $72.7 \pm 0.3$ \diff{+1.0} & $17.8 \pm 1.6$ \diff{-18.0} & \valuenotbest{$8.3$} & $171.7 \pm 1.0$ \\
\midrule
LTD (Ours) & $74.4 \pm 4.6$ \diff{-4.9} & $99.5 \pm 0.0$ \diff{-0.4} & $71.0 \pm 0.1$ \diff{-0.7} & $49.8 \pm 10.6$ \diff{+14.1} & \valuebest{$5.0$} & $213.7 \pm 1.5$ \\
\bottomrule
\end{tabular}}
\end{table}

\begin{table}[H]
\caption{Results on the affected-class subset for the $50\%$ deletion setting for class 0 (\texttt{apple}). Values are mean $\pm$ std over 5 runs, in $\%$. Blue values show deviation from Retrain. Avg. Gap is computed over RA$_{\text{affected-class}}$, UA$_{\text{affected-class}}$, and TA$_{\text{affected-class}}$.}
\label{tab:apple_forget_class_subset_50pct}
\centering
\small
\scalebox{0.9}{%
\setlength{\tabcolsep}{6pt}
\begin{tabular}{lcccc}
\multicolumn{1}{c}{\bf Method} &
\multicolumn{1}{c}{\bf RA$_{\text{affected-class}}$} &
\multicolumn{1}{c}{\bf UA$_{\text{affected-class}}$} &
\multicolumn{1}{c}{\bf TA$_{\text{affected-class}}$} &
\multicolumn{1}{c}{\bf Avg. Gap}
\\ \toprule
Retrain & $99.7 \pm 0.2$ & $79.3 \pm 2.0$ & $82.0 \pm 0.8$ & 0 \\
\midrule
RL & $94.1 \pm 2.6$ \diff{-5.7} & $88.4 \pm 5.2$ \diff{+9.1} & $78.4 \pm 3.8$ \diff{-3.6} & \valuenotbest{$6.1$} \\
FT & $95.0 \pm 3.0$ \diff{-4.8} & $85.6 \pm 4.1$ \diff{+6.3} & $75.0 \pm 5.2$ \diff{-7.0} & \valuenotbest{$6.0$} \\
GA & $71.0 \pm 0.5$ \diff{-28.7} & $67.4 \pm 0.5$ \diff{-12.0} & $52.0 \pm 0.0$ \diff{-30.0} & \valuenotbest{$23.6$} \\
IU & $97.6 \pm 1.7$ \diff{-2.1} & $97.1 \pm 1.5$ \diff{+17.8} & $70.4 \pm 5.2$ \diff{-11.6} & \valuenotbest{$10.5$} \\
SalUn & $89.9 \pm 4.0$ \diff{-9.8} & $83.4 \pm 5.7$ \diffbest{+4.0} & $72.6 \pm 3.4$ \diff{-9.4} & \valuenotbest{$7.7$} \\
AMUN & $97.5 \pm 0.7$ \diff{-2.2} & $88.0 \pm 1.2$ \diff{+8.7} & $81.8 \pm 1.2$ \diffbest{-0.2} & \valuebest{$3.7$} \\
\midrule
LTD (Ours) & $98.7 \pm 0.6$ \diffbest{-1.0} & $74.4 \pm 4.6$ \diff{-4.9} & $74.4 \pm 4.8$ \diff{-7.6} & \valuenotbest{$4.5$} \\
\bottomrule
\end{tabular}}
\end{table}

\subsubsection{Class 36: \texttt{hamster}}

For \texttt{hamster}, LTD achieves the lowest aggregate Avg. Gap and the lowest affected-class Avg. Gap. The method stays close to retraining on RA, TA, and MIA, while also reducing the local affected-class discrepancy relative to the baselines. This class is therefore consistent with the broader trend across additional affected classes: LTD provides balanced agreement with retraining both globally and in the neighborhood most directly coupled to the forget set.

\begin{figure}[H]
    \centering
    \includegraphics[width=1.\linewidth]{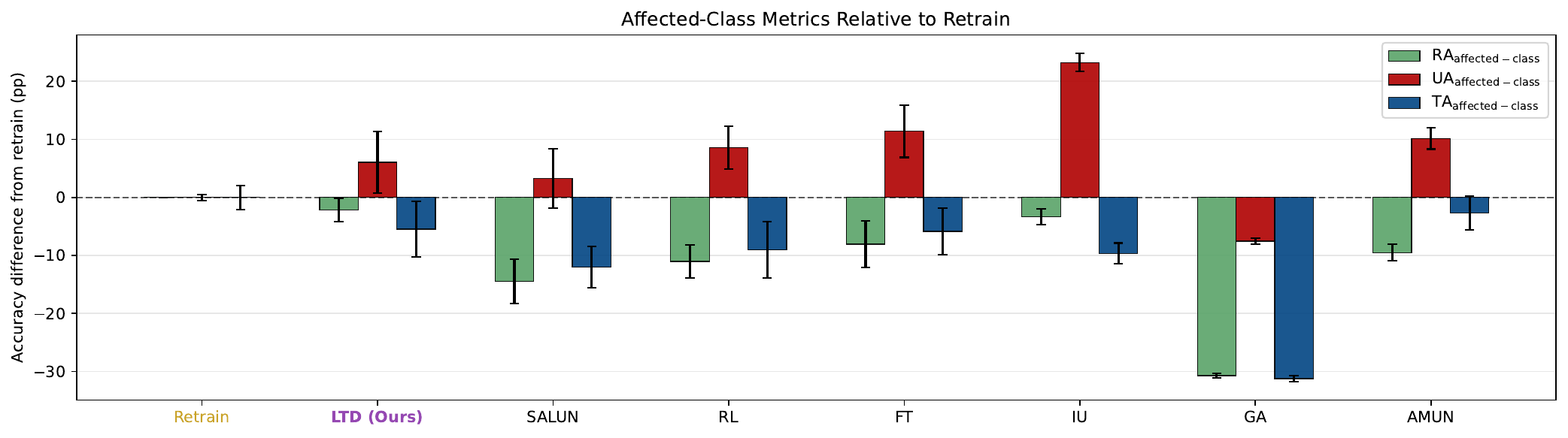}
    \caption{\textbf{Affected-class performance difference from Retrain, \texttt{hamster} (36).}
    Bars show differences from Retrain for RA$_{\text{affected-class}}$,
    UA$_{\text{affected-class}}$, and TA$_{\text{affected-class}}$; values closer to
    zero indicate better agreement. The figure reveals localized collateral
    forgetting not captured by aggregate metrics.}
    \label{fig:forget-class-36}
\end{figure}

\begin{figure}[H]
    \centering
    \includegraphics[width=1.\linewidth]{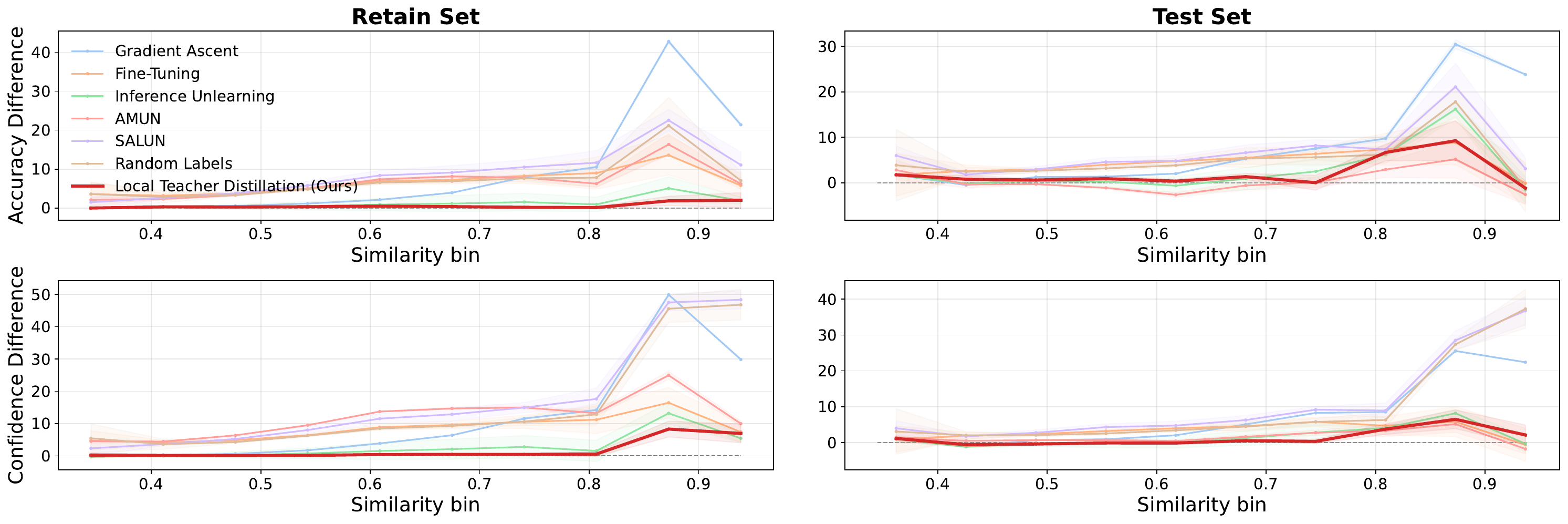}
    \caption{\textbf{Localized prediction deviations, \texttt{hamster} (36).}
    We plot bin-level differences from retraining:
    $\mathrm{Acc}_{\mathrm{retrain}}-\mathrm{Acc}_{\mathrm{unlearn}}$ for accuracy
    (top row) and
    $f_{\theta_{\mathrm{retrain}}}(x)_y-f_{\theta_{\mathrm{unlearn}}}(x)_y$ for
    correct-class confidence (bottom row), as functions of similarity to $D_f$ on
    both retain and test sets. Larger deviations at high similarity indicate that
    unlearning effects concentrate near the removed data.}
    \label{fig:acc-drop-36}
\end{figure}

\begin{table}[H]
\caption{Main results for the $50\%$ deletion setting for class 36 (\texttt{hamster}). Values are mean $\pm$ std over 5 runs; all accuracy metrics are in $\%$. Blue values show deviation from Retrain. Avg. Gap is the mean absolute deviation from Retrain over UA, RA, TA, and MIA.}
\label{tab:hamster_50pct_main_results}
\centering
\small
\resizebox{\linewidth}{!}{%
\setlength{\tabcolsep}{4pt}
\begin{tabular}{lcccccc}
\multicolumn{1}{c}{\bf Method} &
\multicolumn{1}{c}{\bf UA} &
\multicolumn{1}{c}{\bf RA} &
\multicolumn{1}{c}{\bf TA} &
\multicolumn{1}{c}{\bf MIA} &
\multicolumn{1}{c}{\bf Avg. Gap} &
\multicolumn{1}{c}{\bf RTE}
\\ \toprule
Retrain & $72.7 \pm 0.5$ & $99.9 \pm 0.0$ & $71.8 \pm 0.4$ & $47.3 \pm 0.4$ & 0 & $1580.3 \pm 3.9$ \\
\midrule
RL & $81.3 \pm 3.7$ \diff{+8.6} & $94.6 \pm 0.3$ \diff{-5.2} & $68.2 \pm 0.2$ \diff{-3.6} & $96.4 \pm 2.4$ \diff{+49.1} & \valuenotbest{$16.6$} & $157.4 \pm 1.0$ \\
FT & $84.1 \pm 4.5$ \diff{+11.4} & $94.4 \pm 0.5$ \diff{-5.5} & $67.7 \pm 0.5$ \diff{-4.1} & $23.4 \pm 5.3$ \diff{-24.0} & \valuenotbest{$11.3$} & $155.9 \pm 1.3$ \\
GA & $65.1 \pm 0.5$ \diff{-7.5} & $97.9 \pm 0.0$ \diff{-1.9} & $69.4 \pm 0.0$ \diff{-2.4} & $64.4 \pm 0.0$ \diff{+17.1} & \valuenotbest{$7.2$} & $1.7 \pm 0.0$ \\
IU & $95.9 \pm 1.6$ \diff{+23.3} & $99.2 \pm 1.2$ \diff{-0.7} & $71.4 \pm 1.3$ \diffbest{-0.3} & $28.1 \pm 5.6$ \diff{-19.3} & \valuenotbest{$10.9$} & $27.3 \pm 0.4$ \\
SalUn & $75.9 \pm 5.1$ \diffbest{+3.3} & $93.4 \pm 1.0$ \diff{-6.4} & $67.2 \pm 0.9$ \diff{-4.6} & $92.0 \pm 5.7$ \diff{+44.7} & \valuenotbest{$14.8$} & $352.5 \pm 2.8$ \\
AMUN & $82.8 \pm 1.8$ \diff{+10.1} & $94.3 \pm 0.2$ \diff{-5.5} & $72.9 \pm 0.1$ \diff{+1.1} & $26.5 \pm 3.6$ \diff{-20.9} & \valuenotbest{$9.4$} & $169.9 \pm 1.3$ \\
\midrule
LTD (Ours) & $78.7 \pm 5.3$ \diff{+6.1} & $99.5 \pm 0.1$ \diffbest{-0.4} & $71.0 \pm 0.3$ \diff{-0.8} & $48.1 \pm 9.4$ \diffbest{+0.7} & \valuebest{$2.0$} & $214.8 \pm 3.8$ \\
\bottomrule
\end{tabular}}
\end{table}

\begin{table}[H]
\caption{Results on the affected-class subset for the $50\%$ deletion setting for class 36 (\texttt{hamster}). Values are mean $\pm$ std over 5 runs, in $\%$. Blue values show deviation from Retrain. Avg. Gap is computed over RA$_{\text{affected-class}}$, UA$_{\text{affected-class}}$, and TA$_{\text{affected-class}}$.}
\label{tab:apple_forget_class_subset_50pct}
\centering
\small
\scalebox{0.9}{%
\setlength{\tabcolsep}{6pt}
\begin{tabular}{lcccc}
\multicolumn{1}{c}{\bf Method} &
\multicolumn{1}{c}{\bf RA$_{\text{affected-class}}$} &
\multicolumn{1}{c}{\bf UA$_{\text{affected-class}}$} &
\multicolumn{1}{c}{\bf TA$_{\text{affected-class}}$} &
\multicolumn{1}{c}{\bf Avg. Gap}
\\ \toprule
Retrain & $100.0 \pm 0.0$ & $72.7 \pm 0.5$ & $72.7 \pm 2.1$ & 0 \\
\midrule
RL & $89.0 \pm 2.9$ \diff{-11.0} & $81.3 \pm 3.7$ \diff{+8.6} & $63.6 \pm 4.8$ \diff{-9.1} & \valuenotbest{$9.6$} \\
FT & $91.9 \pm 4.0$ \diff{-8.1} & $84.1 \pm 4.5$ \diff{+11.4} & $66.8 \pm 4.1$ \diff{-5.9} & \valuenotbest{$8.5$} \\
GA & $69.3 \pm 0.4$ \diff{-30.7} & $65.1 \pm 0.5$ \diff{-7.5} & $41.4 \pm 0.5$ \diff{-31.3} & \valuenotbest{$23.2$} \\
IU & $96.6 \pm 1.4$ \diff{-3.4} & $95.9 \pm 1.6$ \diff{+23.3} & $63.0 \pm 1.8$ \diff{-9.7} & \valuenotbest{$12.1$} \\
SalUn & $85.5 \pm 3.9$ \diff{-14.5} & $75.9 \pm 5.1$ \diffbest{+3.3} & $60.6 \pm 3.6$ \diff{-12.1} & \valuenotbest{$9.9$} \\
AMUN & $90.5 \pm 1.4$ \diff{-9.5} & $82.8 \pm 1.8$ \diff{+10.1} & $70.0 \pm 2.9$ \diffbest{-2.7} & \valuenotbest{$7.4$} \\
\midrule
LTD (Ours) & $97.8 \pm 2.0$ \diffbest{-2.2} & $78.7 \pm 5.3$ \diff{+6.1} & $67.2 \pm 4.8$ \diff{-5.5} & \valuebest{$4.6$} \\
\bottomrule
\end{tabular}}
\end{table}

\subsubsection{Class 68: \texttt{road}}

For \texttt{road}, LTD achieves the lowest aggregate Avg. Gap. On the affected-class subset, FT, LTD, and AMUN
are very close, with LTD ranking second and nearly matching the best affected-class Avg. Gap. This class shows
a regime where several methods can approximate the local affected-class behavior reasonably well, but LTD gives
the best overall trade-off across the aggregate metrics. In particular, it preserves RA and TA close to retraining
while keeping the MIA deviation smaller than most baselines.

\begin{figure}[H]
    \centering
    \includegraphics[width=1.\linewidth]{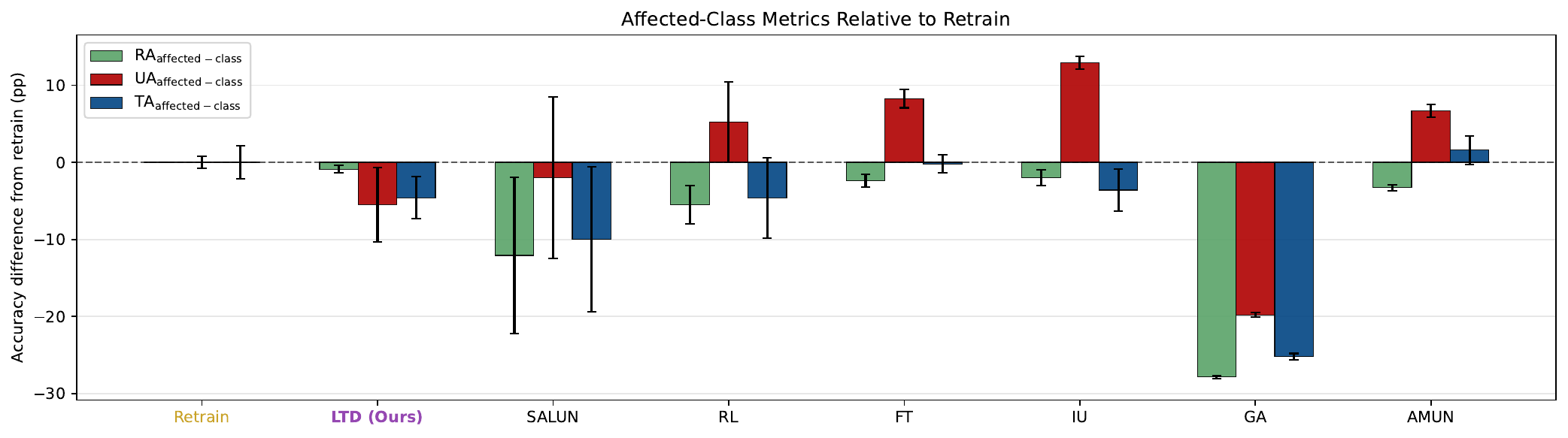}
    \caption{\textbf{Affected-class performance difference from Retrain, \texttt{road} (68).}
    Bars show differences from Retrain for RA$_{\text{affected-class}}$,
    UA$_{\text{affected-class}}$, and TA$_{\text{affected-class}}$; values closer to
    zero indicate better agreement. The figure reveals localized collateral
    forgetting not captured by aggregate metrics.}
    \label{fig:forget-class-68}
\end{figure}

\begin{figure}[H]
    \centering
    \includegraphics[width=1.\linewidth]{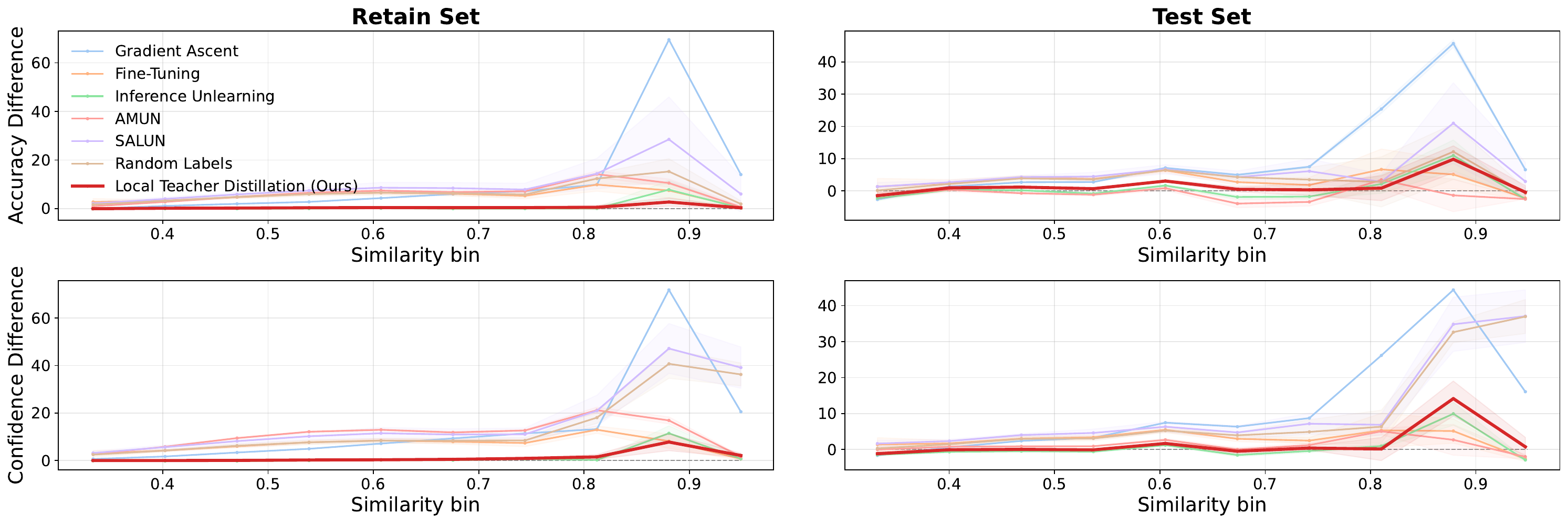}
    \caption{\textbf{Localized prediction deviations, \texttt{road} (68).}
    We plot bin-level differences from retraining:
    $\mathrm{Acc}_{\mathrm{retrain}}-\mathrm{Acc}_{\mathrm{unlearn}}$ for accuracy
    (top row) and
    $f_{\theta_{\mathrm{retrain}}}(x)_y-f_{\theta_{\mathrm{unlearn}}}(x)_y$ for
    correct-class confidence (bottom row), as functions of similarity to $D_f$ on
    both retain and test sets. Larger deviations at high similarity indicate that
    unlearning effects concentrate near the removed data.}
    \label{fig:acc-drop-68}
\end{figure}

\begin{table}[H]
\caption{Main results for the $50\%$ deletion setting for class 68 (\texttt{road}). Values are mean $\pm$ std over 5 runs; all accuracy metrics are in $\%$. Blue values show deviation from Retrain. Avg. Gap is the mean absolute deviation from Retrain over UA, RA, TA, and MIA.}
\label{tab:road_50pct_main_results}
\centering
\small
\resizebox{\linewidth}{!}{%
\setlength{\tabcolsep}{4pt}
\begin{tabular}{lcccccc}
\multicolumn{1}{c}{\bf Method} &
\multicolumn{1}{c}{\bf UA} &
\multicolumn{1}{c}{\bf RA} &
\multicolumn{1}{c}{\bf TA} &
\multicolumn{1}{c}{\bf MIA} &
\multicolumn{1}{c}{\bf Avg. Gap} &
\multicolumn{1}{c}{\bf RTE}
\\ \toprule
Retrain & $86.1 \pm 0.8$ & $99.8 \pm 0.1$ & $72.0 \pm 0.3$ & $26.3 \pm 2.5$ & 0 & $1557.8 \pm 12.1$ \\
\midrule
RL & $91.4 \pm 5.2$ \diff{+5.2} & $94.6 \pm 0.7$ \diff{-5.3} & $68.0 \pm 0.4$ \diff{-4.0} & $85.7 \pm 6.4$ \diff{+59.4} & \valuenotbest{$18.5$} & $157.8 \pm 0.6$ \\
FT & $94.4 \pm 1.2$ \diff{+8.3} & $94.6 \pm 0.5$ \diff{-5.2} & $68.0 \pm 0.4$ \diff{-4.0} & $8.6 \pm 2.6$ \diff{-17.7} & \valuenotbest{$8.8$} & $156.0 \pm 0.6$ \\
GA & $66.3 \pm 0.3$ \diff{-19.8} & $97.0 \pm 0.0$ \diff{-2.8} & $68.4 \pm 0.0$ \diff{-3.7} & $50.2 \pm 0.3$ \diff{+24.0} & \valuenotbest{$12.6$} & $1.7 \pm 0.0$ \\
IU & $99.0 \pm 0.8$ \diff{+12.9} & $99.8 \pm 0.1$ \diffbest{-0.0} & $72.1 \pm 0.3$ \diffbest{+0.0} & $12.5 \pm 3.0$ \diff{-13.8} & \valuenotbest{$6.7$} & $27.4 \pm 0.3$ \\
SalUn & $84.2 \pm 10.5$ \diffbest{-2.0} & $93.2 \pm 1.2$ \diff{-6.6} & $67.4 \pm 0.9$ \diff{-4.6} & $77.8 \pm 14.8$ \diff{+51.5} & \valuenotbest{$16.2$} & $347.4 \pm 3.9$ \\
AMUN & $92.8 \pm 0.8$ \diff{+6.7} & $94.3 \pm 0.1$ \diff{-5.5} & $72.8 \pm 0.3$ \diff{+0.7} & $11.1 \pm 1.2$ \diff{-15.1} & \valuenotbest{$7.0$} & $171.1 \pm 0.5$ \\
\midrule
LTD (Ours) & $80.6 \pm 4.8$ \diff{-5.5} & $99.5 \pm 0.0$ \diff{-0.3} & $70.8 \pm 0.2$ \diff{-1.2} & $33.1 \pm 8.1$ \diffbest{+6.9} & \valuebest{$3.5$} & $202.6 \pm 3.4$ \\
\bottomrule
\end{tabular}}
\end{table}

\begin{table}[H]
\caption{Results on the affected-class subset for the $50\%$ deletion setting for class 68 (\texttt{road}). Values are mean $\pm$ std over 5 runs, in $\%$. Blue values show deviation from Retrain. Avg. Gap is computed over RA$_{\text{affected-class}}$, UA$_{\text{affected-class}}$, and TA$_{\text{affected-class}}$.}
\label{tab:road_forget_class_subset_50pct}
\centering
\small
\scalebox{0.9}{%
\setlength{\tabcolsep}{6pt}
\begin{tabular}{lcccc}
\multicolumn{1}{c}{\bf Method} &
\multicolumn{1}{c}{\bf RA$_{\text{affected-class}}$} &
\multicolumn{1}{c}{\bf UA$_{\text{affected-class}}$} &
\multicolumn{1}{c}{\bf TA$_{\text{affected-class}}$} &
\multicolumn{1}{c}{\bf Avg. Gap}
\\ \toprule
Retrain & $100.0 \pm 0.0$ & $86.1 \pm 0.8$ & $88.0 \pm 2.2$ & 0 \\
\midrule
RL & $94.5 \pm 2.5$ \diff{-5.5} & $91.4 \pm 5.2$ \diff{+5.2} & $83.4 \pm 5.2$ \diff{-4.6} & \valuenotbest{$5.1$} \\
FT & $97.6 \pm 0.8$ \diff{-2.4} & $94.4 \pm 1.2$ \diff{+8.3} & $87.8 \pm 1.2$ \diffbest{-0.2} & \valuebest{$3.6$} \\
GA & $72.2 \pm 0.2$ \diff{-27.8} & $66.3 \pm 0.3$ \diff{-19.8} & $62.8 \pm 0.4$ \diff{-25.2} & \valuenotbest{$24.3$} \\
IU & $98.0 \pm 1.0$ \diff{-2.0} & $99.0 \pm 0.8$ \diff{+12.9} & $84.4 \pm 2.7$ \diff{-3.6} & \valuenotbest{$6.2$} \\
SalUn & $87.9 \pm 10.1$ \diff{-12.1} & $84.2 \pm 10.5$ \diffbest{-2.0} & $78.0 \pm 9.4$ \diff{-10.0} & \valuenotbest{$8.0$} \\
AMUN & $96.7 \pm 0.4$ \diff{-3.3} & $92.8 \pm 0.8$ \diff{+6.7} & $89.6 \pm 1.9$ \diff{+1.6} & \valuenotbest{$3.8$} \\
\midrule
LTD (Ours) & $99.1 \pm 0.5$ \diffbest{-0.9} & $80.6 \pm 4.8$ \diff{-5.5} & $83.4 \pm 2.7$ \diff{-4.6} & \valuenotbest{$3.7$} \\
\bottomrule
\end{tabular}}
\end{table}

\subsubsection{Class 73: \texttt{shark}}

For \texttt{shark}, LTD is the clearest positive case among the additional classes: it achieves the lowest aggregate
Avg. Gap and the lowest affected-class Avg. Gap. The method stays close to retraining on RA, TA, and MIA, while
also reducing the mismatch on the affected-class subset. Competing methods often improve one component at the
cost of another, for example by preserving global utility but substantially shifting UA or degrading affected-class
accuracy. This supports the interpretation that local teacher targets can better match the retraining-induced local
prediction structure.

\begin{figure}[H]
    \centering
    \includegraphics[width=1.\linewidth]{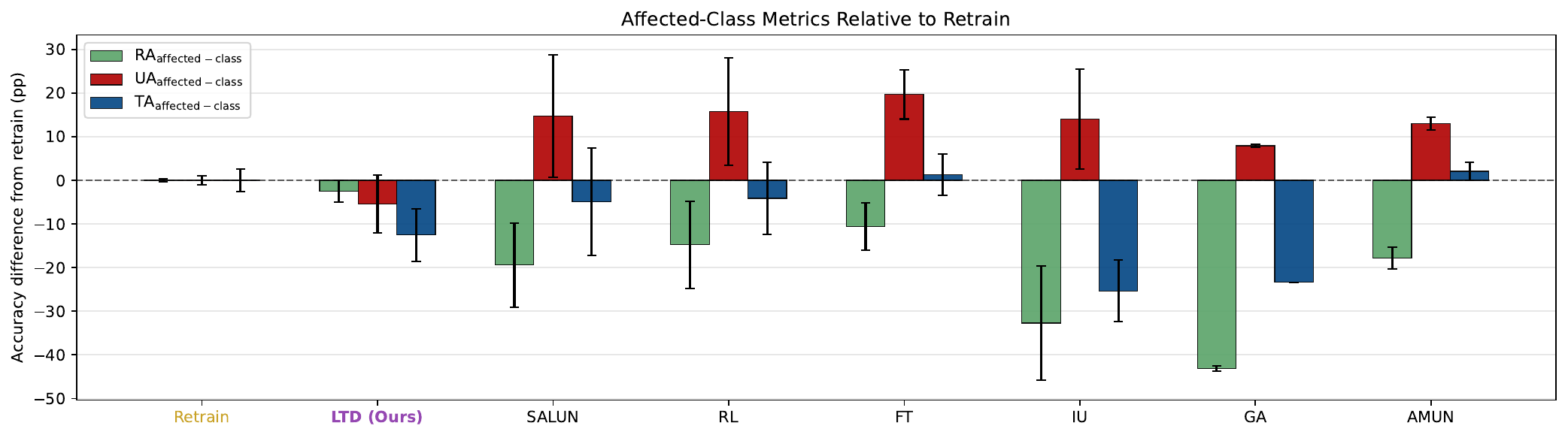}
    \caption{\textbf{Affected-class performance difference from Retrain.}
    Bars show differences from Retrain for RA$_{\text{affected-class}}$,
    UA$_{\text{affected-class}}$, and TA$_{\text{affected-class}}$; values closer to
    zero indicate better agreement. The figure reveals localized collateral
    forgetting not captured by aggregate metrics.}
    \label{fig:forget-class-73}
\end{figure}

\begin{figure}[H]
    \centering
    \includegraphics[width=1.\linewidth]{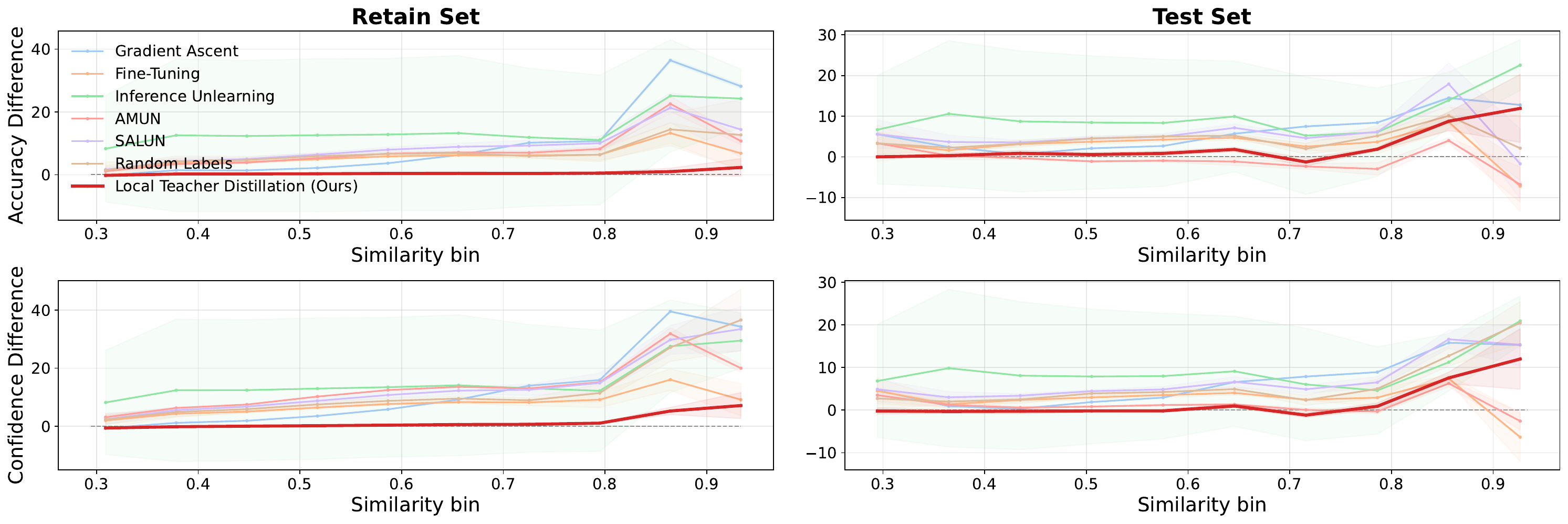}
    \caption{\textbf{Localized prediction deviations.}
    We plot bin-level differences from retraining:
    $\mathrm{Acc}_{\mathrm{retrain}}-\mathrm{Acc}_{\mathrm{unlearn}}$ for accuracy
    (top row) and
    $f_{\theta_{\mathrm{retrain}}}(x)_y-f_{\theta_{\mathrm{unlearn}}}(x)_y$ for
    correct-class confidence (bottom row), as functions of similarity to $D_f$ on
    both retain and test sets. Larger deviations at high similarity indicate that
    unlearning effects concentrate near the removed data.}
    \label{fig:acc-drop-73}
\end{figure}

\begin{table}[H]
\caption{Main results for the $50\%$ deletion setting for class 73 (\texttt{shark}). Values are mean $\pm$ std over 5 runs; all accuracy metrics are in $\%$. Blue values show deviation from Retrain. Avg. Gap is the mean absolute deviation from Retrain over UA, RA, TA, and MIA.}
\label{tab:shark_50pct_main_results}
\centering
\small
\resizebox{\linewidth}{!}{%
\setlength{\tabcolsep}{4pt}
\begin{tabular}{lcccccc}
\multicolumn{1}{c}{\bf Method} &
\multicolumn{1}{c}{\bf UA} &
\multicolumn{1}{c}{\bf RA} &
\multicolumn{1}{c}{\bf TA} &
\multicolumn{1}{c}{\bf MIA} &
\multicolumn{1}{c}{\bf Avg. Gap} &
\multicolumn{1}{c}{\bf RTE}
\\ \toprule
Retrain & $54.9 \pm 1.0$ & $99.8 \pm 0.0$ & $71.8 \pm 0.3$ & $70.4 \pm 0.9$ & 0 & $1581.3 \pm 4.6$ \\
\midrule
RL & $70.7 \pm 12.3$ \diff{+15.8} & $94.0 \pm 0.9$ \diff{-5.8} & $67.5 \pm 0.5$ \diff{-4.3} & $90.7 \pm 5.5$ \diff{+20.3} & \valuenotbest{$11.5$} & $157.9 \pm 0.8$ \\
FT & $74.6 \pm 5.7$ \diff{+19.7} & $94.8 \pm 0.3$ \diff{-5.1} & $68.1 \pm 0.3$ \diff{-3.7} & $37.0 \pm 8.4$ \diff{-33.4} & \valuenotbest{$15.5$} & $155.8 \pm 0.8$ \\
GA & $62.9 \pm 0.4$ \diff{+7.9} & $96.6 \pm 0.0$ \diff{-3.3} & $68.8 \pm 0.0$ \diff{-3.0} & $58.2 \pm 0.6$ \diff{-12.2} & \valuenotbest{$6.6$} & $1.7 \pm 0.0$ \\
IU & $69.0 \pm 11.5$ \diff{+14.0} & $87.2 \pm 24.0$ \diff{-12.7} & $63.2 \pm 15.7$ \diff{-8.6} & $62.0 \pm 19.2$ \diff{-8.4} & \valuenotbest{$10.9$} & $27.2 \pm 0.3$ \\
SalUn & $69.7 \pm 14.0$ \diff{+14.7} & $93.1 \pm 0.7$ \diff{-6.7} & $66.9 \pm 0.4$ \diff{-4.9} & $79.5 \pm 6.4$ \diff{+9.1} & \valuenotbest{$8.8$} & $351.5 \pm 1.7$ \\
AMUN & $67.9 \pm 1.5$ \diff{+13.0} & $94.3 \pm 0.0$ \diff{-5.5} & $72.8 \pm 0.2$ \diff{+1.0} & $47.1 \pm 1.3$ \diff{-23.3} & \valuenotbest{$10.7$} & $173.5 \pm 0.9$ \\
\midrule
LTD (Ours) & $49.5 \pm 6.6$ \diffbest{-5.4} & $99.5 \pm 0.1$ \diffbest{-0.3} & $70.9 \pm 0.4$ \diffbest{-0.9} & $76.3 \pm 5.9$ \diffbest{+5.9} & \valuebest{$3.1$} & $210.3 \pm 0.7$ \\
\bottomrule
\end{tabular}}
\end{table}

\begin{table}[H]
\caption{Results on the affected-class subset for the $50\%$ deletion setting for class 73 (\texttt{shark}). Values are mean $\pm$ std over 5 runs, in $\%$. Blue values show deviation from Retrain. Avg. Gap is computed over RA$_{\text{affected-class}}$, UA$_{\text{affected-class}}$, and TA$_{\text{affected-class}}$.}
\label{tab:shark_forget_class_subset_50pct}
\centering
\small
\scalebox{0.9}{%
\setlength{\tabcolsep}{6pt}
\begin{tabular}{lcccc}
\multicolumn{1}{c}{\bf Method} &
\multicolumn{1}{c}{\bf RA$_{\text{affected-class}}$} &
\multicolumn{1}{c}{\bf UA$_{\text{affected-class}}$} &
\multicolumn{1}{c}{\bf TA$_{\text{affected-class}}$} &
\multicolumn{1}{c}{\bf Avg. Gap}
\\ \toprule
Retrain & $99.7 \pm 0.4$ & $54.9 \pm 1.0$ & $43.3 \pm 2.6$ & 0 \\
\midrule
RL & $85.0 \pm 10.0$ \diff{-14.8} & $70.7 \pm 12.3$ \diff{+15.8} & $39.2 \pm 8.3$ \diff{-4.1} & \valuenotbest{$11.6$} \\
FT & $89.1 \pm 5.4$ \diff{-10.6} & $74.6 \pm 5.7$ \diff{+19.7} & $44.6 \pm 4.8$ \diffbest{+1.3} & \valuenotbest{$10.5$} \\
GA & $56.6 \pm 0.6$ \diff{-43.2} & $62.9 \pm 0.4$ \diff{+7.9} & $20.0 \pm 0.0$ \diff{-23.3} & \valuenotbest{$24.8$} \\
IU & $67.0 \pm 13.1$ \diff{-32.7} & $69.0 \pm 11.5$ \diff{+14.0} & $18.0 \pm 7.0$ \diff{-25.3} & \valuenotbest{$24.0$} \\
SalUn & $80.3 \pm 9.6$ \diff{-19.4} & $69.7 \pm 14.0$ \diff{+14.7} & $38.4 \pm 12.3$ \diff{-4.9} & \valuenotbest{$13.0$} \\
AMUN & $81.9 \pm 2.4$ \diff{-17.8} & $67.9 \pm 1.5$ \diff{+13.0} & $45.4 \pm 2.1$ \diff{+2.1} & \valuenotbest{$11.0$} \\
\midrule
LTD (Ours) & $97.2 \pm 2.5$ \diffbest{-2.5} & $49.5 \pm 6.6$ \diffbest{-5.4} & $30.8 \pm 6.0$ \diff{-12.5} & \valuebest{$6.8$} \\
\bottomrule
\end{tabular}}
\end{table}

\subsubsection{Summary across classes}

Overall, the additional classes support the main conclusion that localized
collateral forgetting is not specific to the \texttt{couch} class. Across all
five randomly selected affected classes, LTD achieves the lowest aggregate
Avg. Gap. This shows consistent global alignment with retraining across UA, RA,
TA, and MIA, rather than a result driven by a favorable choice of the main
affected class.

The affected-class results provide a more localized view. LTD achieves the best
affected-class Avg. Gap on three of the five classes and ranks second on the
remaining two. Thus, while the method is uniformly best in aggregate, the local
affected-class comparison is more competitive. This distinction is important:
aggregate metrics measure overall agreement with retraining, whereas affected-class
metrics isolate the neighborhood where collateral forgetting is most likely to
appear.

Importantly, the retraining target is class-dependent. In partial-class deletion,
there is no single desired level of UA or affected-class test accuracy: classes
such as \texttt{road} and \texttt{shark} induce different retraining trade-offs
between forgetting $D_f$ and preserving same-class test behavior. Thus, the
setting requires more than uniformly suppressing the deleted examples; the
unlearning method must flexibly match the class-specific retraining behavior. Notably, the LTD teacher is not merely preserving the original class identity: in partial deletion, the retrained model can assign nontrivial soft predictions to deleted examples, which are not determined by the original label alone. 

AMUN is the strongest competing baseline in these additional experiments,
especially on affected-class metrics. It achieves the best affected-class
Avg. Gap for \texttt{apple} and remains close to the best methods for
\texttt{road}. LTD is more consistent across the two evaluation views: it is
best on all aggregate tables and remains among the top two on all affected-class
tables. This suggests that local teacher targets improve retraining alignment
systematically across classes, rather than only for a single deletion target or
a single metric.

\subsection{Additional Experiments on SVHN with ViT-Tiny}
\label{sec:svhn_vit}

To assess the generalisability of the compared methods beyond the CIFAR setting, we conduct additional experiments on the SVHN dataset~\cite{Netzer2011ReadingDI} with a ViT-Tiny backbone~\cite{DBLP:journals/corr/abs-2010-11929} (patch size~$4$, image size~$32\times32$, 10 output classes) and apply class-forgetting for class~9. We consider two forgetting fractions: $\rho=0.5$, where half of the class-9 training samples are designated as the forget set, and the more challenging $\rho=0.9$.
All methods are evaluated across 5 random seeds.

Tables~\ref{tab:svhn_vit_main_50} and~\ref{tab:svhn_vit_affected_50}
show the results for $50\%$ deletion. LTD achieves the lowest global Avg.~Gap to retraining, improving over the strongest baseline AMUN ($2.6$ vs.\ $4.0$). The affected-class results show the same trend: LTD obtains the lowest affected-class Avg.~Gap and closely matches the retrained model on affected-class test accuracy. The corresponding bar plot in Figure~\ref{fig:svhn_vit_affected_bar_50} shows a competitive affected-class comparison: LTD closely matches retraining on affected-class RA and TA, with
small variance across runs, although it still overshoots UA.

\begin{table}[H]
\caption{
SVHN + ViT-Tiny results for partial deletion of digit $9$ with
class-fraction-to-forget $=0.5$.
Values are mean $\pm$ std over 5 runs; accuracy metrics are in $\%$ and RTE is in seconds.
Blue values show deviation from Retrain. Avg.~Gap is the mean absolute deviation from
Retrain over UA, RA, TA, and MIA.
}
\label{tab:svhn_vit_main_50}
\resizebox{\linewidth}{!}{%
\centering\small\setlength{\tabcolsep}{4pt}
\begin{tabular}{lcccccc}
\multicolumn{1}{c}{\bf Method} & \multicolumn{1}{c}{\bf UA} & \multicolumn{1}{c}{\bf RA} & \multicolumn{1}{c}{\bf TA} & \multicolumn{1}{c}{\bf MIA} & \multicolumn{1}{c}{\bf Avg.~Gap} & \multicolumn{1}{c}{\bf RTE} \\ \toprule
Retrain & $74.9 \pm 1.1$ & $100.0 \pm 0.0$ & $84.9 \pm 0.3$ & $29.9 \pm 1.6$ & 0 & $2115.0 \pm 5.8$ \\
\midrule
RL & $90.3 \pm 0.9$ \diff{+15.4} & $99.7 \pm 0.0$ \diff{-0.3} & $84.2 \pm 0.0$ \diff{-0.6} & $99.9 \pm 0.1$ \diff{+70.0} & \valuenotbest{$21.6$} & $197.6 \pm 0.9$ \\
FT & $40.0 \pm 48.9$ \diff{-34.9} & $61.2 \pm 32.9$ \diff{-38.8} & $55.4 \pm 25.9$ \diff{-29.5} & $58.2 \pm 47.6$ \diff{+28.3} & \valuenotbest{$32.8$} & $187.5 \pm 1.2$ \\
GA & $100.0 \pm 0.0$ \diff{+25.1} & $100.0 \pm 0.0$ \diffbest{+0.0} & $85.2 \pm 0.0$ \diff{+0.3} & $0.0 \pm 0.0$ \diff{-29.9} & \valuenotbest{$13.8$} & $5.3 \pm 0.1$ \\
IU & $100.0 \pm 0.0$ \diff{+25.1} & $100.0 \pm 0.0$ \diff{+0.0} & $85.2 \pm 0.0$ \diffbest{+0.3} & $0.0 \pm 0.0$ \diff{-29.9} & \valuenotbest{$13.8$} & $36.8 \pm 0.2$ \\
SalUn & $86.7 \pm 1.1$ \diff{+11.8} & $99.6 \pm 0.0$ \diff{-0.4} & $84.1 \pm 0.1$ \diff{-0.8} & $98.6 \pm 0.6$ \diff{+68.7} & \valuenotbest{$20.4$} & $356.4 \pm 1.1$ \\
AMUN & $75.2 \pm 14.0$ \diffbest{+0.4} & $93.0 \pm 6.8$ \diff{-7.0} & $83.4 \pm 3.0$ \diff{-1.5} & $37.2 \pm 20.8$ \diff{+7.3} & \valuenotbest{$4.0$} & $305.9 \pm 0.8$ \\
\midrule
Distill (ours) & $84.3 \pm 4.3$ \diff{+9.4} & $100.0 \pm 0.0$ \diff{-0.0} & $84.4 \pm 0.2$ \diff{-0.5} & $29.5 \pm 5.7$ \diffbest{-0.4} & \valuebest{$2.6$} & $348.1 \pm 5.0$ \\
\bottomrule
\end{tabular}
}
\end{table}

\begin{table}[H]
\caption{
Affected-class results on SVHN + ViT-Tiny for partial deletion of digit $9$ with
class-fraction-to-forget $=0.5$.
Values are mean $\pm$ std over 5 runs, in $\%$. Blue values show deviation from Retrain.
Avg.~Gap is computed over
RA$_{\text{affected-class}}$, UA$_{\text{affected-class}}$, and
TA$_{\text{affected-class}}$.
}
\label{tab:svhn_vit_affected_50}
\centering\small\setlength{\tabcolsep}{6pt}
\scalebox{0.9}{
\begin{tabular}{lcccc}
\multicolumn{1}{c}{\bf Method} & \multicolumn{1}{c}{\bf RA$_{\text{affected-class}}$} & \multicolumn{1}{c}{\bf UA$_{\text{affected-class}}$} & \multicolumn{1}{c}{\bf TA$_{\text{affected-class}}$} & \multicolumn{1}{c}{\bf Avg.~Gap} \\ \toprule
Retrain & $100.0 \pm 0.0$ & $74.9 \pm 1.1$ & $72.1 \pm 1.4$ & 0 \\
\midrule
RL & $92.5 \pm 1.2$ \diff{-7.5} & $90.3 \pm 0.9$ \diff{+15.4} & $64.9 \pm 1.1$ \diff{-7.2} & \valuenotbest{$10.0$} \\
FT & $40.1 \pm 48.9$ \diff{-59.9} & $40.0 \pm 48.9$ \diff{-34.9} & $32.4 \pm 39.6$ \diff{-39.7} & \valuenotbest{$44.8$} \\
GA & $100.0 \pm 0.0$ \diffbest{+0.0} & $100.0 \pm 0.0$ \diff{+25.1} & $81.3 \pm 0.1$ \diff{+9.2} & \valuenotbest{$11.4$} \\
IU & $100.0 \pm 0.0$ \diff{+0.0} & $100.0 \pm 0.0$ \diff{+25.1} & $81.2 \pm 0.0$ \diff{+9.1} & \valuenotbest{$11.4$} \\
SalUn & $88.2 \pm 1.2$ \diff{-11.8} & $86.7 \pm 1.1$ \diff{+11.8} & $61.0 \pm 1.3$ \diff{-11.1} & \valuenotbest{$11.6$} \\
AMUN & $83.0 \pm 18.5$ \diff{-17.0} & $75.2 \pm 14.0$ \diffbest{+0.4} & $66.5 \pm 12.2$ \diff{-5.6} & \valuenotbest{$7.7$} \\
\midrule
Distill (ours) & $100.0 \pm 0.0$ \diff{-0.0} & $84.3 \pm 4.3$ \diff{+9.4} & $71.4 \pm 3.3$ \diffbest{-0.7} & \valuebest{$3.4$} \\
\bottomrule
\end{tabular}}
\end{table}

\begin{figure}[H]
    \centering
    \includegraphics[width=1.\linewidth]{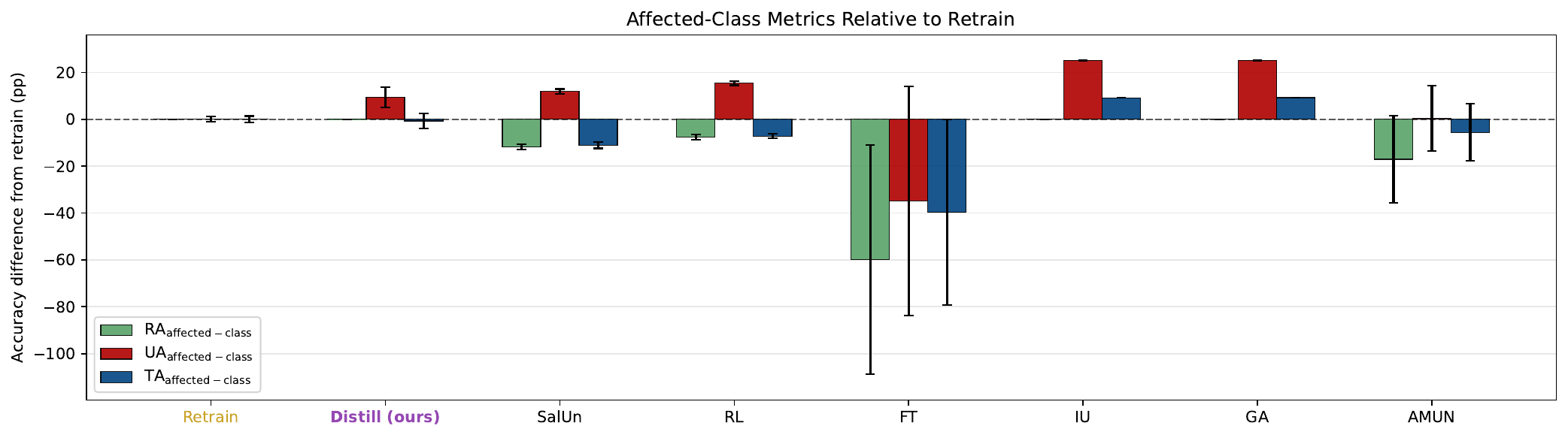}
    \caption{
    Affected-class deviations from retraining on SVHN + ViT-Tiny for partial
    deletion of digit $9$ with class-fraction-to-forget $=0.5$.
    Bars show differences in affected-class RA, UA, and TA relative to Retrain;
    zero corresponds to exact agreement with the retrained model.
    }
    \label{fig:svhn_vit_affected_bar_50}
\end{figure}

Figure~\ref{fig:svhn_vit_binned_50} further compares methods across similarity
bins. For $50\%$ deletion, LTD stays close to retraining across most retain and
test bins, both in accuracy and confidence. By contrast, several baselines have
large deviations in the high-similarity region near the forget set. Thus, in this partial-deletion regime, LTD transfers well to SVHN and a transformer backbone.

\begin{figure}[H]
    \centering
    \includegraphics[width=1.\linewidth]{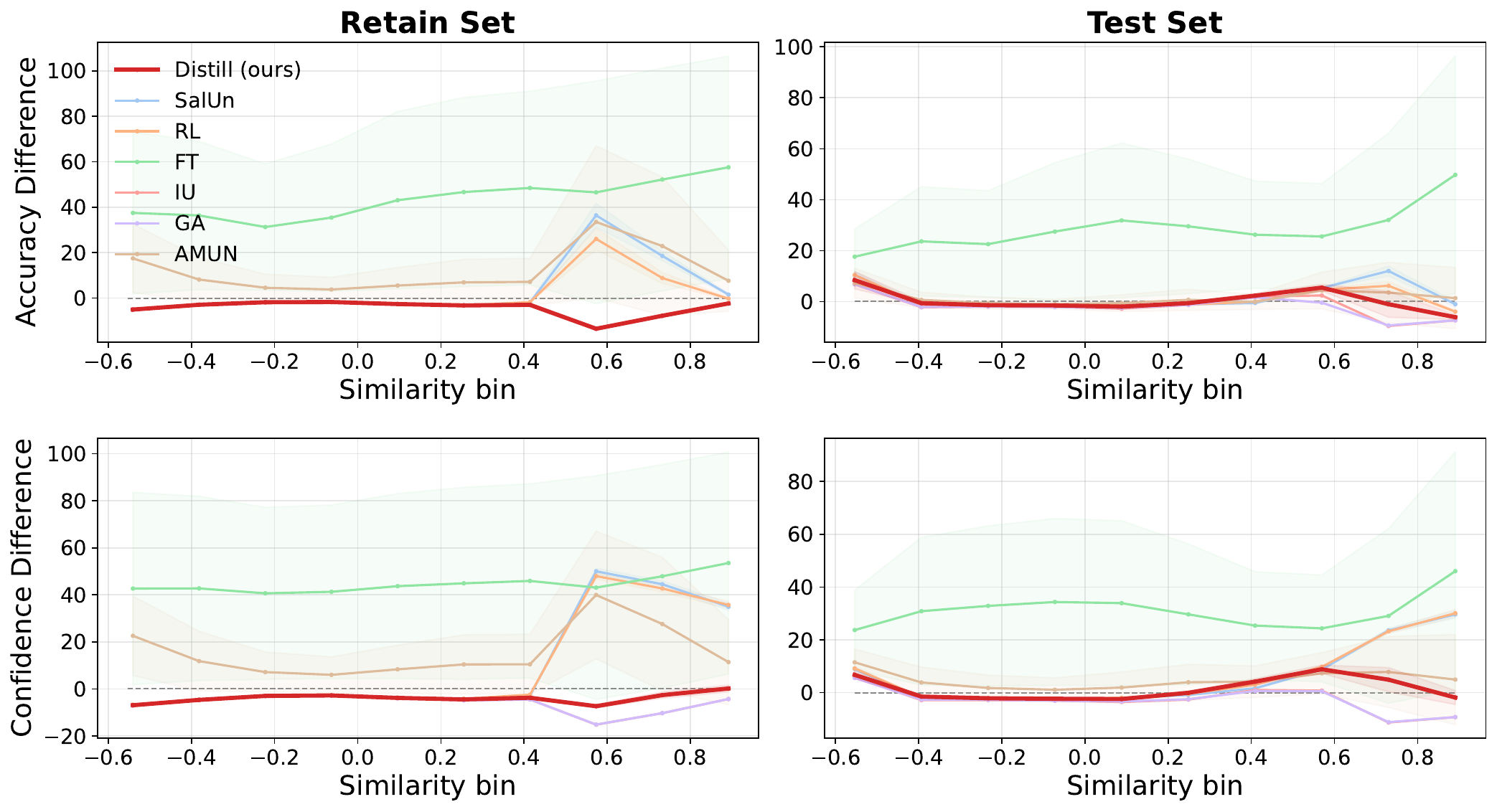}
    \caption{
    Retrain-relative accuracy and confidence differences across similarity bins on
    SVHN + ViT-Tiny for class-fraction-to-forget $=0.5$.
    Similarity is computed in the penultimate-layer representation space of the
    fully trained model with respect to the forget set. Top row: accuracy
    differences; bottom row: correct-class confidence differences. Left: retain
    set; right: test set. Values close to zero indicate agreement with retraining.
    }
    \label{fig:svhn_vit_binned_50}
\end{figure}

The $90\%$ deletion setting is more challenging
(Tables~\ref{tab:svhn_vit_main_90} and~\ref{tab:svhn_vit_affected_90}).
Here AMUN achieves the lowest global Avg.~Gap, while LTD ranks second. LTD closely matches retraining in UA, RA, and TA, but differs more substantially
in MIA. In our convention, a higher MIA value means that deleted examples are
more often classified as non-members; however, our evaluation is
retrain-relative, so the target is to match the auditing behavior of the
retrained model rather than to maximize this score. Thus, the larger MIA value
of LTD indicates an over-shift relative to retraining in this regime.

The affected-class metrics show a different view. LTD achieves the lowest affected-class Avg.~Gap, improving over AMUN ($7.3$ vs.\ $12.1$), and keeps all three affected-class deviations moderate. AMUN better matches affected-class TA, but has a larger drop in affected-class RA. 

\begin{table}[H]
\caption{
SVHN + ViT-Tiny results for partial deletion of digit $9$ with
class-fraction-to-forget $=0.9$.
Values are mean $\pm$ std over 5 runs; accuracy metrics are in $\%$ and RTE is in seconds.
Blue values show deviation from Retrain. Avg.~Gap is the mean absolute deviation from
Retrain over UA, RA, TA, and MIA.
}
\label{tab:svhn_vit_main_90}
\resizebox{\linewidth}{!}{%
\centering\small\setlength{\tabcolsep}{4pt}
\begin{tabular}{lcccccc}
\multicolumn{1}{c}{\bf Method} & \multicolumn{1}{c}{\bf UA} & \multicolumn{1}{c}{\bf RA} & \multicolumn{1}{c}{\bf TA} & \multicolumn{1}{c}{\bf MIA} & \multicolumn{1}{c}{\bf Avg.~Gap} & \multicolumn{1}{c}{\bf RTE} \\ \toprule
Retrain & $51.1 \pm 2.2$ & $100.0 \pm 0.0$ & $82.9 \pm 1.3$ & $56.2 \pm 2.4$ & 0 & $2058.9 \pm 5.9$ \\
\midrule
RL & $36.0 \pm 4.1$ \diff{-15.0} & $99.6 \pm 0.0$ \diff{-0.4} & $81.9 \pm 0.2$ \diff{-1.0} & $100.0 \pm 0.0$ \diff{+43.8} & \valuenotbest{$15.1$} & $200.5 \pm 1.8$ \\
FT & $43.7 \pm 46.1$ \diff{-7.4} & $80.4 \pm 20.5$ \diff{-19.6} & $70.9 \pm 15.5$ \diff{-12.0} & $58.6 \pm 47.8$ \diff{+2.4} & \valuenotbest{$10.3$} & $182.3 \pm 0.6$ \\
GA & $36.8 \pm 28.7$ \diff{-14.3} & $54.0 \pm 22.2$ \diff{-46.0} & $46.4 \pm 19.0$ \diff{-36.5} & $60.4 \pm 24.3$ \diff{+4.2} & \valuenotbest{$25.2$} & $7.9 \pm 0.1$ \\
IU & $100.0 \pm 0.0$ \diff{+48.9} & $100.0 \pm 0.0$ \diffbest{+0.0} & $85.2 \pm 0.0$ \diff{+2.3} & $0.0 \pm 0.0$ \diff{-56.1} & \valuenotbest{$26.8$} & $36.6 \pm 0.3$ \\
SalUn & $31.5 \pm 3.8$ \diff{-19.6} & $99.5 \pm 0.0$ \diff{-0.5} & $81.7 \pm 0.2$ \diff{-1.1} & $100.0 \pm 0.0$ \diff{+43.8} & \valuenotbest{$16.3$} & $357.6 \pm 1.2$ \\
AMUN & $57.0 \pm 15.5$ \diffbest{+5.9} & $94.1 \pm 4.7$ \diff{-5.9} & $83.4 \pm 2.6$ \diffbest{+0.5} & $57.6 \pm 16.7$ \diffbest{+1.4} & \valuebest{$3.4$} & $393.9 \pm 0.8$ \\
\midrule
Distill (ours) & $57.6 \pm 7.4$ \diff{+6.5} & $99.9 \pm 0.0$ \diff{-0.1} & $82.2 \pm 0.3$ \diff{-0.7} & $88.1 \pm 3.8$ \diff{+31.9} & \valuenotbest{$9.8$} & $306.9 \pm 2.9$ \\
\bottomrule
\end{tabular}
}
\end{table}

\begin{table}[H]
\caption{
Affected-class results on SVHN + ViT-Tiny for partial deletion of digit $9$ with
class-fraction-to-forget $=0.9$.
Values are mean $\pm$ std over 5 runs, in $\%$. Blue values show deviation from Retrain.
Avg.~Gap is computed over
RA$_{\text{affected-class}}$, UA$_{\text{affected-class}}$, and
TA$_{\text{affected-class}}$.
}
\label{tab:svhn_vit_affected_90}
\centering\small\setlength{\tabcolsep}{6pt}
\scalebox{0.9}{
\begin{tabular}{lcccc}
\multicolumn{1}{c}{\bf Method} & \multicolumn{1}{c}{\bf RA$_{\text{affected-class}}$} & \multicolumn{1}{c}{\bf UA$_{\text{affected-class}}$} & \multicolumn{1}{c}{\bf TA$_{\text{affected-class}}$} & \multicolumn{1}{c}{\bf Avg.~Gap} \\ \toprule
Retrain & $100.0 \pm 0.0$ & $51.1 \pm 2.2$ & $48.9 \pm 2.3$ & 0 \\
\midrule
RL & $36.3 \pm 4.3$ \diff{-63.7} & $36.0 \pm 4.1$ \diff{-15.0} & $23.8 \pm 2.6$ \diff{-25.1} & \valuenotbest{$34.6$} \\
FT & $44.5 \pm 45.5$ \diff{-55.5} & $43.7 \pm 46.1$ \diff{-7.4} & $34.5 \pm 37.1$ \diff{-14.4} & \valuenotbest{$25.8$} \\
GA & $36.3 \pm 28.0$ \diff{-63.7} & $36.8 \pm 28.7$ \diff{-14.3} & $29.9 \pm 24.0$ \diff{-19.0} & \valuenotbest{$32.3$} \\
IU & $100.0 \pm 0.0$ \diffbest{+0.0} & $100.0 \pm 0.0$ \diff{+48.9} & $81.2 \pm 0.0$ \diff{+32.3} & \valuenotbest{$27.0$} \\
SalUn & $30.7 \pm 3.2$ \diff{-69.3} & $31.5 \pm 3.8$ \diff{-19.6} & $20.8 \pm 2.4$ \diff{-28.1} & \valuenotbest{$39.0$} \\
AMUN & $70.0 \pm 21.8$ \diff{-30.0} & $57.0 \pm 15.5$ \diffbest{+5.9} & $49.4 \pm 14.2$ \diffbest{+0.5} & \valuenotbest{$12.1$} \\
\midrule
Distill (ours) & $93.0 \pm 4.1$ \diff{-7.0} & $57.6 \pm 7.4$ \diff{+6.5} & $40.6 \pm 5.1$ \diff{-8.3} & \valuebest{$7.3$} \\
\bottomrule
\end{tabular}}
\end{table}

\begin{figure}[H]
    \centering
    \includegraphics[width=1.\linewidth]{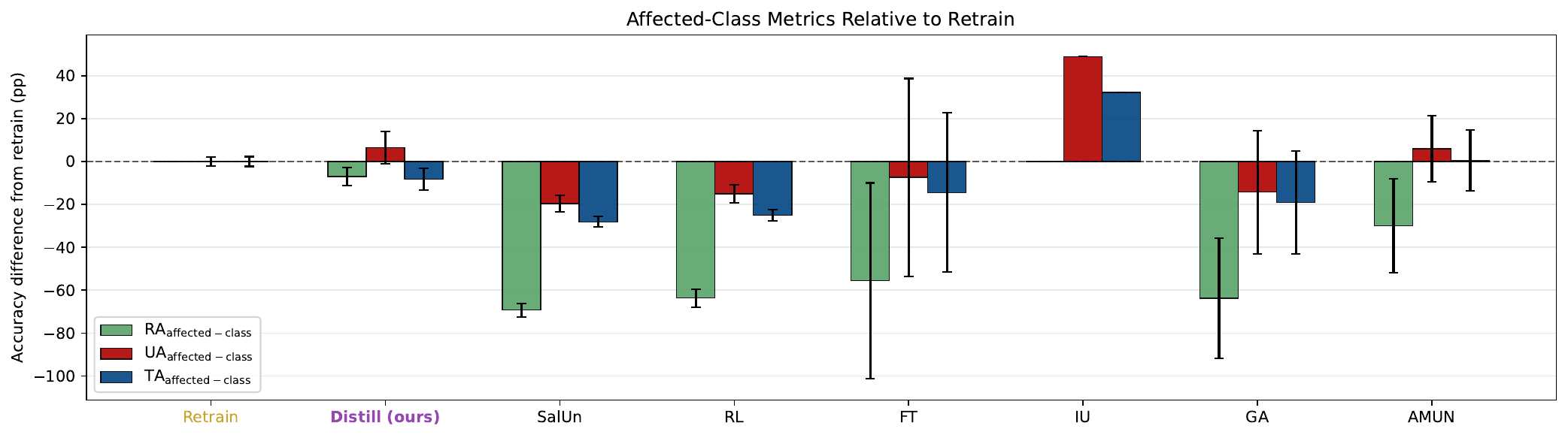}
    \caption{
    Affected-class deviations from retraining on SVHN + ViT-Tiny for partial
    deletion of digit $9$ with class-fraction-to-forget $=0.9$.
    Bars show differences in affected-class RA, UA, and TA relative to Retrain;
    zero corresponds to exact agreement with the retrained model.
    }
    \label{fig:svhn_vit_affected_bar_90}
\end{figure}

The binned plots in Figure~\ref{fig:svhn_vit_binned_90} provide a more fine-grained view of this effect. In the aggressive $90\%$ deletion regime, several baselines exhibit large retrain-relative deviations that are
concentrated in the high-similarity region near the forget set. LTD substantially reduces this localized deviation: its curves stay close to retraining over
low- and medium-similarity bins and increase more mildly near the forget neighborhood. The effect is not eliminated, especially for confidence on the
most similar test examples, but it is reduced compared with the stronger localized shifts observed for many baselines.

\begin{figure}[H]
    \centering
    \includegraphics[width=1.\linewidth]{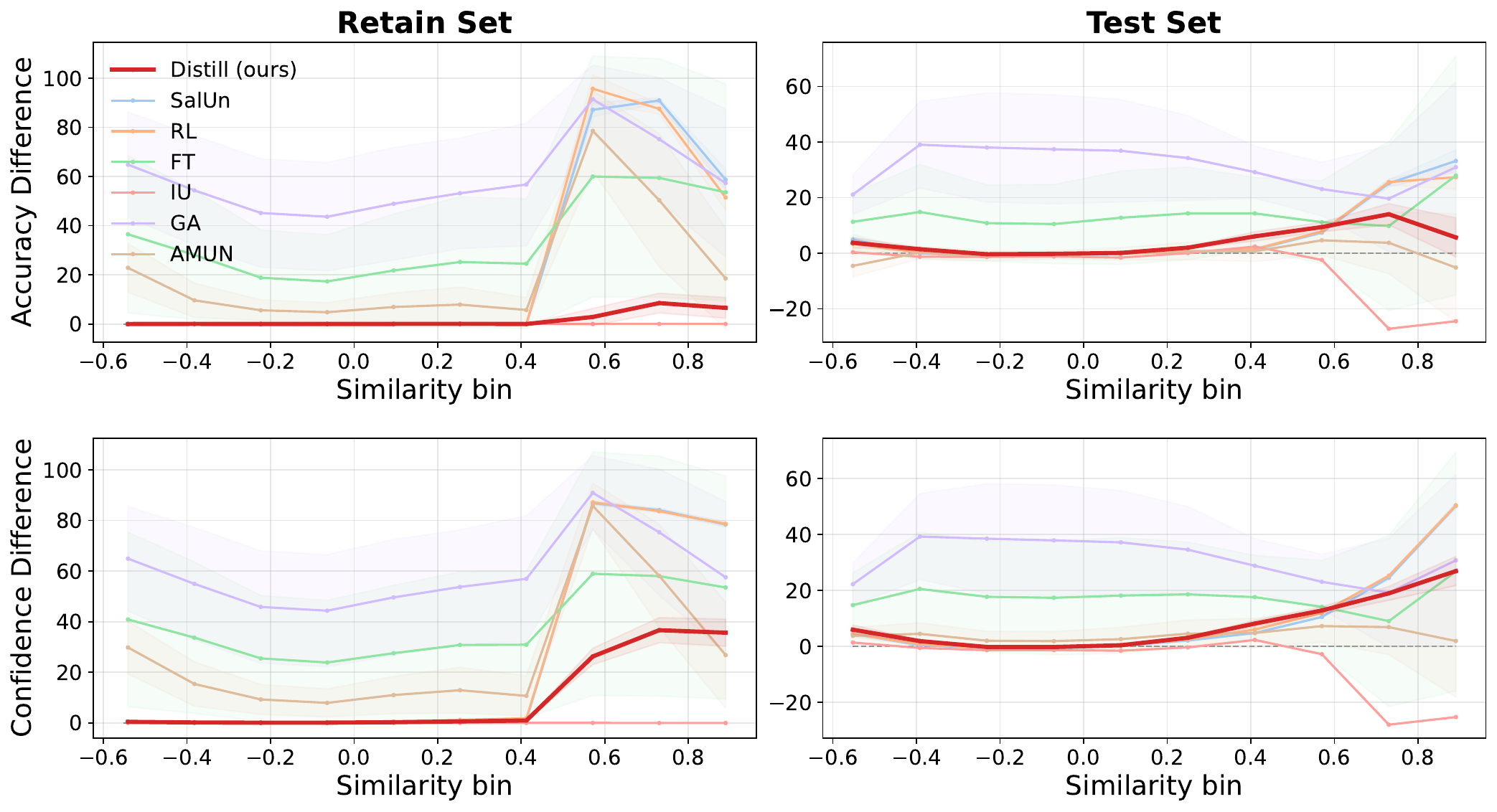}
    \caption{
    Retrain-relative accuracy and confidence differences across similarity bins on
    SVHN + ViT-Tiny for class-fraction-to-forget $=0.9$.
    Similarity is computed in the penultimate-layer representation space of the
    fully trained model with respect to the forget set. Top row: accuracy
    differences; bottom row: correct-class confidence differences. Left: retain
    set; right: test set. Values close to zero indicate agreement with retraining.
    }
    \label{fig:svhn_vit_binned_90}
\end{figure}

\section{Broader impacts}~\label{app:broader-impacts}

Machine unlearning can support data-deletion requests and improve transparency in
model maintenance by clarifying when unlearning deviates from retraining. At the
same time, unreliable approximate unlearning may create a false sense of deletion
or privacy protection if used without careful auditing. Our work is diagnostic
and methodological, evaluated on public image-classification benchmarks, and
should not be interpreted as providing certified deletion guarantees.

\end{document}